\documentclass{article} 

\usepackage{amsmath,amsfonts,mathrsfs,bm}

\def\eqref#1{equation~\ref{#1}}

\def\ceil#1{\lceil #1 \rceil}

\def\1{\bm{1}}

\DeclareMathAlphabet{\mathsfit}{\encodingdefault}{\sfdefault}{m}{sl}
\SetMathAlphabet{\mathsfit}{bold}{\encodingdefault}{\sfdefault}{bx}{n}

\def\gF{{\mathcal{F}}}
\def\gG{{\mathcal{G}}}
\def\gH{{\mathcal{H}}}

\def\gM{{\mathcal{M}}}

\def\gT{{\mathcal{T}}}

\def\gX{{\mathcal{X}}}
\def\gY{{\mathcal{Y}}}

\newcommand{\E}{\mathbb{E}}

\newcommand{\R}{\mathbb{R}}

\def\P{\mathbb{P}}
\def\E{\mathbb{E}}

\usepackage{float}
\usepackage{url}

\usepackage{graphicx}
\graphicspath{{./}{./figures/}}

\usepackage{epigraph}
\usepackage{amsmath,amsthm,amssymb,mathtools}
\usepackage[bookmarksnumbered=true,
			bookmarksopen=true,
			backref=section,
			colorlinks=true,
			citecolor=blue,
			linkcolor=blue,
			anchorcolor=red,
			urlcolor=blue]{hyperref}
\usepackage{titlesec}

\usepackage{natbib}
\setcitestyle{authoryear,open={(},close={)}}

\usepackage[misc]{ifsym} 

\newtheorem{theorem}{Theorem}
\newtheorem{prop}{Proposition}
\newtheorem{lemma}{Lemma}[section]

\newtheorem*{remark}{Remark}
\newtheorem{example}{Example}[section]

\newcommand{\RC}{\mathcal{R}}
\newcommand{\Bn}{\mathcal{B}}
\newcommand{\eP}{\mathbb{P}}
\newcommand{\nf}{\tilde{f}}
\newcommand{\nh}{\tilde{h}}
\newcommand{\Size}{\mathrm{Size}}

\title{Rethinking Breiman's Dilemma in Neural Networks: \\ Phase Transitions of Margin Dynamics}

\author{Weizhi Zhu, Yifei Huang, and Yuan Yao\footnote{Correspondence email: \url{yuany@ust.hk}. }
\\ \\
\emph{Department of Mathematics} \\
\emph{Hong Kong University of Science and Technology} \\
}
\date{}

\begin{document}
\maketitle

\begin{abstract}
Margin enlargement over training data has been an important strategy since perceptrons in machine learning for the purpose of boosting the confidence of training toward a good generalization ability. Yet Breiman shows a dilemma \citep{Breiman99} that a uniform improvement on margin distribution \emph{does not} necessarily reduces generalization errors. In this paper, we revisit Breiman's dilemma in deep neural networks with recently proposed spectrally normalized margins, from a novel perspective based on phase transitions of normalized margin distributions in training dynamics. Normalized margin distribution of a classifier over the data, can be divided into two parts: low/small margins such as some negative margins for misclassified samples vs. high/large margins for high confident correctly classified samples, that often behave differently during the training process. Low margins for training and test datasets are often effectively reduced in training, along with reductions of training and test errors; while high margins may exhibit different dynamics, reflecting the trade-off between expressive power of models and complexity of data. When data complexity is comparable to the model expressiveness, high margin distributions for both training and test data undergo similar decrease-increase phase transitions during training. In such cases, one can predict the trend of generalization or test error by margin-based generalization bounds with restricted Rademacher complexities, shown in two ways in this paper with early stopping time exploiting such phase transitions. On the other hand, over-expressive models may have both low and high training margins undergoing uniform improvements, with a distinct phase transition in test margin dynamics. This reconfirms the Breiman's dilemma associated with overparameterized neural networks where margins fail to predict overfitting. Experiments are conducted with some basic convolutional networks, AlexNet, VGG-16, and ResNet-18, on several datasets including Cifar10/100 and mini-ImageNet.
\end{abstract}

\section{Introduction}
Margin, as a measurement of the robustness allowing some perturbations on classifier without changing its decision on training data, has a long history in characterizing the performance of classification algorithms in machine learning. As early as \cite{Novikoff62}, it played a central role in the proof on finite-stopping or convergence of perceptron algorithm when training data is separable. Equipped with convex optimization technique, a plethora of large margin classifiers were triggered by support vector machines   \citep{svm,Vapnik98}. For neural networks, \cite{Bartlett97,Bartlett98} showed that the generalization error can be bounded by a margin-sensitive fat-shattering dimension, which is in turn bounded by the $\ell_1$-norm of weights, shedding light on possible good generalization ability of over-parameterizd networks with small size weights despite the large VC dimensionality. The same idea was later applied to AdaBoost, an iterative algorithm to combine an ensemble of classifiers proposed by \cite{adaboost}, often exhibiting a phenomenon of resistance to overfitting that during the training process the generalization error does not increase even when the training error drops to zero. Toward deciphering such a resistance to overfitting phenomenon, \cite{SFBL98} proposed an explanation that the training process keeps on improving a notion of classification margins in boosting, among later improvement \citep{koltchinskii2002empirical} and works on establishing consistency of boosting via early stopping regularization \citep{BuhYu02,ZhaYu05,YaoRosCap07}. Lately such a resistance to overfitting was again observed in deep neural networks with over-parameterized models \citep{zhang2016understanding}. A renaissance of margin theory was brought by \cite{BFT17} with a normalization of network using Lipschitz constants bounded by products of operator spectral norms. It inspires many further investigations in various settings \citep{miyato2018spectral,nati18_pacbayesian,Poggio_memo91}.  

However, the margin theory has a limitation that the improvement of margin distributions does not necessarily guarantee a better generalization performance, which is at least traced back to \cite{Breiman99} in his effort to understanding AdaBoost. In this work, Breiman designed an algorithm \emph{arc-gv} such that the margin can be maximized via a prediction game, then he demonstrated an example that one can achieve uniformly larger margin distributions on training data than AdaBoost but suffer a higher generalization error. In the end of this paper, Breiman made the following comments with a dilemma: 

\indent \emph{``The results above leave us in a quandary. The laboratory results for various arcing algorithms are excellent, but the theory is in disarray. The evidence is that if we try too hard to make the margins larger, then overfitting sets in. $\cdots$ My sense of it is that we just do not understand enough about what is going on."}


In this paper, we are going to revisit Breiman's dilemma in the scenario of deep neural networks. We shall see margin distributions on training and test data may behave differently on the low and high parts during training processes. First of all, let's look at the following illustration example.
\medskip
\begin{example}[Breiman's Dilemma with a CNN] \label{exmp:first}
A basic 5-layer convolutional neural network of $c$ channels (see Section \ref{sect:exp} for details) is trained with CIFAR-10 dataset whose 10 percent labels are randomly permuted as injected noises. When $c=50$ with $92,610$ parameters, Figure \ref{fig:bcnn-main} (a) shows the training error and generalization (test) error in solid curves. From the generalization error in (a) one can see that overfitting indeed happens after about 10 epochs, despite that training error continuously drops down to zero. One can successfully predict such an overfitting phenomenon from Figure \ref{fig:bcnn-main} (b), the evolution of normalized training margin distributions defined later in this paper. In (b), while low or small margins are monotonically improved during training, high or large margins undergoes a phase transition from increase to decrease around 10 epochs such that one can predict the tendency of generalization error in (a) using high margin dynamics. Two particular sections of high margin dynamics are highlighted in (b), one at 9.8 on $x$-axis that measures the percentage of normalized training margins no more than 9.8 (training margin error) and the other at 0.8 on $y$-axis that measures the normalized margins at quantile $q=0.8$ (i.e. $1/\hat{\gamma}_{q,t}$ defined later). Both of them meet the tendency of generalization error in (a) and find good early stopping time to avoid overfitting. However, as we increase the channel number to $c=400$ with about $5.8M$ parameters and retrain the model, (c) shows a similar overfitting phenomenon in generalization error; on the other hand, (d) exhibits a uniform improvement of both low and high normalized margins without a phase transition during the training and thus fails to capture the overfitting. This demonstrates the Breiman's dilemma in wide CNN.
\end{example}

\begin{figure}[ht!] 
\begin{center}
\includegraphics[width=.48\textwidth]{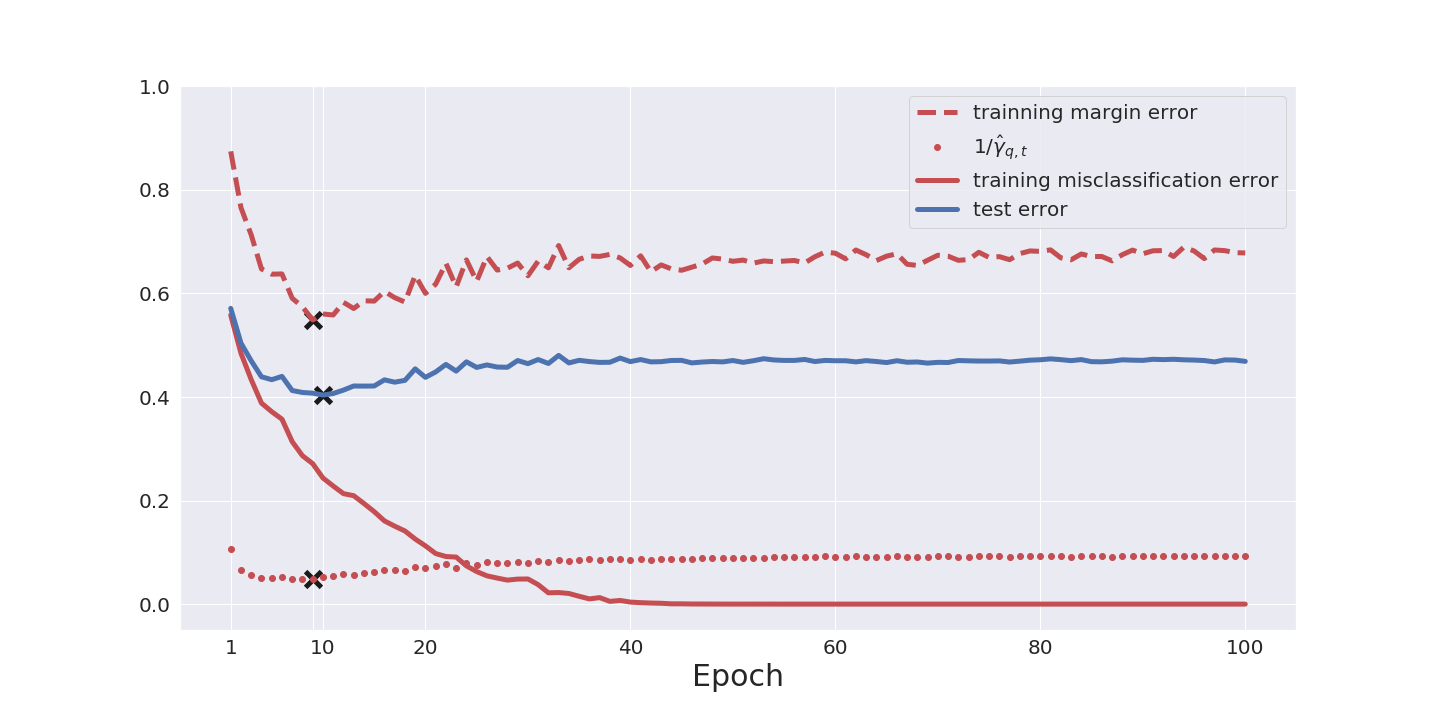}
\includegraphics[width=.48\textwidth]{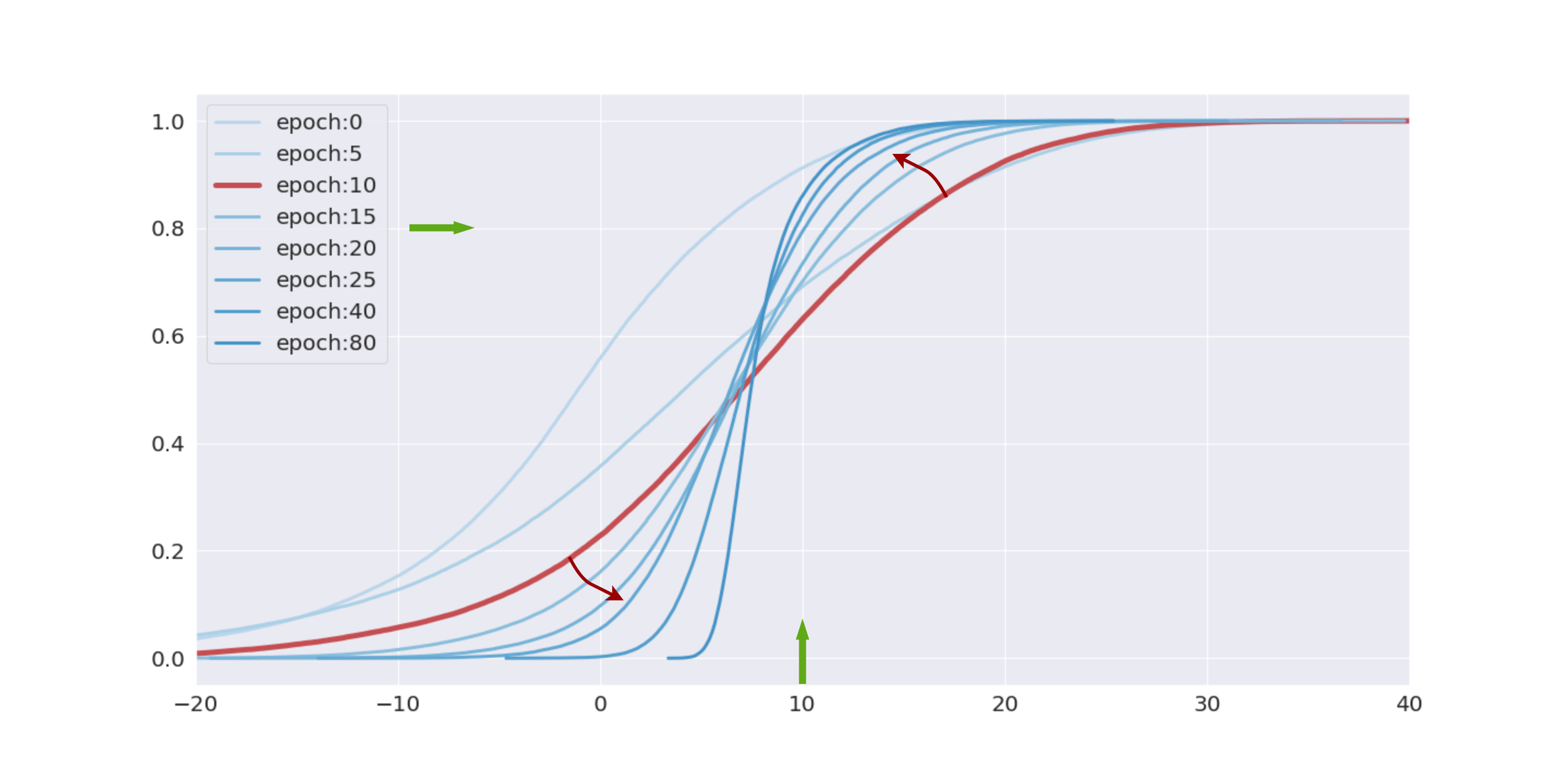}\\
(a)\ \ \ \ \ \ \ \ \ \ \ \ \ \ \ \ \ \ \ \ \ \ \ \ \ \ \ \ \ \ \ \ \ \ \ \ \ \ \ \ \ \ \ \ \ \ \ \ \ \ \ \ \ \ \ \ \ (b) \\
\includegraphics[width=.48\textwidth]{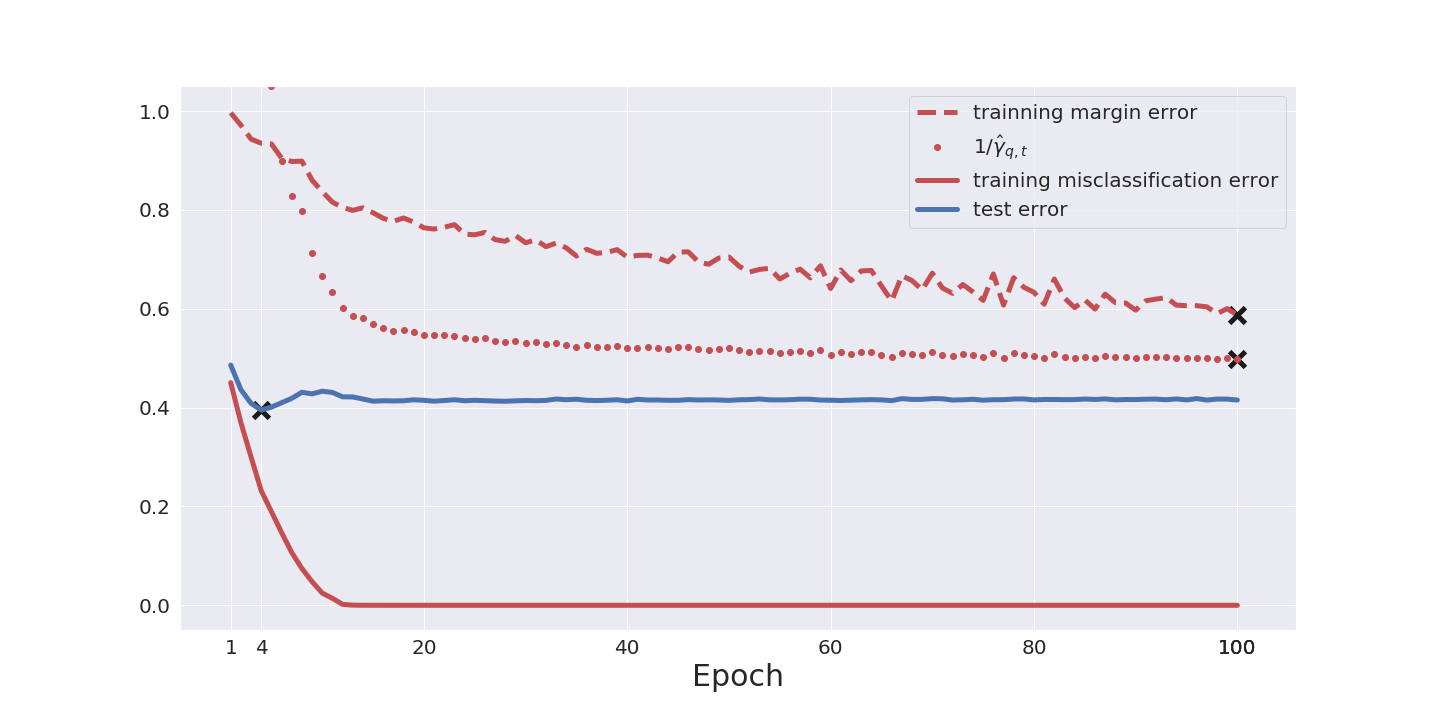} 
\includegraphics[width=.48\textwidth]{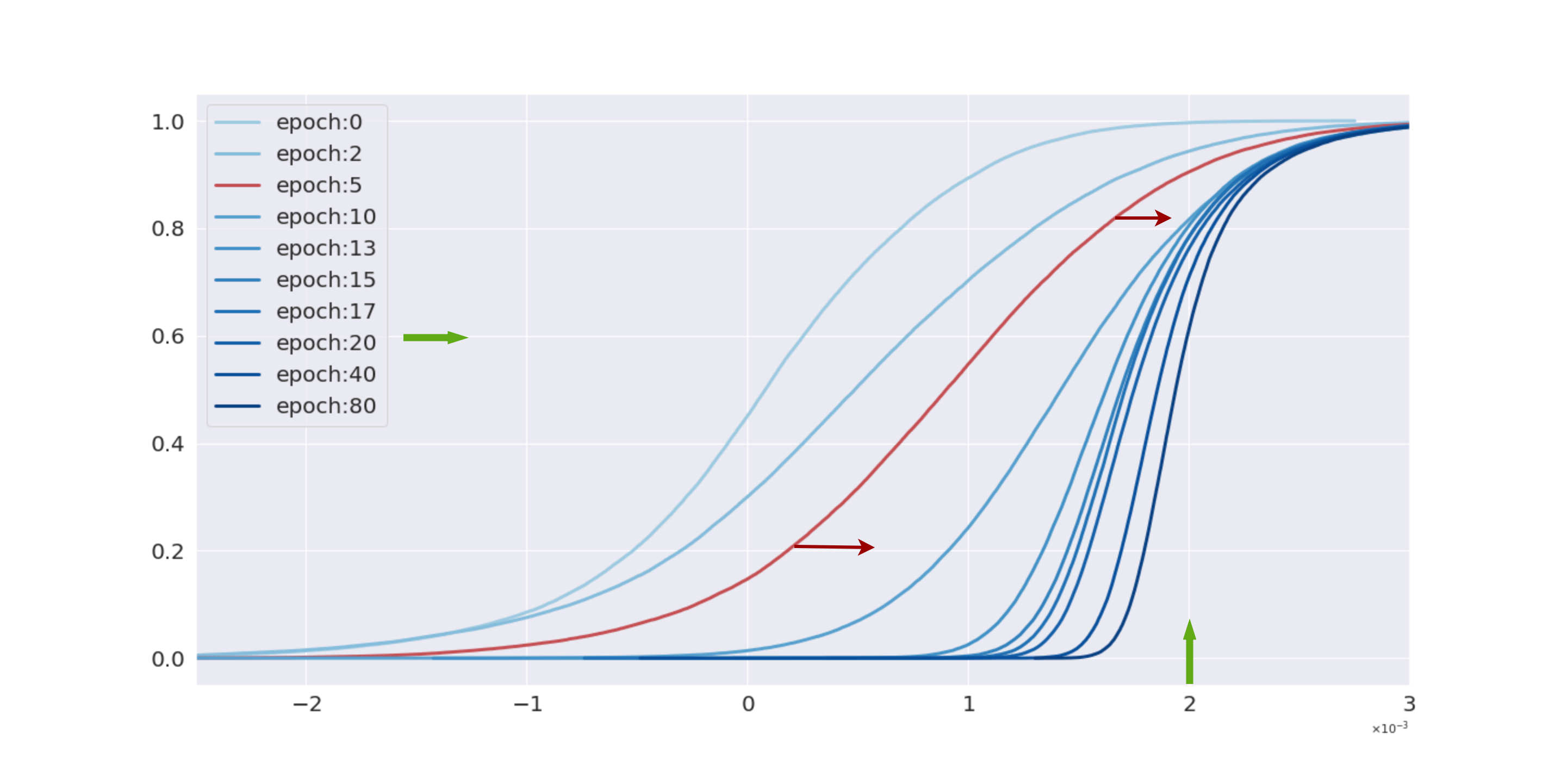} \\
(c)\ \ \ \ \ \ \ \ \ \ \ \ \ \ \ \ \ \ \ \ \ \ \ \ \ \ \ \ \ \ \ \ \ \ \ \ \ \ \ \ \ \ \ \ \ \ \ \ \ \ \ \ \ \ \ \ \ (d)
\end{center}
\caption{Demonstration of Breiman's Dilemma in Convolutional Neural Networks. See Example \ref{exmp:first} for details. 
}\label{fig:bcnn-main}
\end{figure}

A key insight behind this dilemma, is that one needs a trade-off between the expressive power of models and the complexity of the dataset to endorse training margins a prediction power. On one hand, when a model has a limited expressive power relative to the training dataset, in the sense that the low and high training margins CAN NOT be uniformly improved during training, low margins can be effectively enlarged during training by reducing the training loss, at the cost of sacrificing high margins that does not affect the training loss as much as low margins indicating mis-classfied samples. In this case, the generalization or test error may be predicted from dynamics of normalized training margin distributions by the increase-decrease phase transition that high margins experience. On the other hand, if we push too hard to improve margins by giving models too much degree of freedom such that the training margins are uniformly improved during training process, the predictability may be lost and overfitting may set in. A trade-off is thus necessary to balance the complexity of model and dataset, otherwise one is doomed to meet Breiman's dilemma when the models arbitrarily increase the expressive power.  

The example above shows that the expressive power of models relative to the complexity of dataset, can be observed from the dynamics of normalized margins in training, instead of counting the number of parameters in neural networks. In the sequel, our main contributions are to make these precise by revisiting the Rademacher complexity bounds on network generalization error. 

\begin{itemize}
\item With the Lipschitz-normalized margins, a linear inequality is established between training margin and test margin in Theorem \ref{thm:marg-err}. When both training and test normalized margin distributions undergo similar phase transitions on increase-decrease during the training process, one may predict the generalization error based on the training margins as illustrated in Figure \ref{fig:bcnn-main}. 
\item In a dual direction, one can define a \emph{quantile margin} via the inverse of margin distribution functions, to establish another linear inequality between the inverse quantile margins and the test margins as shown in Theorem \ref{thm:qmargin}. Quantile margin is far easier to tune in practice and enjoys a stronger prediction power exploiting an adaptive selection of margins along model training.
\item In all cases, Breiman's dilemma may fail both of the methods above when dynamics of normalized training margins undergo different phase transitions to that of test margins during training, where a uniform improvement of training margins results in overfitting. 
\end{itemize}

Section \ref{sect:method} describes our method to derive the two linear inequalities of generalization bounds above. Extensive experimental results are
shown in Section \ref{sect:exp} with basic CNNs, AlexNet, VGG, ResNet, and various datasets including CIFAR10, CIFAR100, and mini-Imagenet. Conclusions and future directions are discussed in Section \ref{sec:conclusion}. More experimental figures and proofs are collected in Appendices. 

\section{Methodology} \label{sect:method}

\subsection{Definitions and Notation}
Let $\gX$ be the input space (e.g. $\gX\subseteq\R^{C\times W\times H}$ in image classification of size \#(channel)-by-\#(width)-by-\#(height)) and $\gY:=\{1,\ldots,K\}$ be the space of $K$ classes. Consider a sample set of $n$ observations $S=\{(x_1, y_1),\ldots,(x_n,y_n):x_i\in\gX, y_i\in \gY\}$ that are drawn i.i.d. from $P_{X,Y}$. For any function $f:\gX\times\gY \to \R$, let $\P f = \int_{\gX\times\gY} f(X, Y) dP_{X, Y}$ be the population expectation and $\eP_n f=(1/n)\sum_{i=1}^nf(x_i)$ be the sample average. 

Define $\gF$ to be the space of functions $f: \gX\to\R^K$ represented by neural networks, 
\begin{equation}\label{eq:F} 
\left\{
\begin{array}{rcl}
x_{0} & = & x, \\ 
x_{i}  & = & \sigma_{i} (W_{i} x_{i-1} + b_{i}), \ \ \ \ i=1,\ldots,l-1,  \\
f(x) & = & W_lx_{l-1}+b_{l},  
\end{array}
\right.
\end{equation}
where $l$ is the depth of the network, $W_i$ is the weight matrix corresponding to a linear operator on $x_{i}$ and $\sigma_i$ stands for either element-wise activation function (e.g. ReLU) or pooling operator that are assumed to be Lipschitz bounded with constant $L_{\sigma_i}$. For example, in convolutional network, $W_i x_i + b_i = w_i \ast x_i + b_i$ where $\ast$ stands for the convolution between input tensor $x_l$ and kernel tensor $w_l$. We equip $\gF$ with the \emph{Lipschitz semi-norm}, that is, for each $f$,
\begin{equation}\label{eq:Lip}
    \|f\|_{\gF}:= \sup_{x\neq x\prime}\frac{\|f(x) - f(x^\prime)\|_2}{\|x-x^\prime\|_2}\leq L_\sigma \prod_{i=1}^{l} \|W_i\|_{\sigma} := L_f, 
\end{equation}
where $\|\cdot\|_{\sigma}$ is the spectral norm and $L_\sigma=\prod_{i=1}^l L_{\sigma_i}$. Without loss of generality, we assume $L_\sigma=1$ for simplicity. Moreover we consider the following family of hypothesis functions as network mapping evaluated at $(x,y)$,
\begin{equation}\label{eq:H}
\gH = \{h(x)=[f(x)]_y: \gX\to\R, f\in\gF, y\in\gY\},
\end{equation}
where $[\cdot]_j$ denotes the $j^{\textnormal{th}}$ coordinate and we further define the following class induced by Lipschitz semi-norm bound on $\gF$,
\begin{equation}\label{eq:H1}
 \gH_L=\{h(x)=[f(x)]_y: \gX\to\R, h(x)=[f(x)]_y\in\gH\ \textnormal{with}\ \|f\|_{\gF}\leq L, y\in\gY\}.
\end{equation}

Now, rather than merely looking at whether a prediction $f(x)$ on $y$ is correct or not, we further consider the prediction \emph{margin} defined as $\zeta(f(x), y) = [f(x)]_{y} - \max_{\{j: j\neq y\}} [f(x)]_j$. With that, we can define the \emph{ramp loss} and \emph{margin error} depending on the confidence of predictions. Given two thresholds $\gamma_2 > \gamma_1 \geq 0$, define the \emph{ramp loss} to be
\begin{equation*}
\ell_{(\gamma_1, \gamma_2)}(\zeta) = \left\{
\begin{array}{rcl}
1 & & {\zeta < \gamma_1},\\
- \frac{1}{\Delta}(\zeta-\gamma_2) & & {\gamma_1 \leq\zeta\leq \gamma_2},\\
0 & & {\zeta>\gamma_2},\\
\end{array} \right.
\end{equation*}
where $\Delta:=\gamma_2-\gamma_1$. In particular $\gamma_1=0$ and $\gamma_2=\gamma$, we also write $\ell_\gamma=\ell_{(0, \gamma)}$ for simplicity. 
Define the \emph{margin error} to measure if $f$ has margin no more than a threshold $\gamma$,
\begin{equation}\label{eq:marg-err}
e_\gamma(f(x), y) = \left\{
\begin{array}{rcl}
1 & & {\zeta(f(x),y) \leq \gamma}\\
0 & & {\zeta(f(x),y) > \gamma}\\
\end{array}. \right.
\end{equation}
In particular, $e_0(f(x),y)$ is the common mis-classification error and $\E [e_0(f(x),y)] = \eP [\zeta(f(x),y)<0]$. Note that $e_0 \leq \ell_\gamma \leq e_\gamma$, and $\ell_\gamma$ is Lipschitz bounded by $1/\gamma$. 

\subsection{Rademacher Complexity and the Scaling Issue}
The central question we try to answer is, \emph{can we find a proper upper bound to predict the tendency of the generalization error along training, such that one can early stop the training near the epoch that $\eP[\zeta(f_t(x),y)<0]$ is minimized?} 

We begin with the following lemma, as a margin-based generalization bound with network Rademacher complexity for multi-label classifications, using the uniform law of large numbers \citep{koltchinskii2002empirical,cortes2013multi,kuznetsov2015rademacher,BFT17} 
\begin{lemma}[Rademacher Complexity based Generalization Bound] \label{lem:margbound}
Given a $\gamma_0>0$, then, for any $\delta \in(0,1)$, with probability at least $1-\delta$, the following holds for any $f\in \gF$ with $\|f\|_\gF\leq L$,
\begin{equation}\label{eq:rc-standard}
    \E[\ell_{\gamma_0}(f(x),y)] \leq \frac{1}{n}\sum_{i=1}^n [\ell_{\gamma_0}(f(x_i),y_i)] +
    \frac{4K}{\gamma_0}\RC_{n}(\gH_L) + \sqrt{\frac{\log(1/\delta)}{2n}},
\end{equation}
where 
\begin{equation}
    \RC_{n}(\gH_L) = \E_{x_i,\varepsilon_i} \sup_{h\in \gH_L} \frac{1}{n} \sum_{i=1}^n \varepsilon_i h(x_{i})  
\end{equation}
is the Rademacher complexity of function class $\gH_L$ with respect to $n$ samples, and the expectation is taken over $x_i, \varepsilon_i$, $i=1,...,n$.
\end{lemma}

Unfortunately, direct application of such bound in neural networks with a constant $\gamma_0$ will suffer from the so-called \emph{scaling issue}. To see this issue, let's look at the following proposition as a lower bound of Rademacher complexity term. 

\begin{prop}[Lower Bound of the Rademacher Complexity] \label{prop:lowbound}
Consider the networks with activation functions $\sigma$, where we assume $\sigma$ is Lipschitz continuous and there exists $x_0$ such that $\sigma'(x_0) \neq 0$ and $\sigma''(x_0)$ exists. Then for any $L>0$, there holds,
\begin{equation}
\RC_n(\gH_L) \geq CL \E_S{\sqrt{\frac{1}{n}\sum_{i=1}^n\|x_i\|_2}},
\end{equation}
where $C>0$ is a constant that does not depend on $S$. 
\end{prop}

This proposition extends Theorem 3.4 in \cite{BFT17} to general activation functions and multi-class scenario, and the proof is presented in Appendix. 

The scaling issue refers to the fact that if the network Lipschitz $L\to\infty$, by this Lemma the upper bound (\ref{eq:rc-standard}) becomes trivial since $\RC_{n}(\gH_L)\to\infty$. On the other hand, the gradient descent method with logistic regression (cross-entropy) loss \citep{Telgarsky13} and exponential loss (boosting) \citep{soudry2017implicit} will drive weight estimates to approach infinity for max-margin classifiers when the data is linearly separable. In particular, the latter work shows the growth rate of weight estimates is $\log (t)$. As for the deep neural network with cross-entropy loss, the input of last layer is usually be viewed as features extracted from original input. Training the last layer with other layers fixed is exactly a logistic regression, and the feature is linearly separable as long as the training error achieves zero. Therefore, without any normalization, the hypothesis space during training has no upper bound on $L$, and thus the upper bound (\ref{eq:rc-standard}) is useless. 

To solve the scaling issue, in the following we are going to present normalization of margins and restricted Rademacher complexity within unit Lipschitz ball. We are going to see when such bounds are tight enough to predict generalization error based on training data.  

\subsection{Generalization Bounds by Normalized Margins and Restricted Rademacher Complexity}
The first remedy is to restrict our attention on $\gH_{1}$ by normalizing $f$ with its Lipschitz semi-norm $\|f\|_\gF$ or some tight upper bound estimates. Note that a normalized network $\nf=f/C$ has the same mis-classification error as $f$ for all $C>0$. For the choice of $C$, it's hard in practice to directly compute the Lipschitz semi-norm of a network, but instead some approximate estimates on the upper bound $L_f$ in (\ref{eq:Lip}) are available as discussed in Section \ref{app:est}. 

In the sequel, let $\nf=f/L_f$ be the normalized network and $\nh=h/L_f=\zeta(f, y)/L_f = \zeta(\nf, y)\in\gH_1$ be the corresponding normalized hypothesis function from (\ref{eq:H}). Now a simple idea is to regard $\RC_{n}(\gH_{1})$ as a constant when the model complexity is not over-expressive against data, then one can predict the tendency of generalization error via training margin error of the normalized network, that avoids the scaling issue and the exact computation of Rademacher complexity. In the following we present two bounds, with one on normalized margin error bound as the direct application of Lemma \ref{lem:margbound} and the other on quantile margin error bound as the inverse of the former that turns out to be more effective in applications.

\subsubsection{Normalized Margin Error Bound} 
The following theorem states that the probability of normalized test margins than $\gamma_1$ is controlled by the percentage of normalized training margins less than $\gamma_2>\gamma_1$, up to a constant $\RC_{n}(\gH_1)/(\gamma_2-\gamma_1)$ if the Rademacher complexity of unit ball $\RC_{n}(\gH_1)$ is not large. 


\begin{theorem}\label{thm:marg-err}
Given $\gamma_{1}$ and $\gamma_{2}$ such that $\gamma_2>\gamma_{1}\geq 0$ and $\Delta:=\gamma_2 - \gamma_1\geq 0$, for any $\delta>0$, with probability at least $1-\delta$, along the training epoch $t=1,\ldots,T$, the following holds for each $f_t$,
\begin{equation} \label{eq:margin-dis}
    \eP[\zeta(\nf_{t}(x),y) < \gamma_{1}] \leq\eP_n 1[\zeta(\nf_t(x),y)<\gamma_{2}] + \frac{C_{\gH}}{\Delta} + \sqrt{\frac{\log(1/\delta)}{2n}} 
\end{equation}
where $C_{\gH} = 4K\RC_{n}(\gH_1)$. 
\end{theorem}
\begin{remark}
In particular, when we take $\gamma_1=0$ and $\gamma_{2}=\gamma>0$, the bound above becomes,
\begin{equation} \label{eq:marg-err}
    \eP[\zeta(f_t(x),y)< 0] \leq\eP_n [\zeta(\nf_t(x_i),y_i)<\gamma] + \frac{C_{\gH}}{\gamma} + \sqrt{\frac{\log(1/\delta)}{2n}}
\end{equation}
\end{remark}

\begin{remark}
Recently \cite{Poggio_memo91} investigates for normalized networks, the strong linear relationship between cross-entropy training loss and test loss when the training epochs are large enough. However, the bound here is applied to the whole training process for all epoch $t$, that enables us to find early stopping time $t^*$ by looking at change points of $\eP_n 1[\zeta(\nf_t(x),y)<\gamma_{2}]$ in dynamics of high training margin distributions that will be discussed below. 
\end{remark}

Theorem \ref{thm:marg-err} says that, one can bound the normalized test margin distribution $\eP[\zeta(\nf_{t}(x),y) < \gamma_{1}]$ by the normalized training margin distribution $\eP_n [\zeta(\nf_t(x),y)<\gamma_{2}]$. In particular, one hopes to predict the trend of generalization (test) error by choosing $\gamma_{1}=0$ and a proper $\gamma>$ such that the high training margin errors $\eP_n [\zeta(\nf_t(x_i),y_i)<\gamma]$ enjoy a high correlation with test error up to a monotone transform. To achieve this goal, the following facts makes it possible. 
\begin{itemize}
\item First, we do not expect the bound, for example (\ref{eq:marg-err}), is tight for every choice of $\gamma>0$, instead we hope there exists some $\gamma$ such that the training margin error almost changes monotonically with generalization error. This indeed happens when the model complexity is not too much where one can not uniformly enlarge the high training margins. For example, Figure \ref{fig:marerr-spear} below shows the existence of such $\gamma$ when models are not too big by exhibiting rank correlations between training margin error at various $\gamma$ and test error, for a CNN trained on CIFAR10 dataset. Moreover, Figure \ref{fig:bcnn-fixrc} below shows that the training margin error at such a good $\gamma$ successfully recover the tendency of generalization error. 
\item Second, the normalizing factor is not necessarily to be an upper bound of Lipschitz semi-norm. The key point is to prevent the complexity term of the normalized network going to infinity. Since for any constant $c>0$, normalization by $\bar{L} = cL$ works in practice where the constant could be absorbed to $\gamma$, we could ignore the Lipschitz constant introduced by general activation functions in the hidden layers. 
\end{itemize}

However, such a strategy may fail. As shown by Example \ref{exmp:first} with Figure \ref{fig:bcnn-main} above, once the training margin distribution is uniformly improved, dynamic of training margin error fails to capture the change point (minimum) of generalization error in the early stage. This is because when network structure becomes complex and over-expressive enough against the data, the training margin distribution could be more easily improved. In this case the restricted Rademacher complexity $\RC_{n}(\gH_1)$ in Theorem \ref{thm:marg-err} will blow up such that it is invalid to bound the generalization error using merely the training margins, $\eP_n [\zeta(\nf_t(x_i),y_i)<\gamma]$, despite it is reduced in training. This is exactly the same observation in \cite{Breiman99} to doubt the margin theory in boosting type algorithms. More detailed discussions will be given in Section \ref{sect:over} with experiments. 

\subsubsection{Quantile Normalized Margin Error Bound}
A serious limitation of Theorem \ref{thm:marg-err} lies in we must fix a $\gamma$ along the whole training process. In fact, the first term and second term in the bound (\ref{eq:marg-err}) vary in the opposite directions with respect to $\gamma$, and thus it is possible that different $f_{t}$ at different $t$ may prefer different $\gamma$ for a trade-off. \emph{Can we adaptively choose good $\gamma_t$ at different $t$}? 

The answer is \emph{Yes}. In fact as shown in Figure \ref{fig:bcnn-main} (b) of Example \ref{exmp:first} above, while choosing $\gamma$ is to fix an $x$-coordinate section of margin distributions, another direction is to look for a $y$-coordinate section which enables different margins for different $f_t$. This motivates us to define the \emph{quantile margin} below. Let $\hat{\gamma}_{q,f}$ be the \emph{$q^{\textnormal{th}}$ quantile margin} of the network $f$ with respect to sample $S$, {\it i.e.} 
\begin{equation}\label{eq:qmargin}
    \hat{\gamma}_{q,f} = \inf \left\{\gamma: \eP_n 1[\zeta(f(x_i),y_i)\leq \gamma] \geq q \right\}.
\end{equation}

The following theorem bounds the generalization error by the inverse of quantile margins on training data. 
\begin{theorem}\label{thm:qmargin}
Assume the input space is bounded by $M>0$, that is $\|x\|_{2}\leq M,\ \forall x\in\gX$. Given a quantile $q\in [0,1]$, for any $\delta\in (0,1)$ and $\tau>0$, the following holds with probability at least $1-\delta$ for all $f_t$ satisfying $\hat{\gamma}_{q,\nf_t}>\tau$, 
\begin{equation} \label{eq:qmargin}
    \eP[\zeta(f_t(x),y)<0] \leq C_q + \frac{C_\gH}{\hat{\gamma}_{q,\nf_t}} 
\end{equation}
where $C_q= q + \sqrt{\frac{\log(2/\delta)}{2n}} + \sqrt{\frac{\log\log_2(4(M+l)/\tau)}{n}}$ and $C_\gH=8K\RC_n(\gH_{1})$.
\end{theorem}
\begin{remark}
We simply denote $\gamma_{q,t}$ for $\gamma_{q, \nf_t}$ when there is no confusion.
\end{remark}

Compared with the bound (\ref{eq:marg-err}), Theorem \ref{thm:qmargin} bound (\ref{eq:qmargin}) makes it possible to choose $\gamma_t$ varying with $f_t$ and the cost is an additional constant term in $C_q$ as well as the constraint $\hat{\gamma}_{q,t}>\tau$ that typically holds for large enough $q$ in practice. In applications, the stochastic gradient descent method often effectively improves the training margin distributions along with the reduction of training errors, a small enough $\tau$ and large enough $q$ usually meet $\hat{\gamma}_{q,t}>\tau$. Moreover, even with the choice $\tau=\exp(-B)$, constant term $ \sqrt{[\log\log_2(4(M+l)/\tau)]/n}=O(\sqrt{\log B/n})$ is still negligible and thus very little cost is paid in the upper bound. 

In practice, tuning $q\in [0,1]$ is far easier than tuning $\gamma>0$ directly and setting a large enough $q$ usually provides us lots of information about the generalization performance. The quantile margin works effectively when the dynamics of high margin distributions reflect the behaviour of generalization error, e.g. as shown in Figure \ref{fig:bcnn-main}. In this case, after certain epochs of training, the high margins have to be sacrificed to further improve the low margins for reducing the training loss, that typically indicates a possible saturation or overfitting in test error. 

\subsection{Estimate of Normalization Factors}
\label{app:est}

It remains to discuss on how to estimate the Lipschitz constant bound in (\ref{eq:Lip}). Given an operator $W$ associated with a convolutional kernel $w$, i.e. $Wx=w\ast x$,
there are two ways to estimate its operator norm. We begin with the following proposition, part (A) of which is adapted from the continuous version of Young's convolution inequality in $L_p$ space (see Theorem 3.9.4 in \cite{bogachev2007measure}), and part (B) of which is a generalization to multiple channel kernels widely used in convolutional networks nowadays. The proof is presented in Appendix \ref{app:convnorm}.

\begin{prop} \label{prop:convnorm} 
(A) For convolution operator $W$ with kernel $w\in \R^{\Size}$ where $\Size=(\Size_i)_{i=1}^d$ is the $d$-dimensional kernel size, there holds 
\begin{equation} 
\|w\ast x\|_2 \leq \|w\|_1 \|x\|_2. 
\end{equation}
In other words, $\|W\|_\sigma\leq \|w\|_1$. 

(B) Consider a multiple channel convolutional kernel $w\in \R^{\mathrm{C_{out}} \times \mathrm{C_{in}} \times \Size}$ with stride $S$, which maps input signal $x$ of $C_{in}$ channels to the output of $C_{out}$ channels by
\[ (W x)(u,c_{out}) =  [w\ast x](u,c_{out}) := \sum_{v,c_{in}} x(v,c_{in}) w(c_{out},c_{in},u-v), \]
where $x$ and $w$ are assumed of zero-padding outside its support. The following upper bounds hold.
\begin{enumerate}
\item Let $\|w\|_{\infty,\infty,1}:=\max_{i,j} \|w(j,i,\cdot)\|_1$, then
\begin{equation} \label{eq:B1}
\|w\ast x\|_2 \leq \sqrt{\|w\|_1\|w\|_{\infty,\infty,1}}\|x\|_2; 
\end{equation}
\item Let $D:=\prod_i \lceil {\mathrm{Size_i}}/S \rceil$ where $\ceil{t}:=\inf_k\{k\in\mathbb{Z}: k\geq t\}$, then
\begin{equation} \label{eq:B2}
\|w\ast x\|_2 \leq \sqrt{D\|w\|_1 \|w\|_{\infty}}\|x\|_2. 
\end{equation}
\end{enumerate} 
\end{prop}

\begin{remark}
For stride $S=1$, the upper bound (\ref{eq:B1}) is tighter than (\ref{eq:B2}), while for a large stride $S\geq 2$, the second bound (\ref{eq:B2}) might become tighter by taking into acount the effect of stride. 
\end{remark} 

In all these cases, the $\ell_1$-norm of $w$ dominates the estimates, so in the following we will simply call these bounds $\ell_1$-based estimates. Another method is given in \citep{miyato2018spectral} based on power iterations \citep{golub2001eigenvalue}, as a fast numerical approximation for the spectral norm of the operator matrix. We compare the two estimates in Figure \ref{fig:powiter}. It turns out both of them can be used to predict the tendency of generalization error using normalized margins and both of them will fail when the network has large enough expressive power. Although using the $\ell_1$-based estimate is very efficient, the power iteration method may be tighter and have a wider range of predictability.


However, a shortcoming of the power method is that it can not be directly applied to the ResNet blocks. In the remaining of this section, we will particularly discuss the treatment of ResNets. ResNet is usually a composition of the basic blocks shown in Figure \ref{fig:res18-basicblock} with short-cut structure. The following method is used in this paper to estimate upper bounds of spectral norm of such a basic block of ResNet. 
 


\begin{figure}[htbp]
\centering
\includegraphics[width=.7\textwidth]{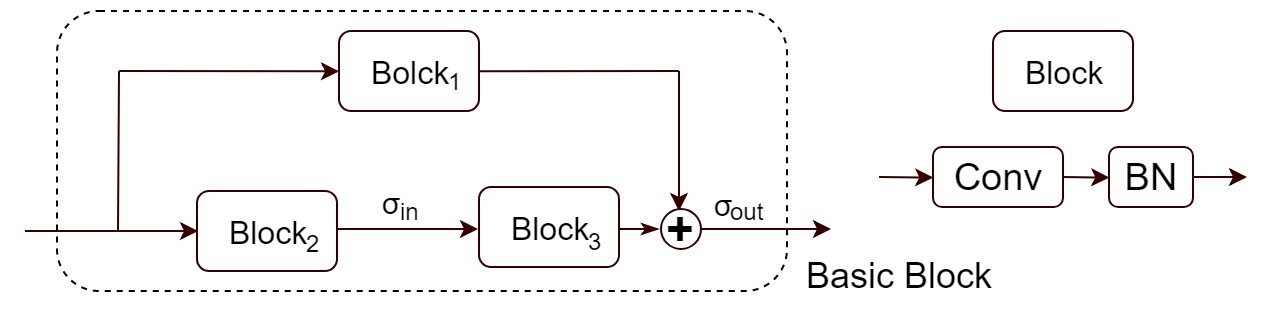}
\caption{A basic block in ResNets used in this paper. The shortcut consists of one block with convolutional and batch-normalization layers, while the main stream has two blocks. ResNets are constructed as a cascading of several basic blocks of various sizes.}\label{fig:res18-basicblock}
\end{figure}

\begin{itemize}
    \item[(a)] Convolution layer: its operator norm can be bounded either by the $\ell_1$-based estimate or by power iteration above.
    \item[(b)] Batch Normalization (BN): in training process, BN normalizes samples by $x^{+}=(x-\mu_B)/\sqrt{\sigma^2_B+\epsilon}$, where $\mu_B, \sigma^2_B$ are mean and variance of batch samples, while keeping an online averaging as $\hat{\mu}$ and $\hat{\sigma}^2$. Then BN rescales $x^{+}$ by estimated parameters $\hat{\alpha}, \hat{\beta}$ and output $\hat{x}=\hat{\alpha}x^{+}+\hat\beta$. Therefore the whole rescaling of BN on the kernel tensor $w$ of the convolution layer is $\hat{w}=w\hat\alpha/\sqrt{\hat\sigma^2+\epsilon}$ and its corresponding rescaled operator is $\|\hat{W}\|_{\sigma} = \|W\|_\sigma\hat\alpha/\sqrt{\hat{\sigma}^2+\epsilon}$. 
    \item[(b)] Activation and pooling: their Lipschitz constants can be known a priori, e.g.  $L_\sigma=1$ for ReLU and hence can be ignored. In general, $L_\sigma$ can not be ignored if they are in the shortcut as discussed below. 
    \item[(d)] Shortcut: In residue net with basic block in Figure \ref{fig:res18-basicblock}, one has to treat the mainstream $(\textnormal{Block}_2, \textnormal{Block}_3)$ and the shortcut $\textnormal{Block}_1$ separately. Since $\|f+g\|_{\gF}\leq\|f\|_{\gF}+\|g\|_{\gF}$, in this paper we take the Lipschitz upper bound by $L_{\sigma_{\textnormal{out}}}(\|\hat{W}_1\|_{\sigma} +  L_{\sigma_{\textnormal{in}}}\|\hat{W}_2\|_{\sigma}\|\hat{W}_3\|_{\sigma})$, where $\|\hat{W}_i\|_\sigma$ denotes a spectral norm estimate of BN-rescaled convolutional operator $W_i$. In particular $L_{\sigma_{\textnormal{out}}}$ can be ignored since all paths are normalized by the same constant, while $L_{\sigma_{\textnormal{in}}}$ can not be ignored due to its asymmetry.
\end{itemize}

\section{Experimental Results}\label{sect:exp}

The spirit of the following experiments is to show, \emph{when and how, the margin bound above could be used to numerically predict the tendency of generalization or test error along the training path?} We are going to show examples of both success and failure.  

\subsection{Networks and Datasets}
The networks and datasets used in the experiments are introduced in brief here. For the network, our illustration Example \ref{exmp:first} is based on a simple convolutional neural network whose architecture is shown in Figure \ref{fig:basic-structrue} (more details in Appendix Figure \ref{fig:cnns}), called \emph{basic CNN($c$)} here with $c$ channels that will be specified in different experiments below. Basically, it has five convolutional layers of $c$ channels at each, followed by batch normalization and ReLU, as well as a fully connected layer in the end. Furthermore, we consider various popular networks in applications, including AlexNet \citep{krizhevsky2012imagenet}, VGG-16 \citep{simonyan2014very} and ResNet-18 \citep{HZRS16}. For the dataset, we consider CIFAR10, CIFAR100 \citep{krizhevsky2009learning} and Mini-ImageNet \citep{vinyals2016matching}.

\begin{figure}[htbp]
\centering
\includegraphics[width=.7\textwidth]{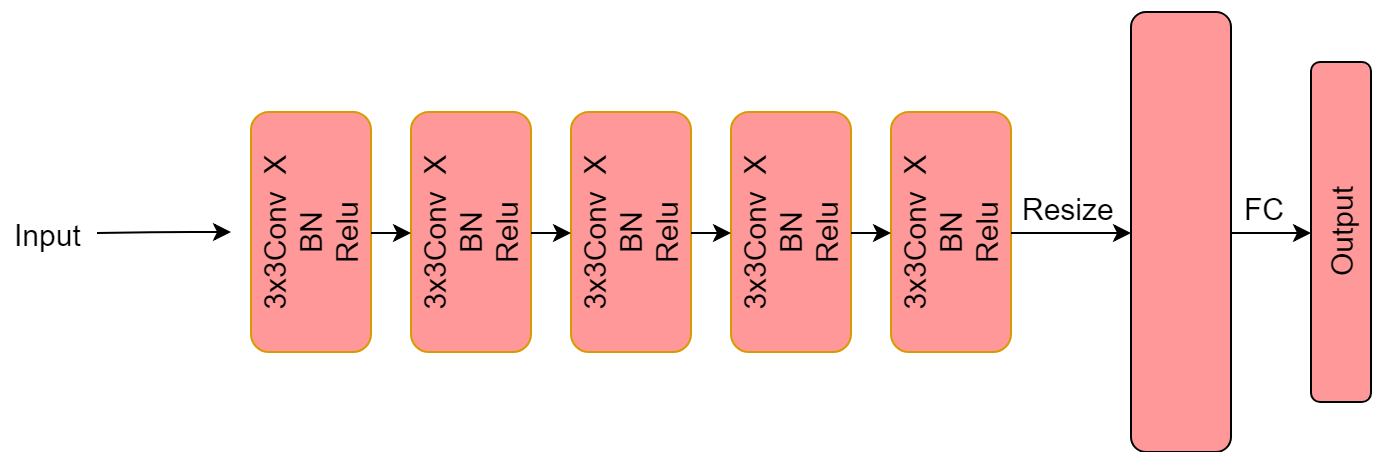}
\caption{Illustration of the architecture of basic CNN.}\label{fig:basic-structrue}
\end{figure}

\subsection{Success: Similar Phase Transitions in Training and Test Margin Dynamics}\label{sect:marg-err}
In this section, we show that when the expressive power of models are comparable to data complexity, the dynamics of training margin distributions and that of test margin distributions share similar phase transitions which enables us to predict generalization (test) error utilizing the theorems in this paper. In this experiment, we are going to demonstrate when there is a nearly monotone relationship between training margin error and test margin error such that Theorem \ref{thm:marg-err} and Theorem \ref{thm:qmargin} can be applied to predict the tendency of generalization (test) error. 
     
First let's consider training a basic CNN(50) on CIFAR10 dataset with and without random noise. The relations between test error and \emph{training margin error} $e_\gamma(\nf(x),y)$ with $\gamma=9.8$, \emph{inverse quantile margin} $1/\hat{\gamma}_{q,t}$ with $q=0.6$ are shown in Figure \ref{fig:bcnn-fixrc}. In this simple example where the network is small and the dataset is simple, the bounds (\ref{eq:margin-dis}) and (\ref{eq:qmargin}) show a good prediction power: they stop either near the epoch of sufficient training without noise (Left, original data) or before an overfitting occurs with noise (Right, 10 percents label corrupted).

\begin{figure}[htbp]
\centering
\includegraphics[width=.48\textwidth]{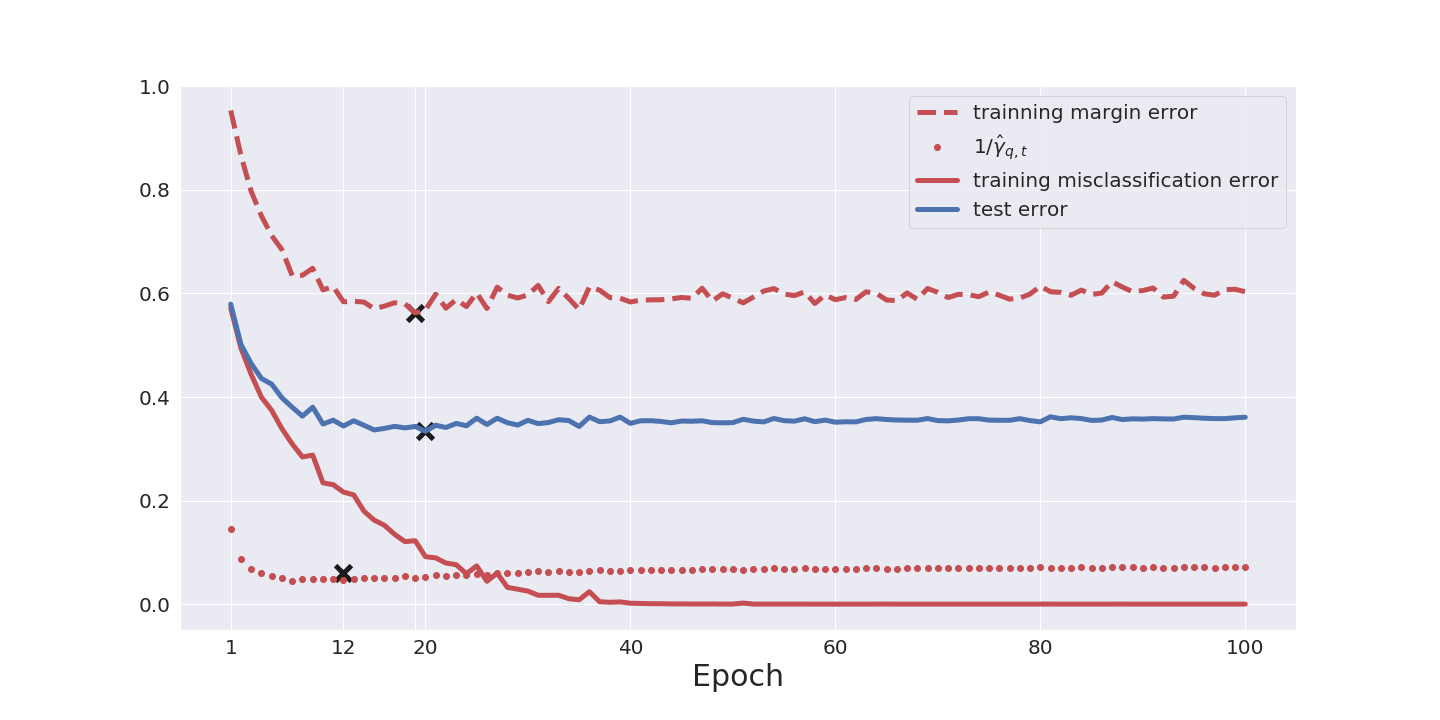}
\includegraphics[width=.48\textwidth]{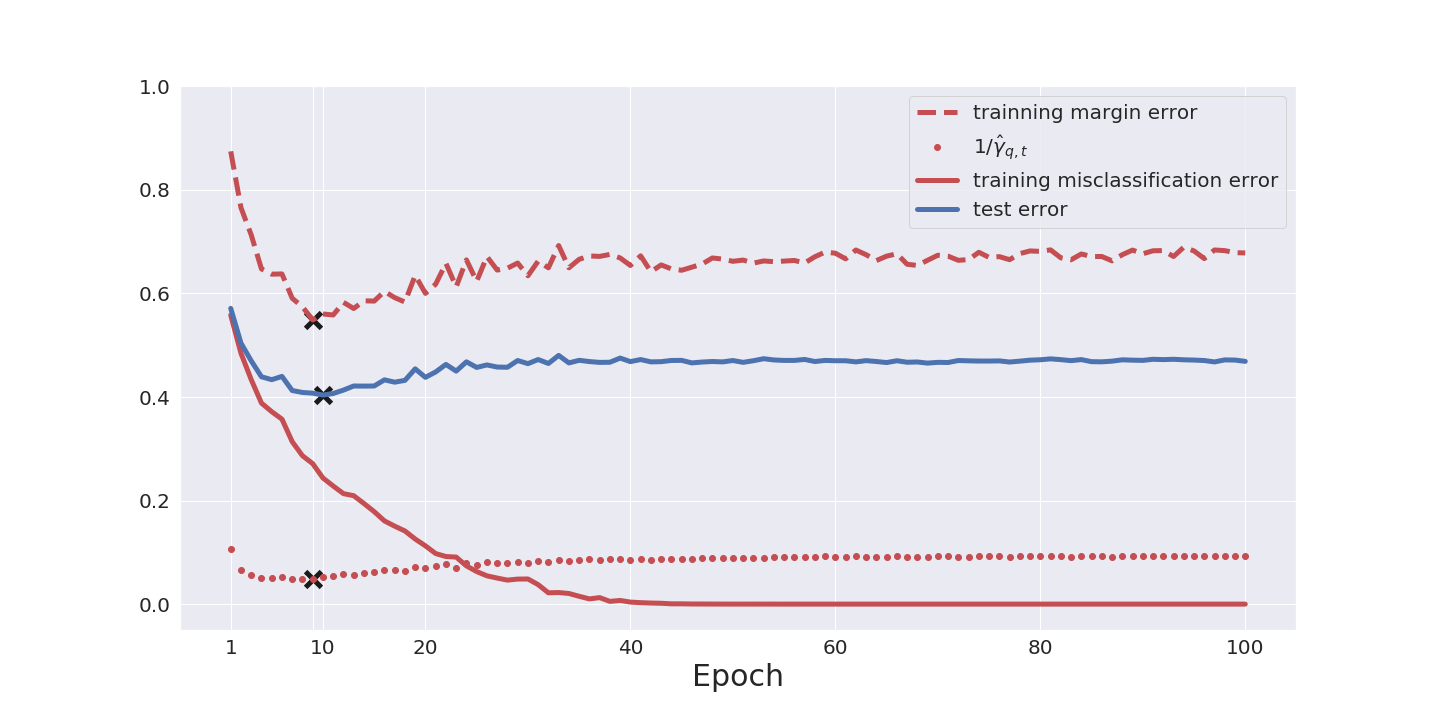}
\caption{Success examples. Net structure: basic CNN (50). Dataset: Original CIFAR10 (Left) and CIFAR10 with 10 percents label corrupted (Right). In each figure, we show training error (red solid), training margin error $\gamma=9.8$ (red dash) and inverse quantile margin (red dotted) with $q=0.6$ and generalization error (blue solid). The marker ``x" in each curve indicates the global minimum along epoch $1,\ldots,T$. Both training margin error and inverse quantile margin successfully predict the tendency of generalization error. }\label{fig:bcnn-fixrc}
\end{figure}

\begin{figure}[htbp]
\centering
\includegraphics[width=.45\textwidth]{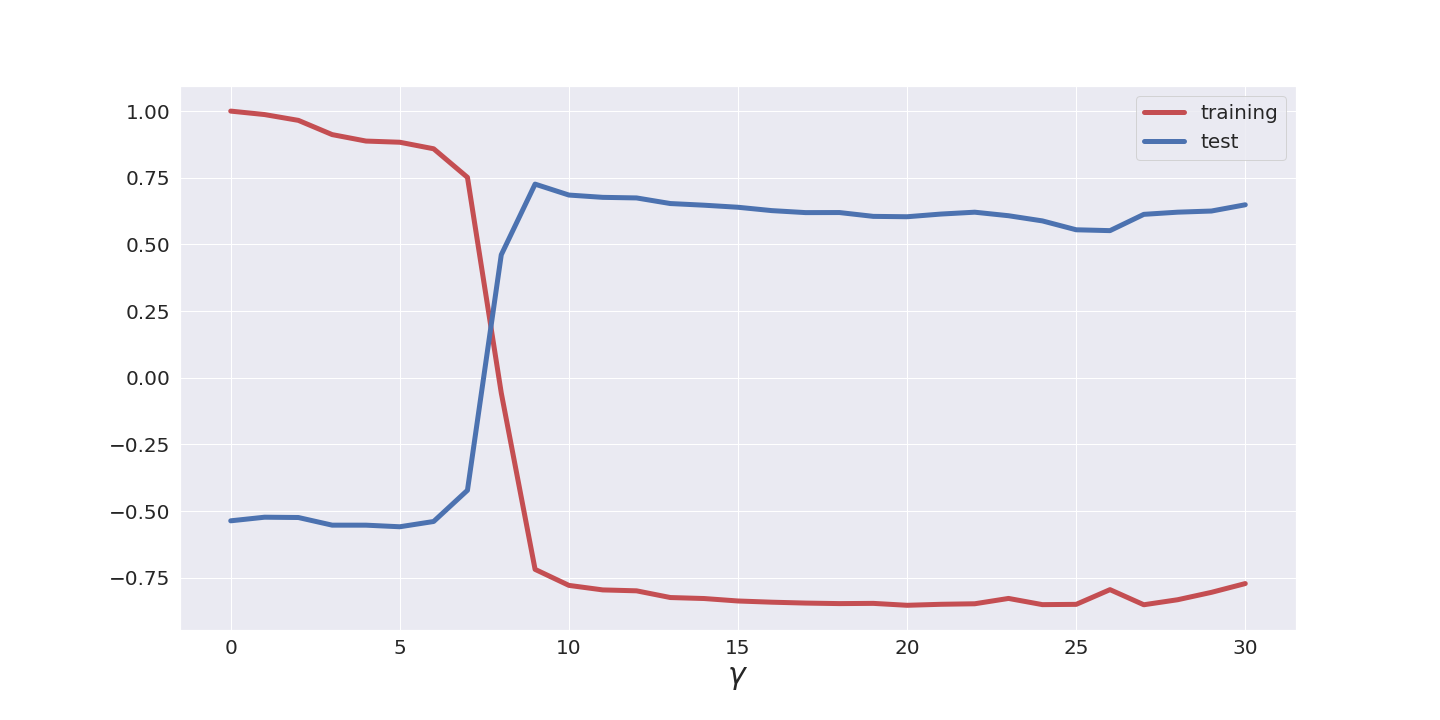}
\includegraphics[width=.45\textwidth]{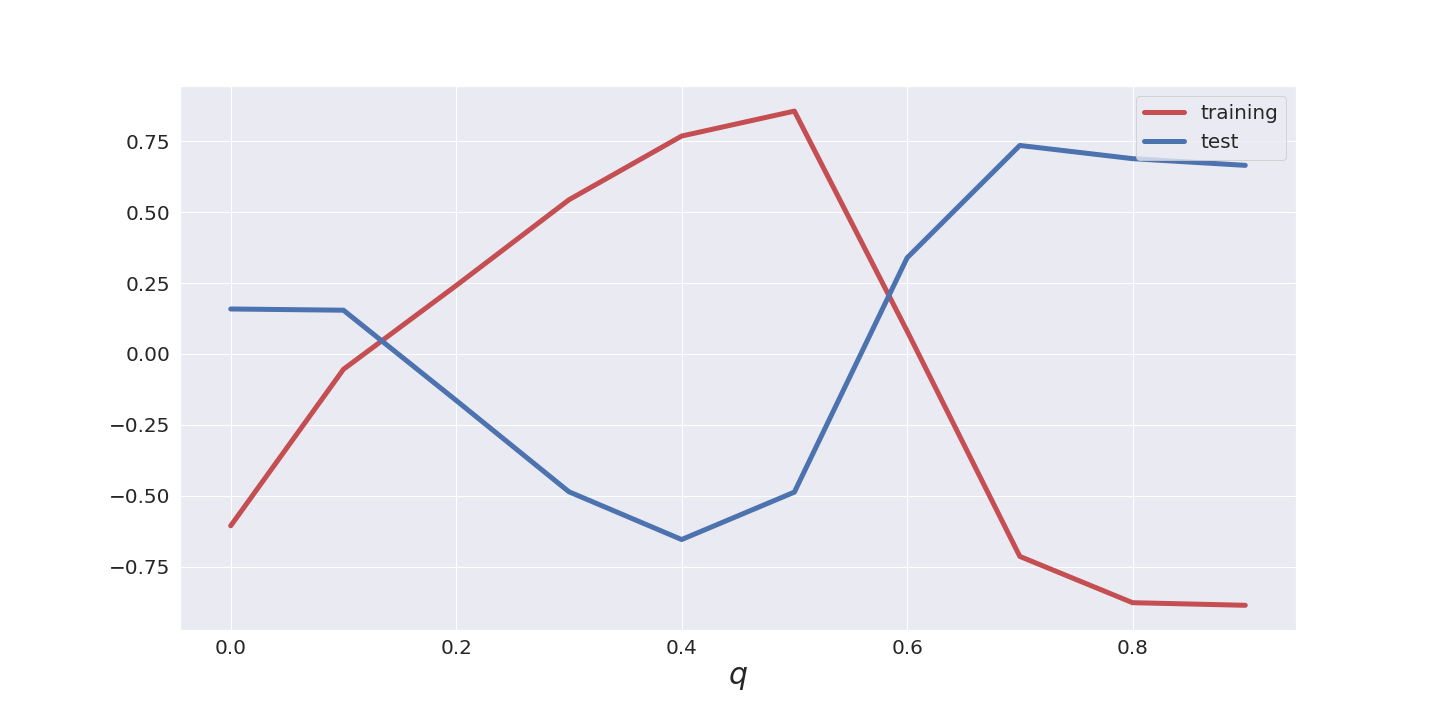}
\includegraphics[width=.45\textwidth]{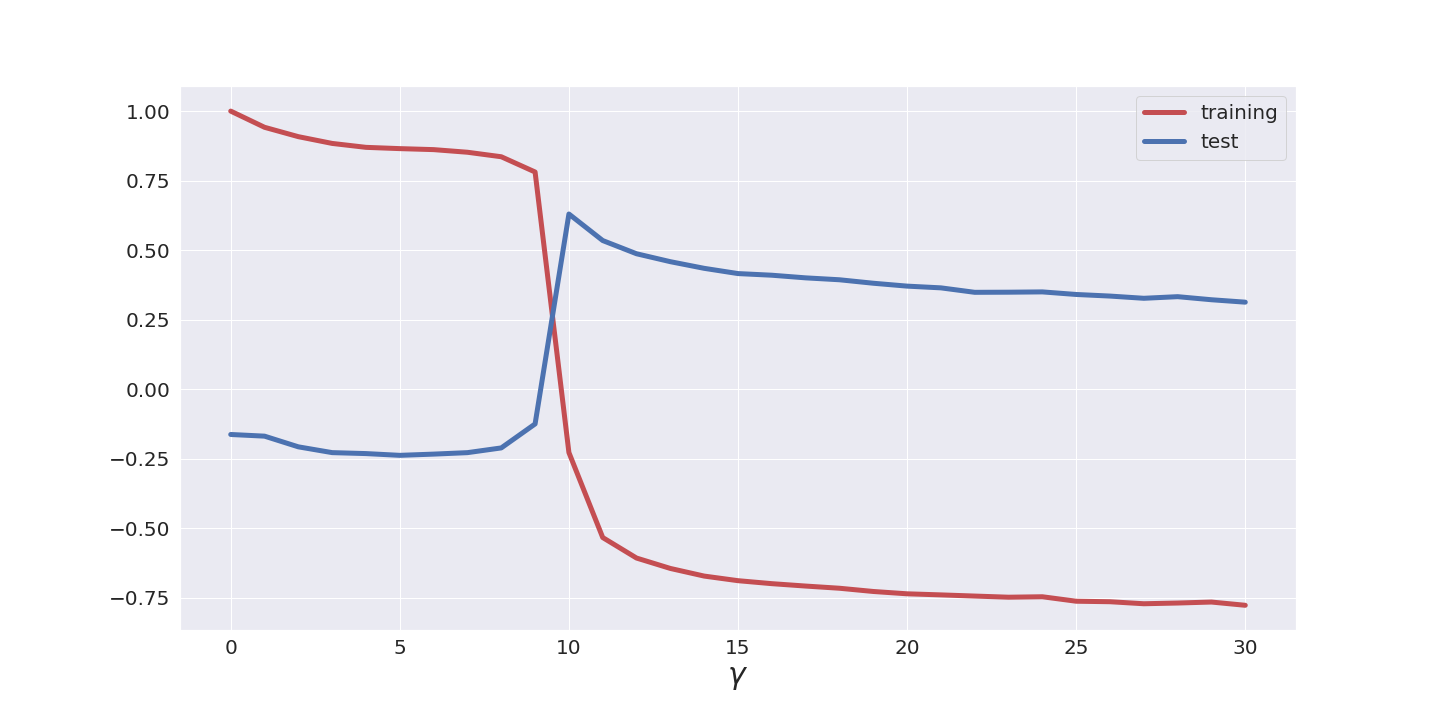}
\includegraphics[width=.45\textwidth]{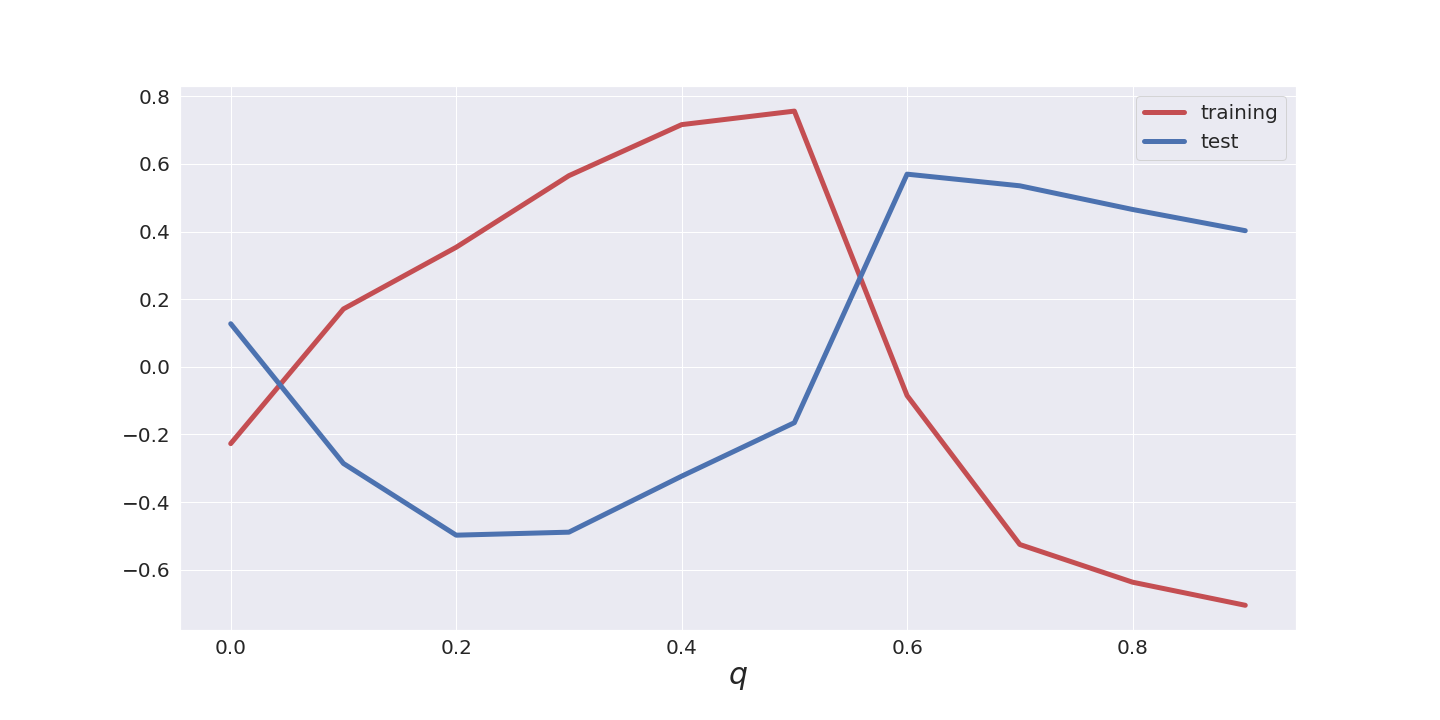}
\caption{Spearman's $\rho$ and Kendall's $\tau$ rank correlations between training (or quantile) margins and training errors, as well as training (or quantile) margins and test errors, at different $\gamma$ (or $q$, respectively). Net structure: Basic CNN(50). Dataset: CIFAR10. Top: Spearman's $\rho$ rank correlation. Bottom: Kendall's $\tau$ rank correlation. Left: Blue curves show rank correlations between training margin error and test (generalization) error, while Red curves show that between the training margin error and training error, at different $\gamma$. Right: Blue curves show rank correlations between inverse quantile margin and test error, and Red curves show that between inverse quantile margin and training error, at different $q$. Both Spearman's $\rho$ and Kendall's $\tau$ show qualitatively the same phenomenon that dynamics of large margins are closely related to the test errors in the sense that they have similar trends marked by large rank correlations. On the other hand, small margins are close to training errors in trend.}\label{fig:marerr-spear}
\end{figure}

Why does it work in this case? Here are some detailed explanations on its mechanism. The training margin error ($\eP_n [\zeta(\nf_t(x_i),y_i)<\gamma]$) and the inverse quantile margin ($1/\hat{\gamma}_{q,t}$) are both closely related to the dynamics of training margin distributions. Figure \ref{fig:bcnn-main} (b) actually shows that the dynamics of training margin distributions undergo a phase transition: while the low margins have a monotonic increase, the high or large margins undergo a phase transition from increase to decrease, indicated by the red arrows. Therefore different choices of $\gamma$ for the linear bounds (\ref{eq:margin-dis}) (a parallel argument holds for $q$ in (\ref{eq:qmargin})) will have different effects. In fact, the training margin error with a small $\gamma$ is close to the training error, while that with a large $\gamma$ is close to test error. Figure \ref{fig:marerr-spear} shows such a relation using rank correlations (in terms of Spearman-$\rho$ and Kendall-$\tau$\footnote{The Spearman's $\rho$ and Kendall's $\tau$ rank correlation coefficients measure how two variables are correlated up to a monotone transform and a larger correlation means a closer tendency.}) between training margin errors (or inverse quantile margins) and training errors, as well as training margin errors (or inverse quantile margins) and test errors, for each $\gamma$ (or $q$, respectively). In these plots one sees that the dynamics of large margins have a similar trend to the test errors, while small margins are close to training errors in rank correlations. Therefore for a good prediction, one can choose a large enough $\gamma=9.8$ (or $q=6.8$, respectively) at the peak point of rank correlation curve between training margins and test errors. Under such choices, the epoch when the phase transition above happens is featured with a \emph{cross-over} in dynamics of training margin distributions in Figure \ref{fig:bcnn-main} (b), and lives near the optima of the training margin error curve.

Although both the training margin error ($\eP_n [\zeta(\nf_t(x_i),y_i)<\gamma]$) and the inverse quantile margin ($1/\hat{\gamma}_{q,t}$) can be used here to successfully predict the trend of test (generalization) error, the latter can be more powerful in our studies. In fact, dynamics of the inverse quantile margins can adaptively select $\gamma_t$ for each $f_t$ without access to the complexity term. Unlike merely looking at the training margin error with a fixed $\gamma$, quantile margin bound (\ref{eq:qmargin}) in Theorem \ref{thm:qmargin} shows a stronger prediction power than (\ref{eq:marg-err}) and is even able to capture more local optima. In Figure \ref{fig:bcnn-qmar-twoloc}, the test error curve has two valleys corresponding to a local optimum and a global optimum, and the quantile margin curve with $q=0.95$ successfully identifies both. However, if we consider the dynamics of training margin errors, it's rarely possible to recover the two valleys at the same time since their critical thresholds $\gamma_{t_{1}}$ and $\gamma_{t_{2}}$ are different. Another example of ResNet-18 is given in Figure \ref{fig:res-qmar-twoloc} in Appendix.

\begin{figure}[htbp]
\centering
\includegraphics[width=.48\textwidth]{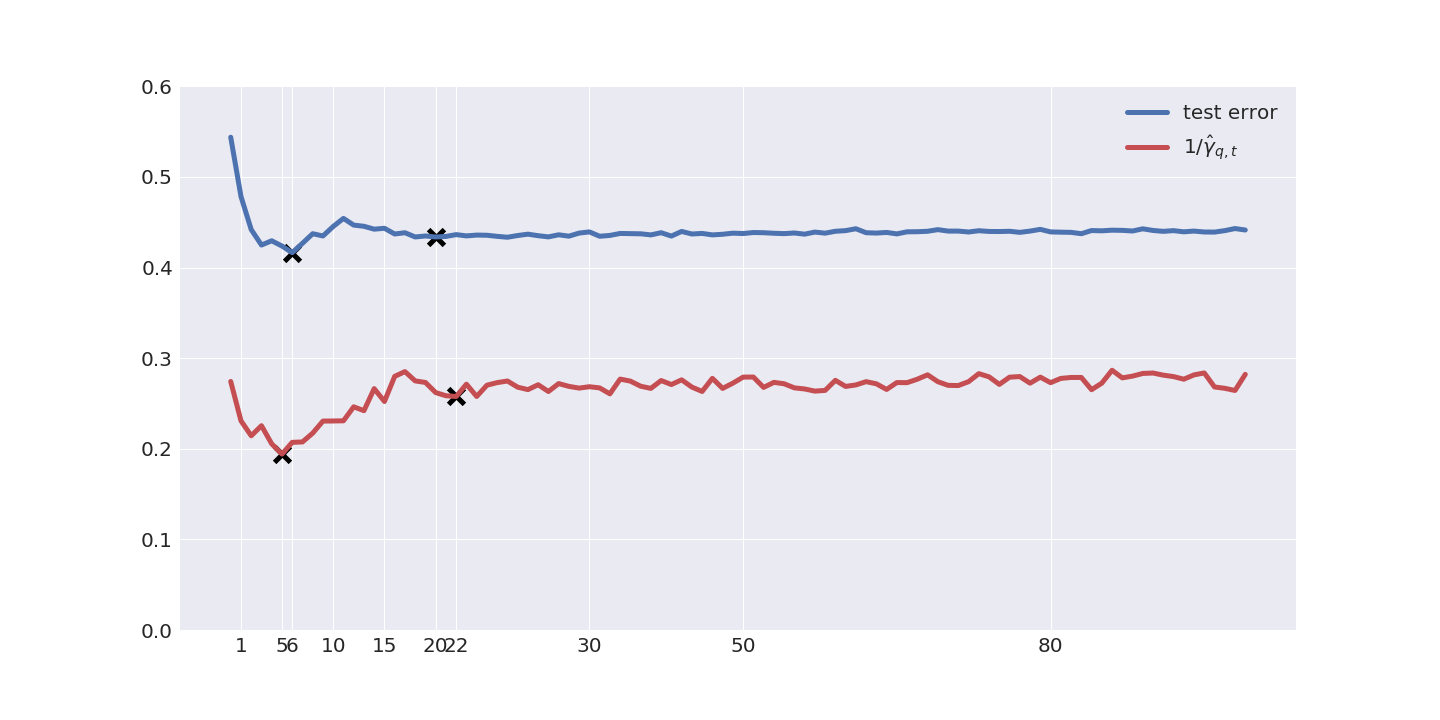}
\includegraphics[width=.48\textwidth]{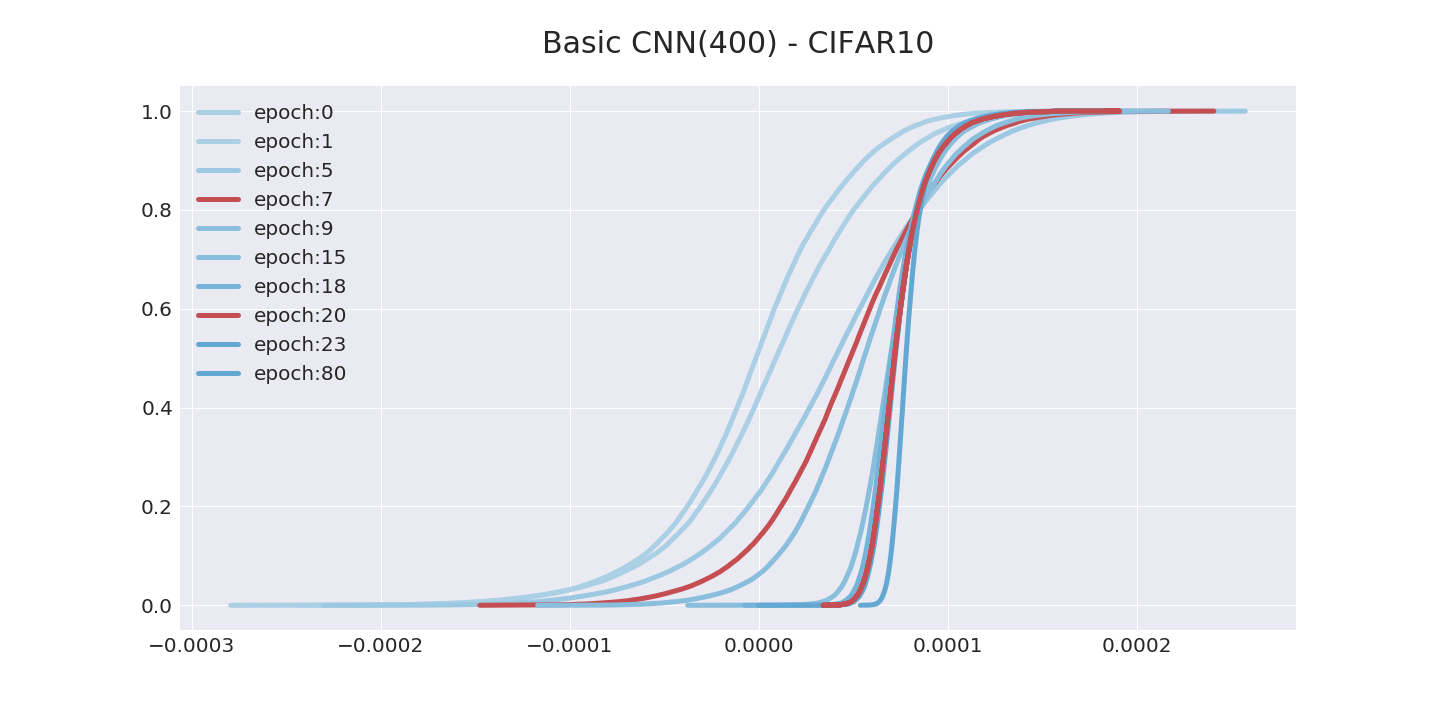}
\caption{Inverse quantile margin. Net structure: CNN(400). Dataset: CIFAR10 with 10 percents label corrupted. Left: the dynamics of test error (blue) and inverse quantile margin with $q=0.95$ (red). Two local minima are marked by ``x" in each curve. Right: dynamics of training margin distributions, where two distributions in red color correspond to when the two local minima occur. The inverse quantile margin successfully captures two local minima of test error.}\label{fig:bcnn-qmar-twoloc}
\end{figure}

In a summary, when training and test margin dynamics share similar phase transitions, both theorems we developed can be used to predict test (generalization) error via normalized training margins, even leaving us data-dependent early stopping rule to avoid overfitting when data is noisy. However, below we shall see a different scenario when training and test margin dynamics are of distinct phase transitions, such a prediction fails as Breiman's dilemma.

\subsection{Failure: Distinct Phase Transitions in Margin Dynamics and Breiman's Dilemma}\label{sect:over}
In this part, when model complexity arbitrarily increase to be over-expressive against the dataset, the training margins can be monotonically improved, while high test margin dynamics undergoes a distinct phase transition of decrease-increase. In this case the prediction power of training margin based bounds is lost and overfitting may set in. This exhibits Breiman's dilemma in neural networks. 

We conduct three sets of experiments in the following. 

\subsubsection{Experiment I: Basic CNNs on CIFAR10} 

\begin{figure}[H]
\centering
\includegraphics[width=.32\textwidth]{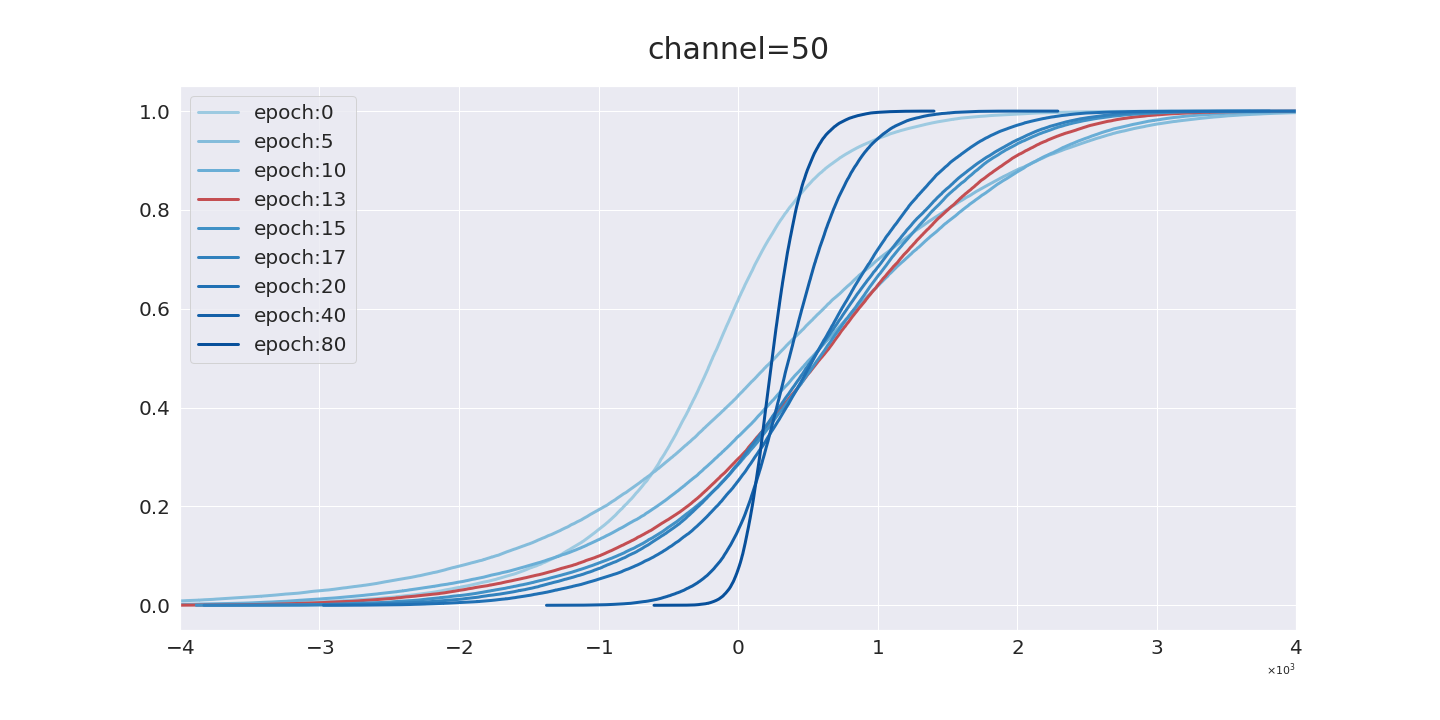}
\includegraphics[width=.32\textwidth]{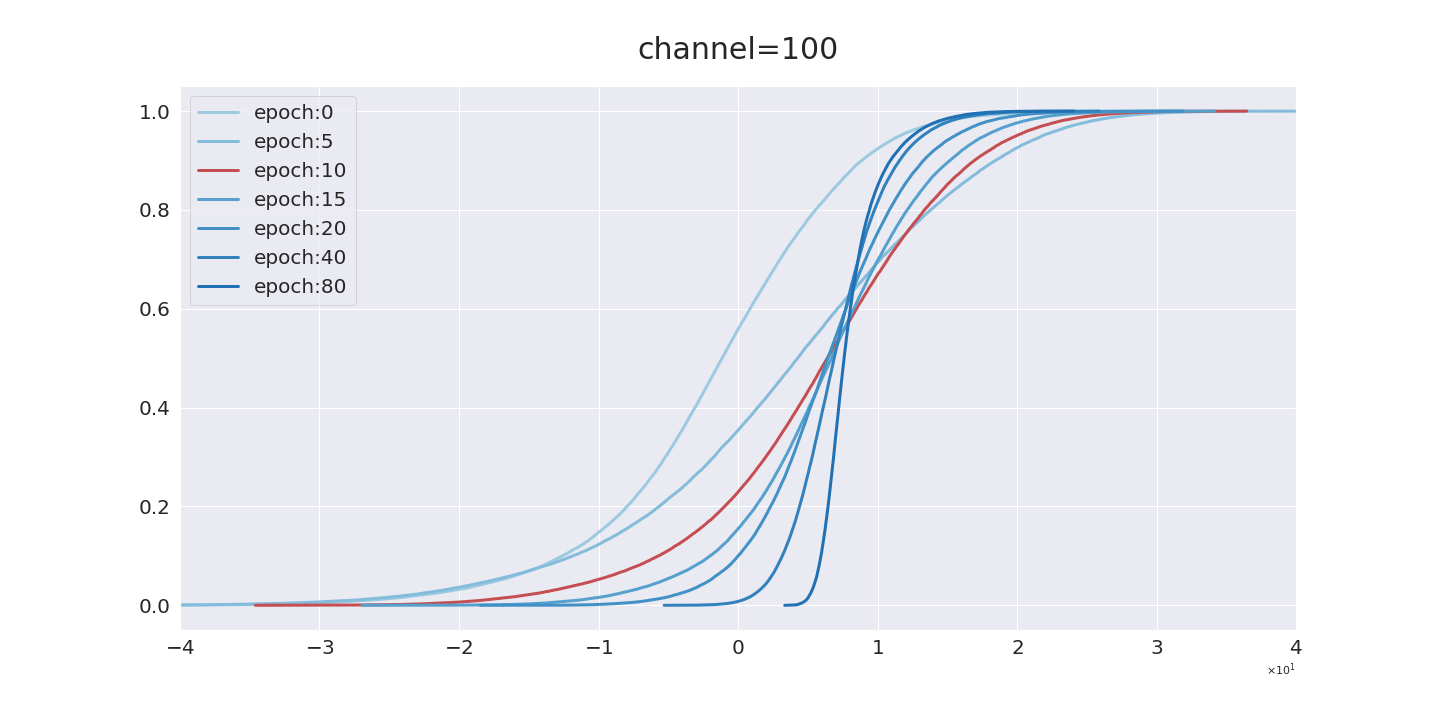}
\includegraphics[width=.32\textwidth]{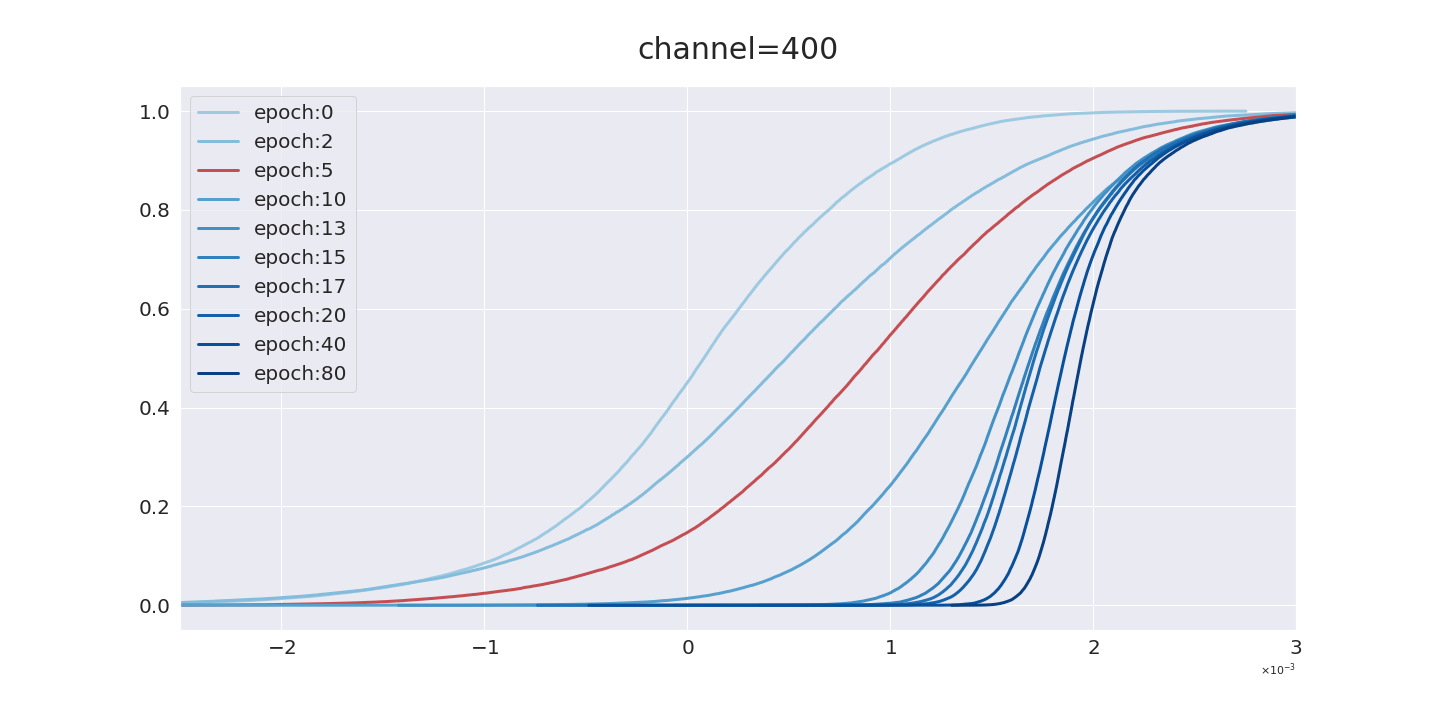}
\includegraphics[width=.32\textwidth]{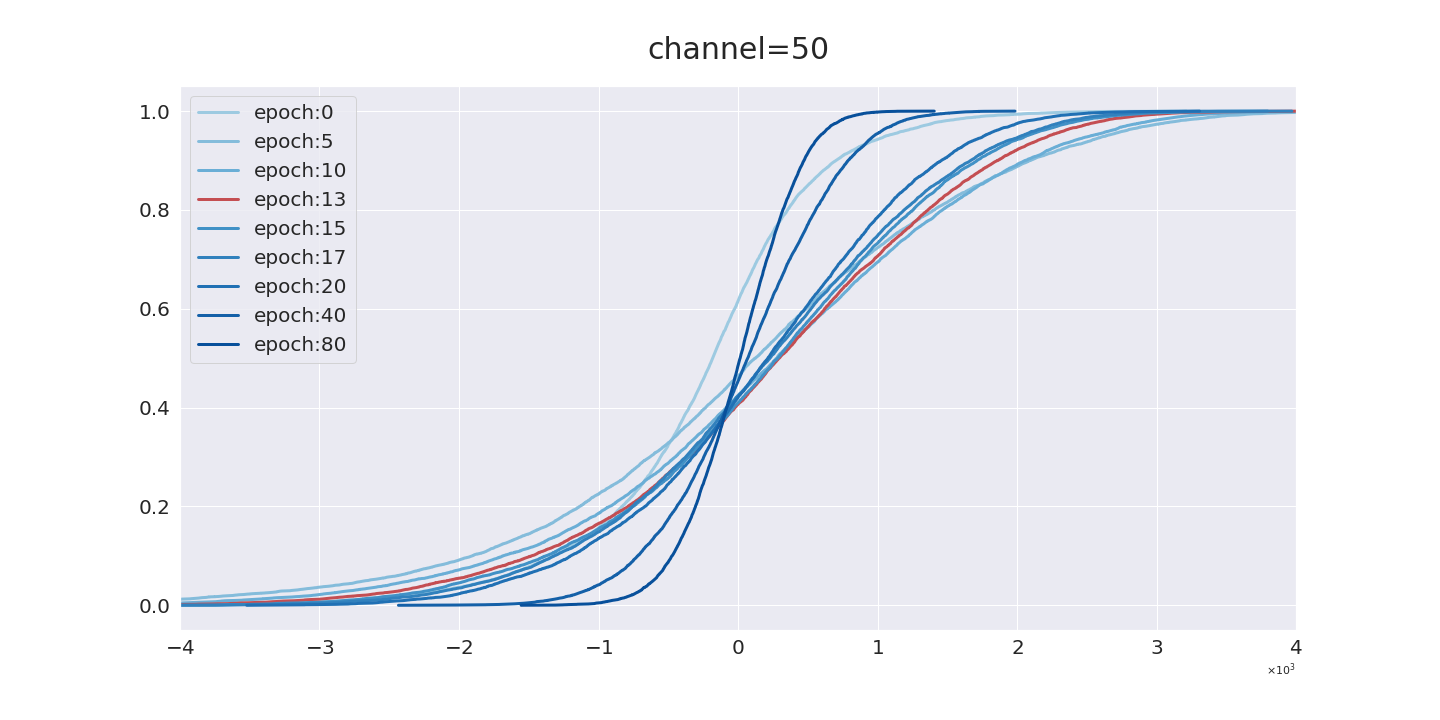}
\includegraphics[width=.32\textwidth]{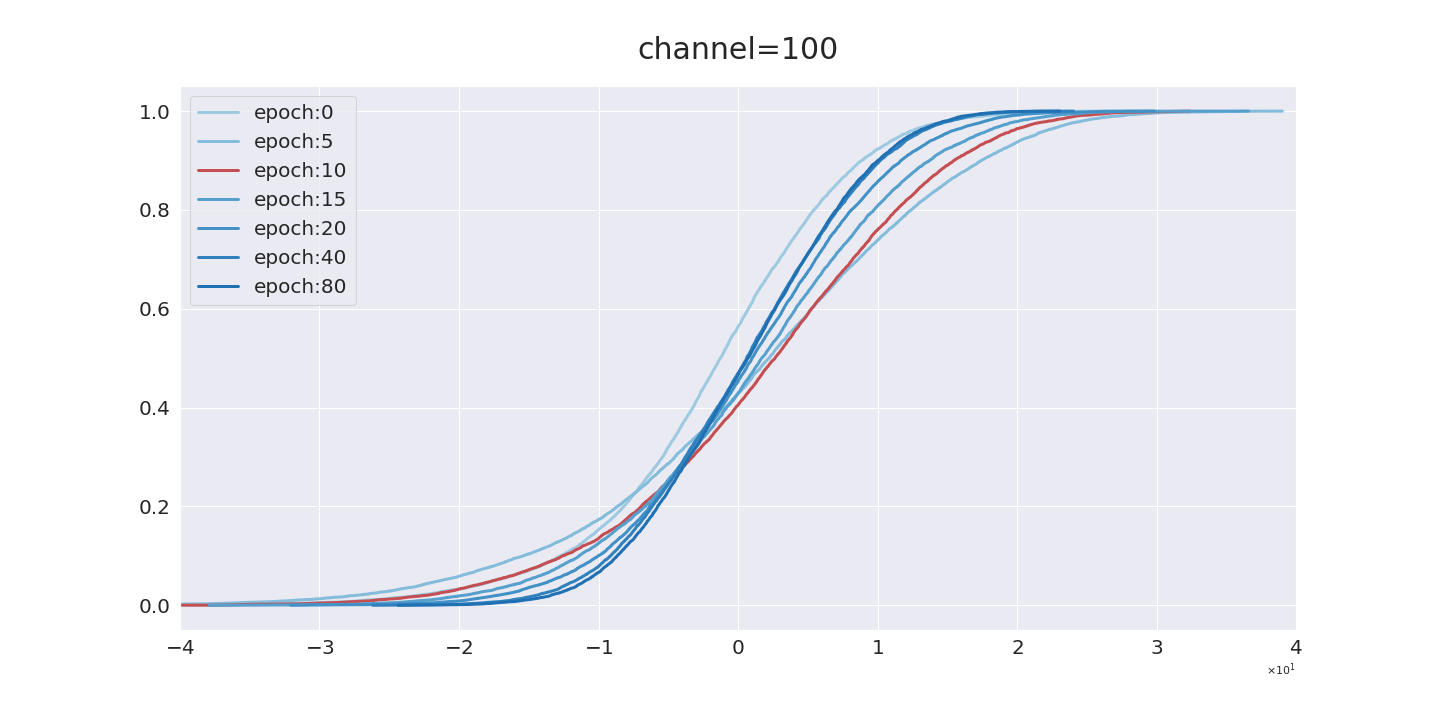}
\includegraphics[width=.32\textwidth]{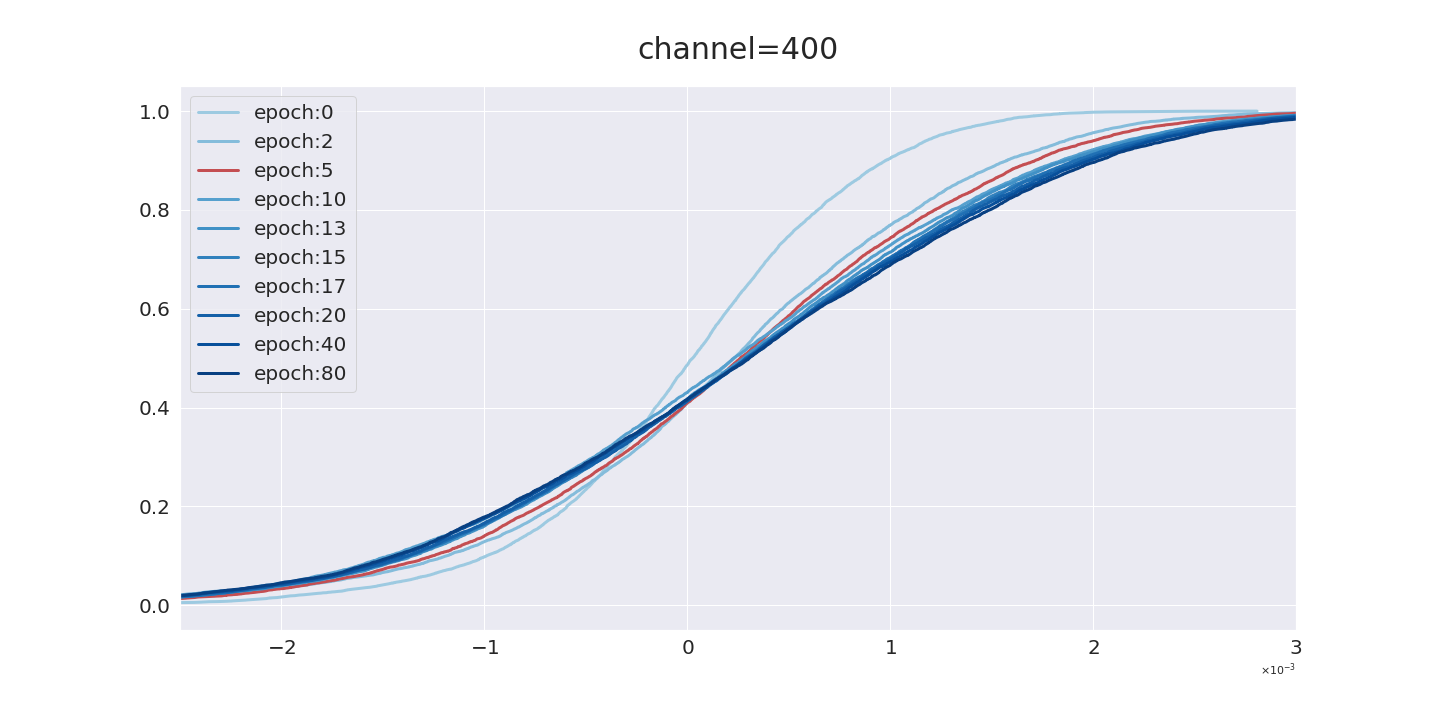}
\includegraphics[width=.32\textwidth]{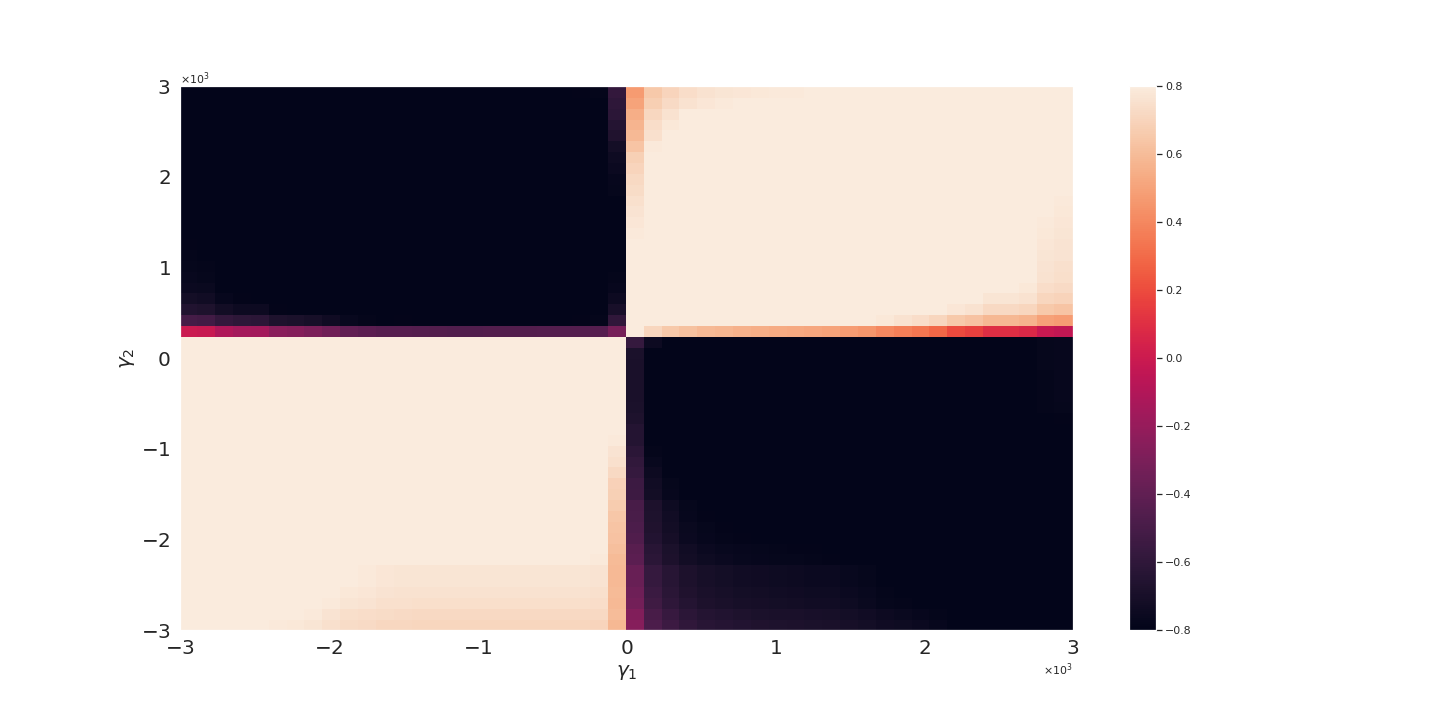}
\includegraphics[width=.32\textwidth]{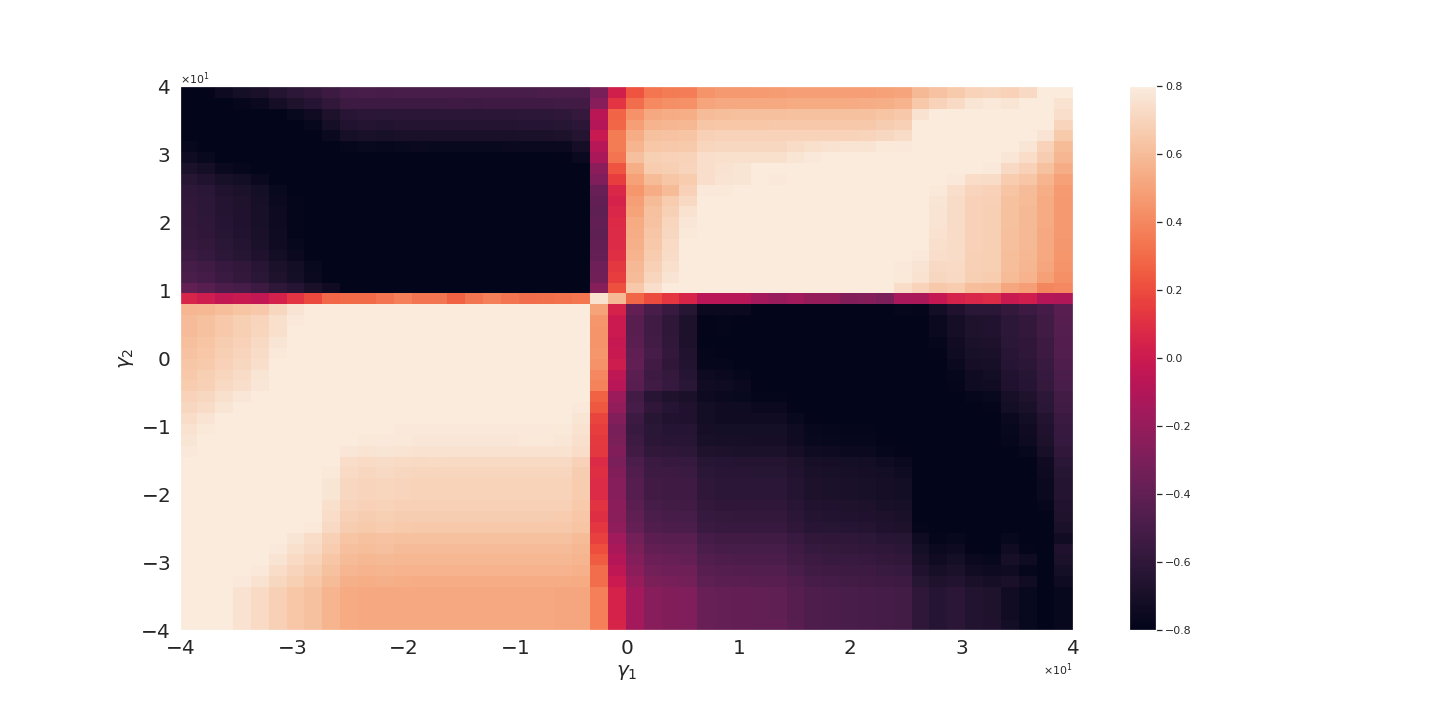}
\includegraphics[width=.32\textwidth]{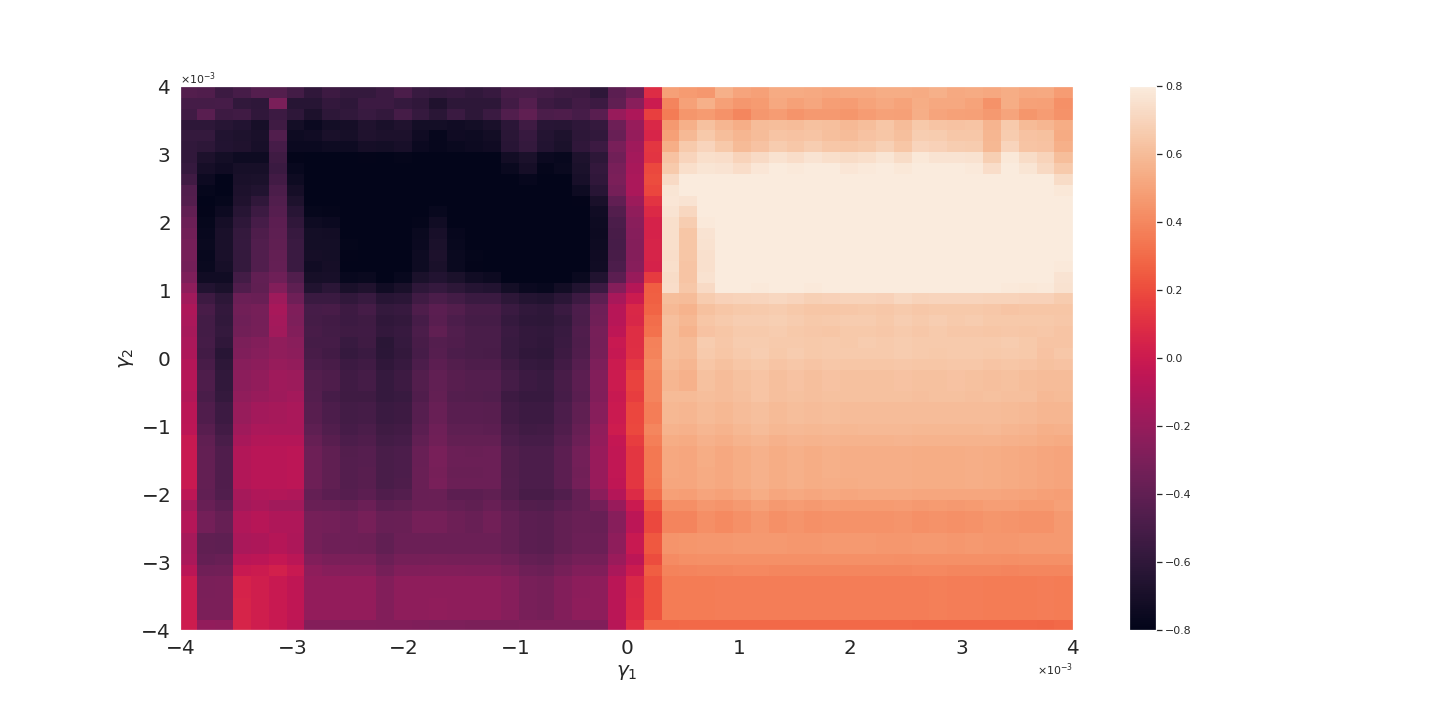}
\caption{Breiman's Dilemma I: comparisons between dynamics of test margin distributions and training margin distributions. Net structure: Basic CNN(50) (Left), Basic CNN(100) (Middle), Basic CNN(400) (Right). Dataset: CIFAR10 with 10 percent labels corrupted. First row: evolutions of training margin distributions. Second row: evolutions of test margin distributions. Third row: heatmaps are Spearman-$\rho$ rank correlation coefficients between dynamics of training margin error ($\eP_n[e_{\gamma_2}(\nf(x_i),y_i)]$) and dynamics of test margin error ($\P[e_{\gamma_1}(\nf_t(x),y)]$) drawn on the $(\gamma_1,\gamma_2)$ plane. CNN(50) and CNN(100) share similar phase transitions in training and test margin dynamics while CNN(400) does not. When model becomes over-representative to dataset, training margins can be monotonically improved while test margins can not be, losing the predictability.}\label{fig:phase-fixdata}
\end{figure}

In the first experiment shown in Figure \ref{fig:phase-fixdata}, we fix the dataset to be CIFAR10 with 10 percent of labels randomly permuted, and gradually increase the channels from basic CNN(50) to CNN(400). For CNN(50) (\#(parameters) is 92,610) and CNN(100) (\#(parameters) is 365,210), both training margin dyamics and test margin dynamics share a similar phase transition during training: small margins are monotonically improved while large margins are firstly improved then dropped afterwards. The last row in Figure \ref{fig:phase-fixdata} shows the heatmaps as Spearman-$\rho$ rank correlations between these two dynamics drawn in $\gamma_1$-$\gamma_2$ plane. The block diagonal structures in the rank correlation heatmaps illustrates such a similarity in phase transitions. To be specific, small (or large) margins in both training margins and test margins share high level rank correlations marked by diagonal blocks in light color, while the difference between small and large margins are marked by off-diagonal blocks in dark color. Particularly at $\gamma_1=0$, the test (generalization) error dynamics can be predicted using large training margins, as their rank correlations are high. 

However, as the channel number increases to CNN(400) (\#(parameters) is 5,780,810), the dynamics of the training margin distributions becomes a monotone improvement without the phase transition above. This phenomenon is not a surprise since with a strong representation power, the whole training margin distribution can be monotonically improved without sacrificing the large margins. On the other hand, the generalization or test error can not be monotonically improved. The heatmap of rank correlations between training and test margin dynamics thus exhibits such a distinction in phase transitions by changing the block diagonal structure above to double column blocks for CNN(400). In particular, for $\gamma_1\leq 0$, test margin dynamics have low rank correlations with all training margin dynamics as they are of different phase transitions in evolutions. As a result, one CAN NOT predict test error at $\gamma=0$ using training margin dynamics.

\subsubsection{Experiment II: CNN(400) and ResNet-18 on CIFAR100 and Mini-ImageNet}
In the second experiment shown in Figure \ref{fig:phase-fixnet}, we compare the normalized margin dynamics of training CNN(400) and ResNet-18 on two different datasets, CIFAR100 and Mini-ImageNet. CIFAR100 is more complex than CIFAR10, but less complex than Mini-ImageNet. It shows that: (a) CNN(400) does not have an over-expressive power on CIFAR100, whose normalized training margin dynamics exhibits a phase transition -- a sacrifice of large margins to improve small margins during training; (b) ResNet-18 does have an over-expressive power on CIFAR100 by exhibiting a monotone improvement on training margins, but loses such a power in Mini-ImageNet with phase transitions of training margin dynamics. 

\begin{figure}[H]
\centering
\includegraphics[width=.32\textwidth]{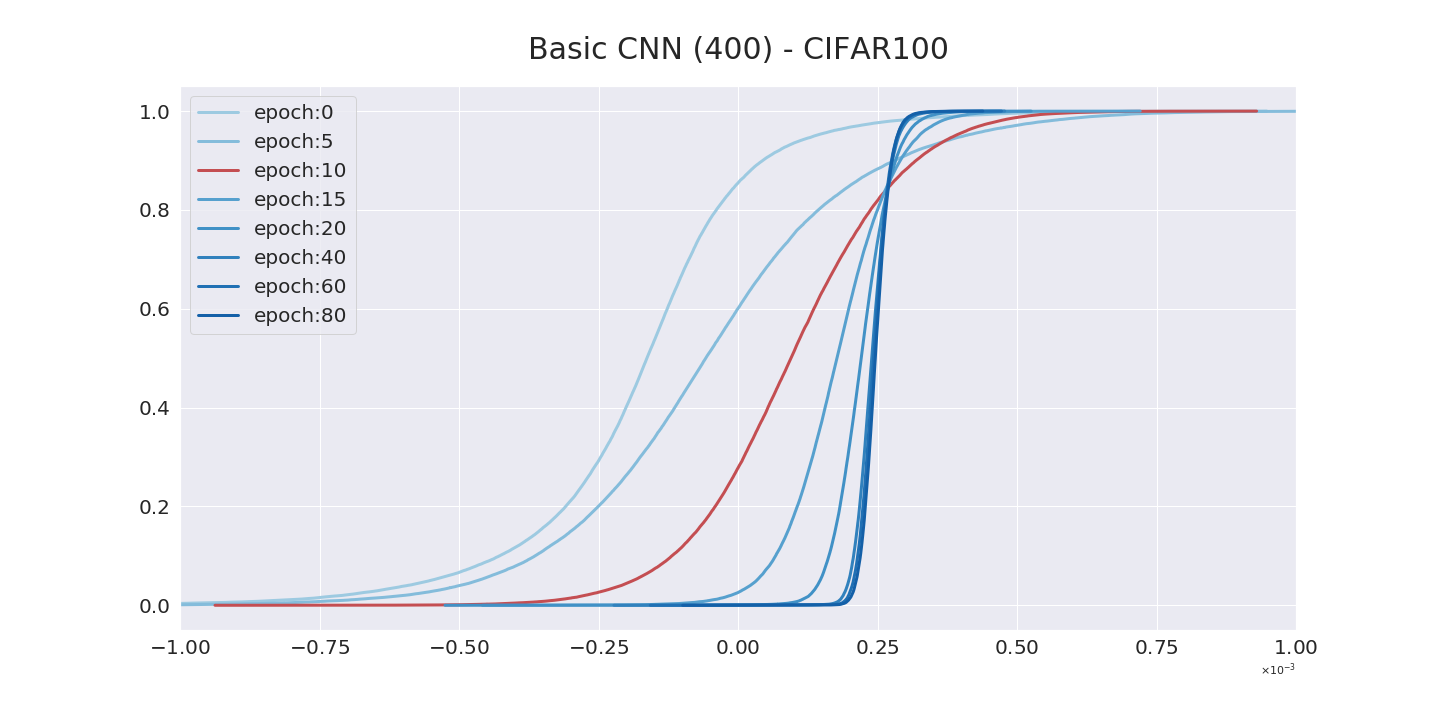}
\includegraphics[width=.32\textwidth]{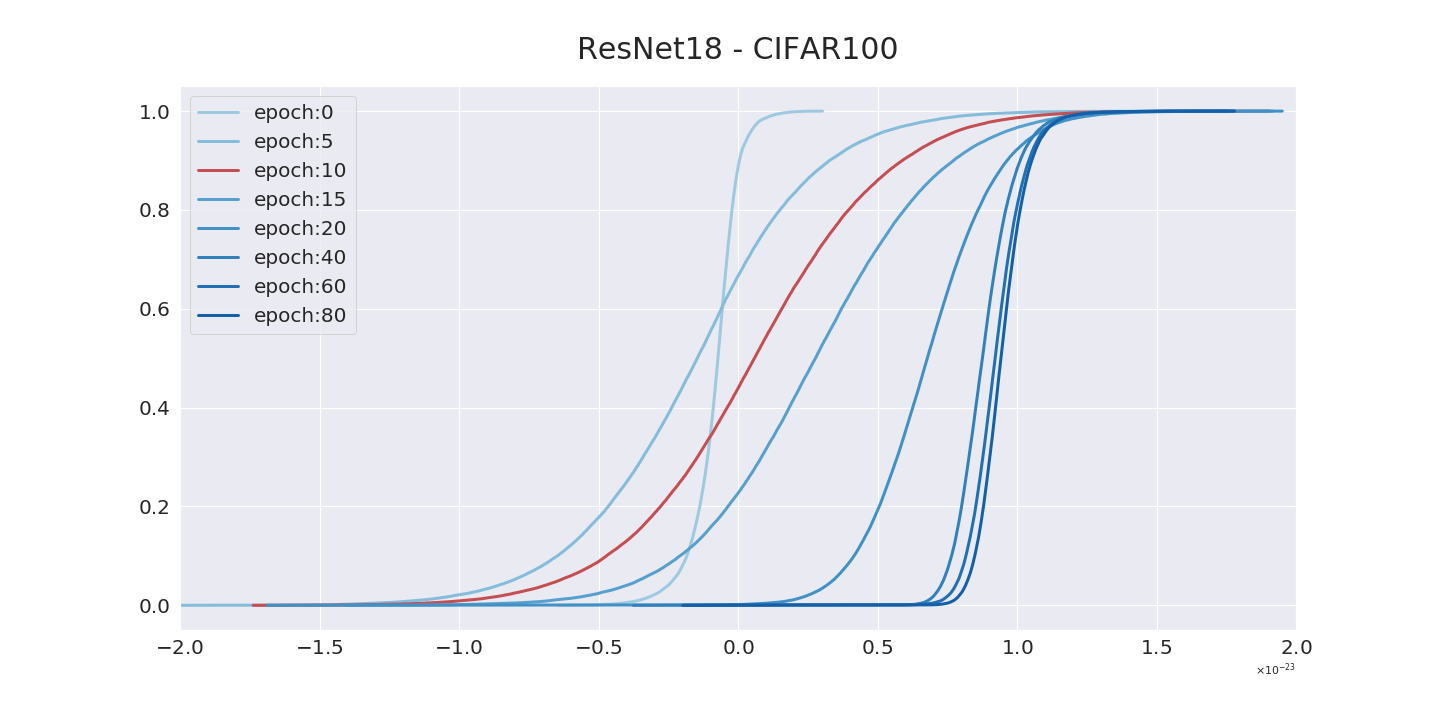}
\includegraphics[width=.32\textwidth]{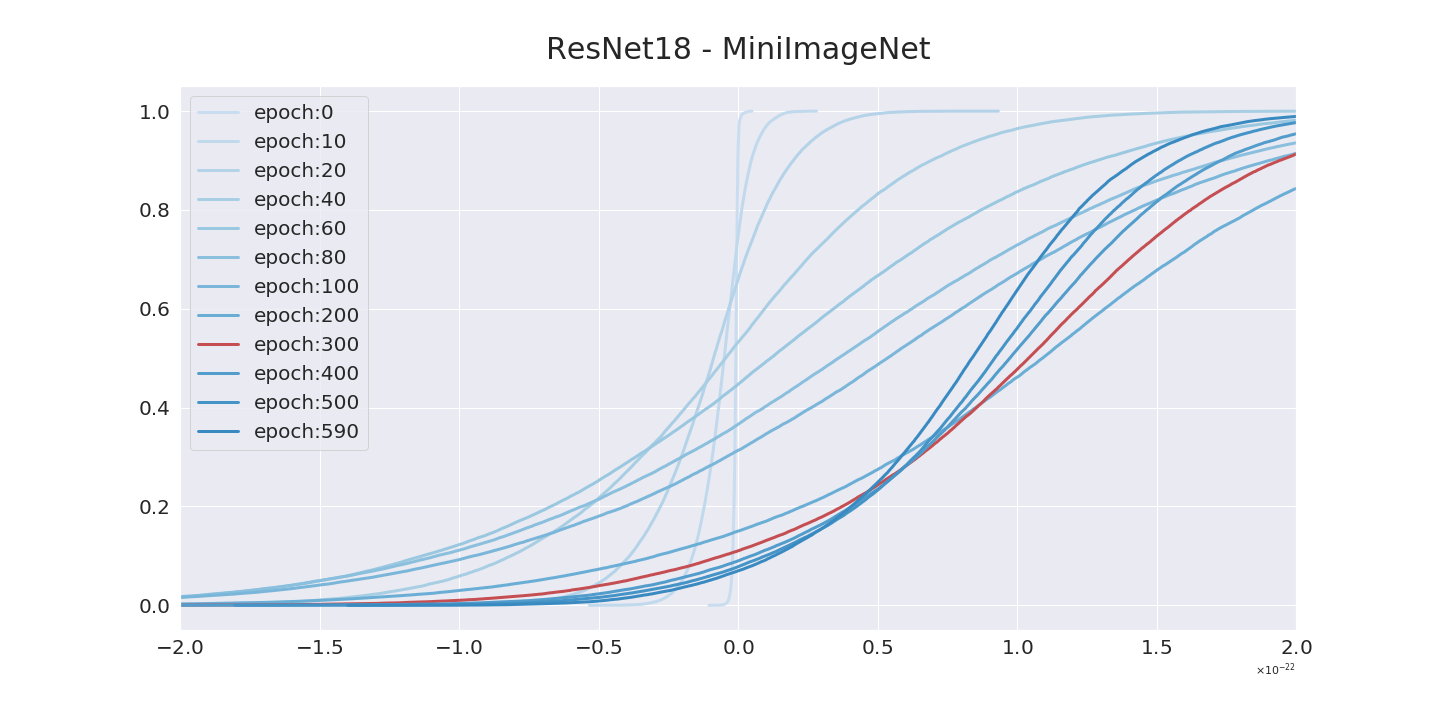}
\caption{Breiman's Dilemma II. Net structure: Basic CNN(400) (Left), ResNet-18 (Middle, Right). Dataset: CIFAR100 (Left, Middle), Mini-ImageNet (Right) with 10 percent labels corrupted. With a fixed network structure, we further explore how the complexity of dataset influences the margin dynamics. Taking ResNet-18 as an example, margin dynamics on CIFAR100 doesn't have any cross-over (phase transition), but on Mini-Imagenet a cross-over occurs.}\label{fig:phase-fixnet}
\end{figure}

From this experiment, one can see that simply counting the number of parameters and samples can not tell us if the model and data complexities are over-representative or comparable. Instead, phase transitions of margin dynamics provide us a tool to investigate their relationship. CNN(400) (5.8M parameters) has a too much expressive power for the simplest CIFAR10 dataset such that the training margins can be monotonically improved during training; but CNN(400)'s expressive power seems comparable to the more complex CIFAR100. Similarly, the more complex model ResNet-18 (11M parameters) has a too much expressive power for CIFAR100, but seems comparable to Mini-ImageNet.

\subsubsection{Comparisons of Basic CNNs, AlexNet, VGG16, and ResNet-18 in CIFAR10/100 and Mini-ImageNet}
In this part, we collect comparisons of various networks on CIFAR10/100 and Mini-ImageNet dataset. Figure \ref{fig:pract-examples} shows both success and failure cases with different networks and datasets. In particular, the predictability of generalization error based on Theorem \ref{thm:marg-err} and Theorem \ref{thm:qmargin} can be rapidly observed on the third column of Figure \ref{fig:pract-examples}, the heatmaps of rank correlations between training margin dynamics and test margin dynamics. On one hand, one can use the training margins to predict the test error as shown in the first column of Figure \ref{fig:pract-examples}. In these cases, model complexity is comparable to data complexity such that the training margin dynamics share similar phase transitions with test margin dynamics, indicated by block diagonal structures in rank correlations (e.g. CNN(100) - CIFAR10, AlexNet - CIFAR100, AlexNet - MiniImageNet, VGG16 - MiniImageNet, and ResNet-18 - MiniImageNet). On the other hand, such a prediction fails when models become over-expressive against datasets such that the training margin dynamics undergo different phase transitions to test margin dynamics, indicated by the lost of block diagonal structures in rank correlations (e.g. CNN(400)- CIFAR10, ResNet-18 - CIFAR100, VGG16 - CIFAR100).

As we have shown, phase transitions of margin dynamics play a central role in characterizing the trade-off between model expressive power and data complexity, hence the predictability of generalization error by our theorems. If one tries hard to improve training margins by arbitrarily increasing the model complexity, the training margin distributions can be monotonically enlarged but it may lead to overfitting. This phenomenon is not unfamiliar to us, since Breiman has pointed out that the improvement of training margins is not enough to guarantee a small generalization or test error in the boosting type algorithms \citep{Breiman99}. Now again we find the same phenomenon ubiquitous in deep neural networks. In this paper, the inspection of the trade-off between expressive power of models and complexity of data via phase transitions of margin dynamics provides us a new perspective to study the Breiman's dilemma in applications. 

\begin{figure}[H]
\centering

\includegraphics[width=.32\textwidth]{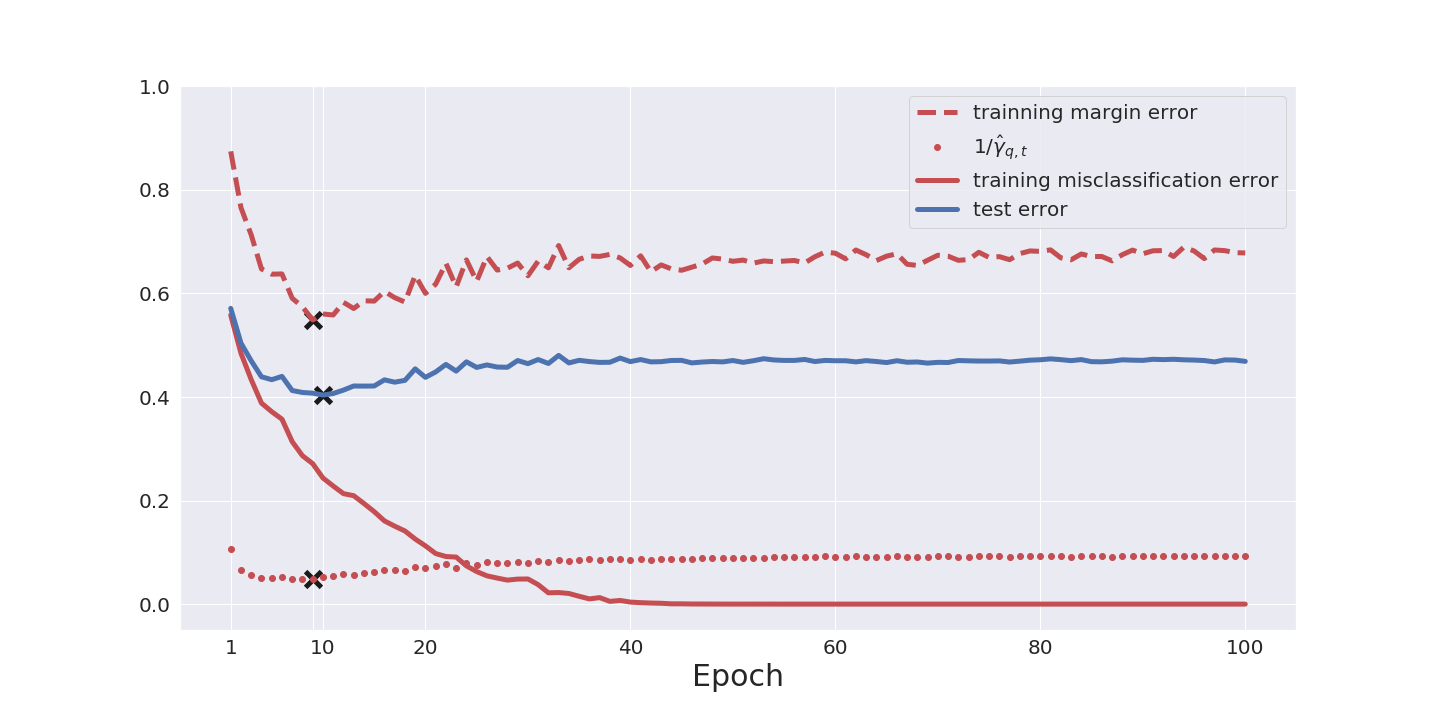}
\includegraphics[width=.32\textwidth]{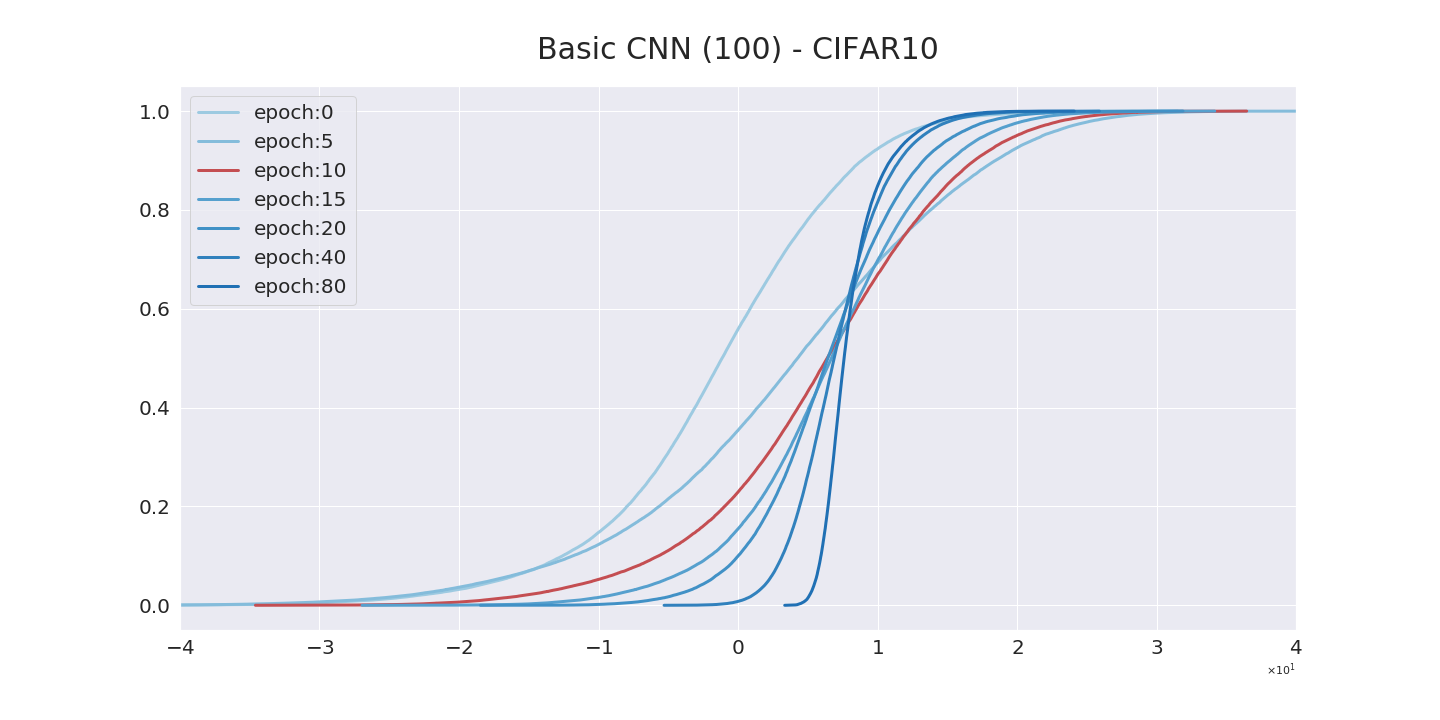}
\includegraphics[width=.32\textwidth]{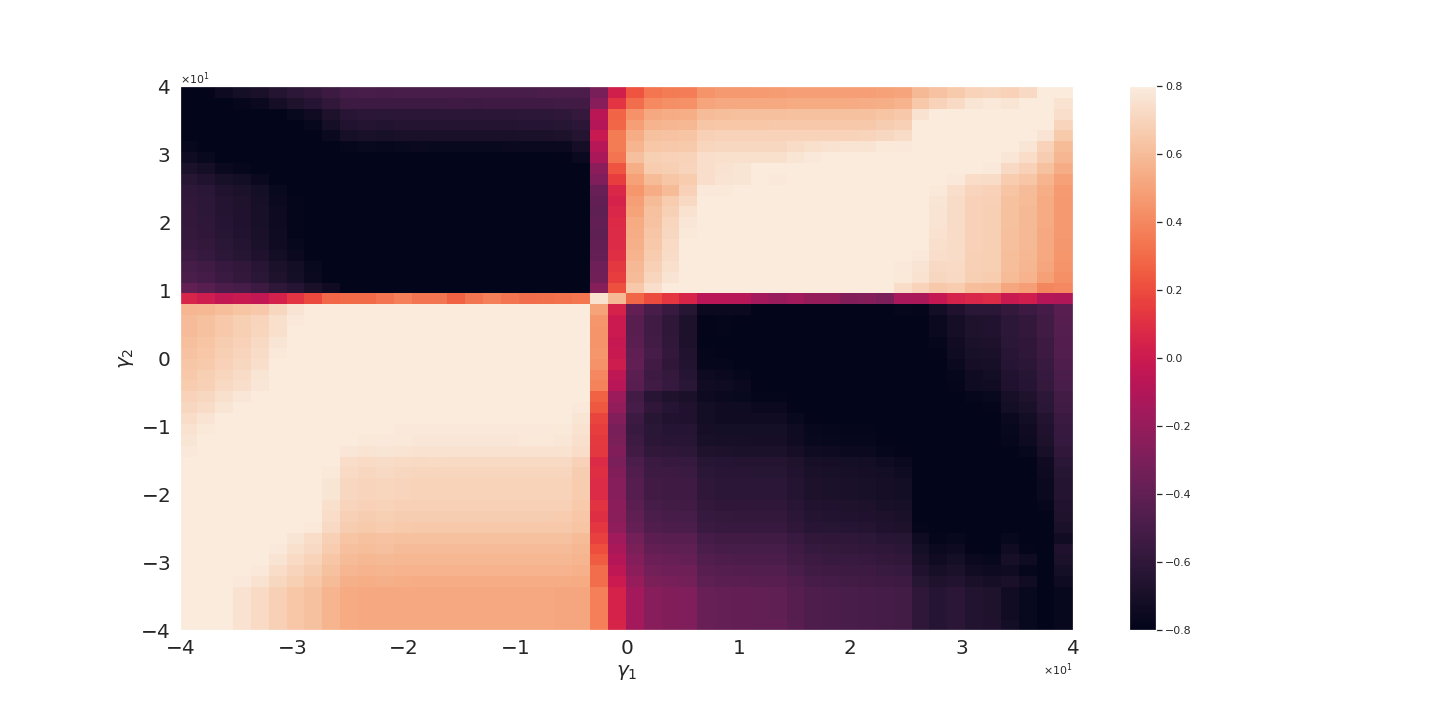}

\includegraphics[width=.32\textwidth]{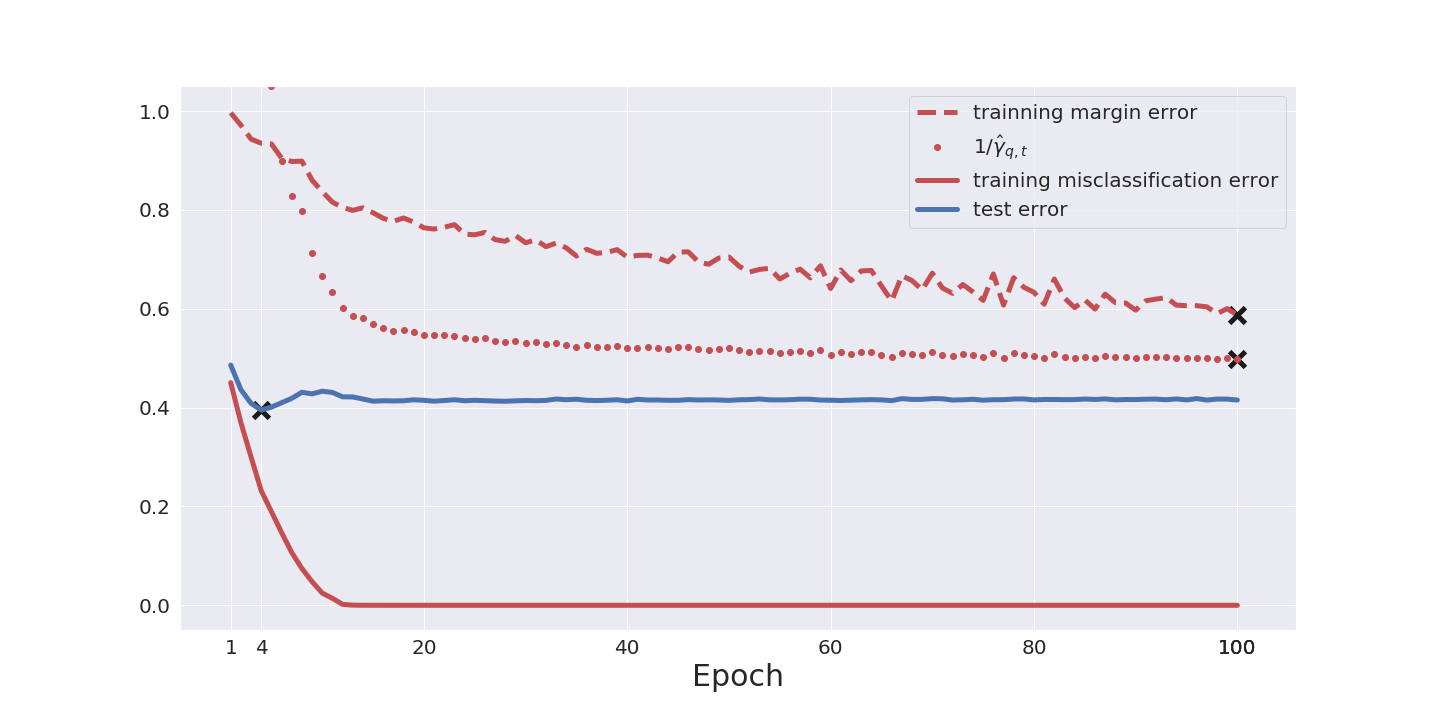}
\includegraphics[width=.32\textwidth]{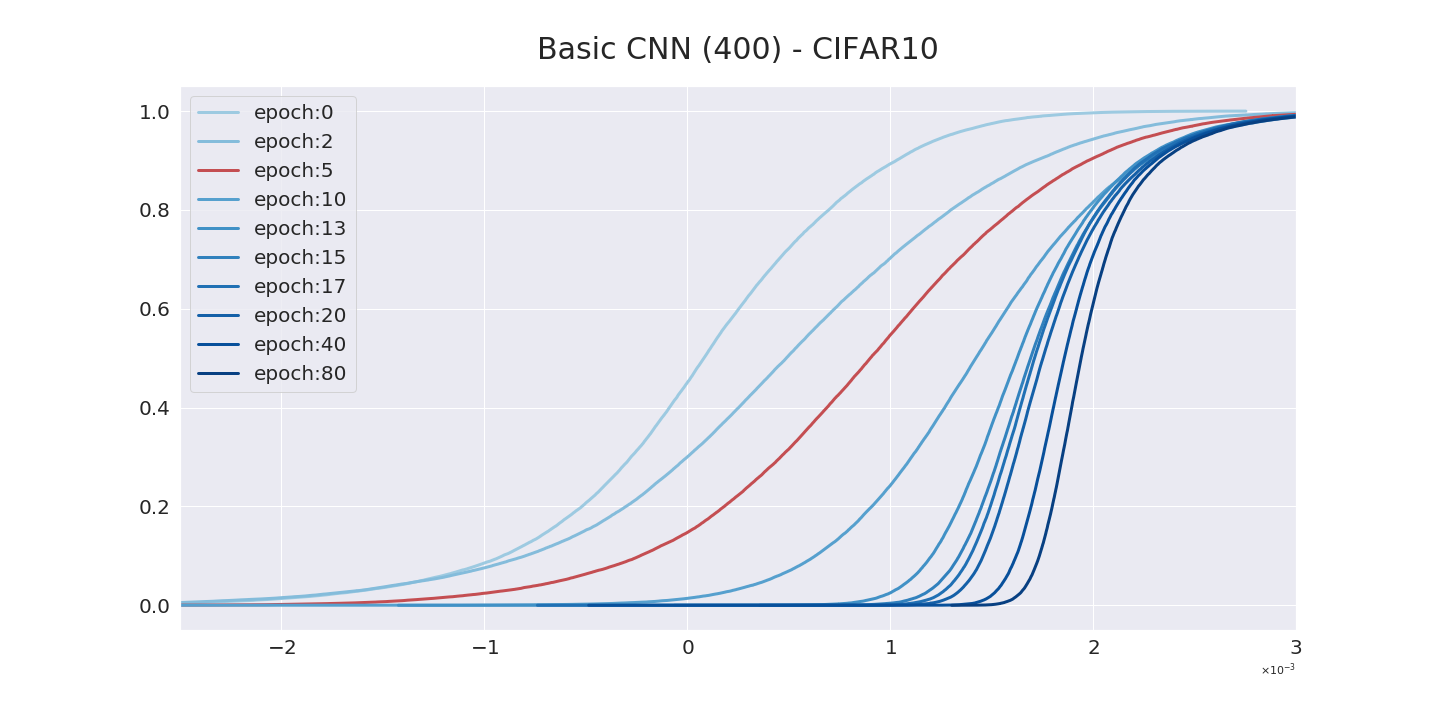}
\includegraphics[width=.32\textwidth]{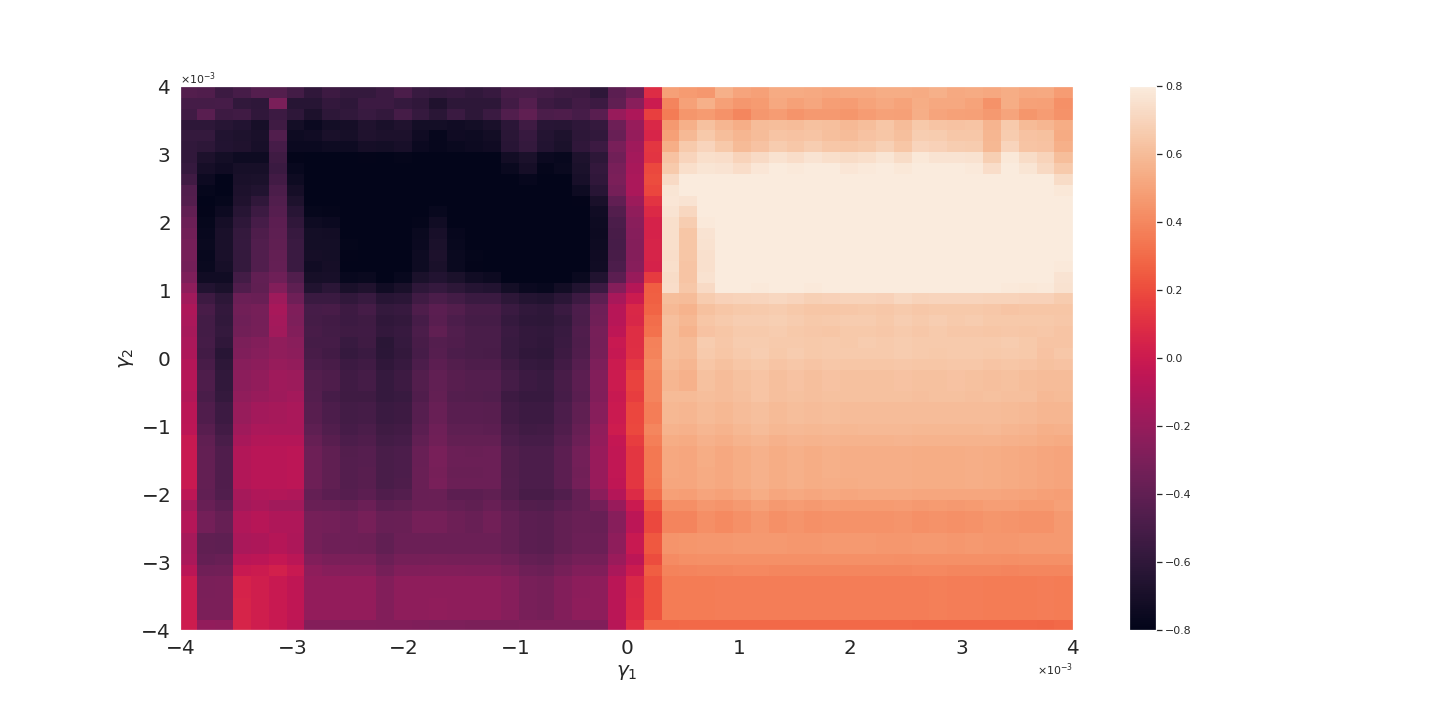}

\includegraphics[width=.32\textwidth]{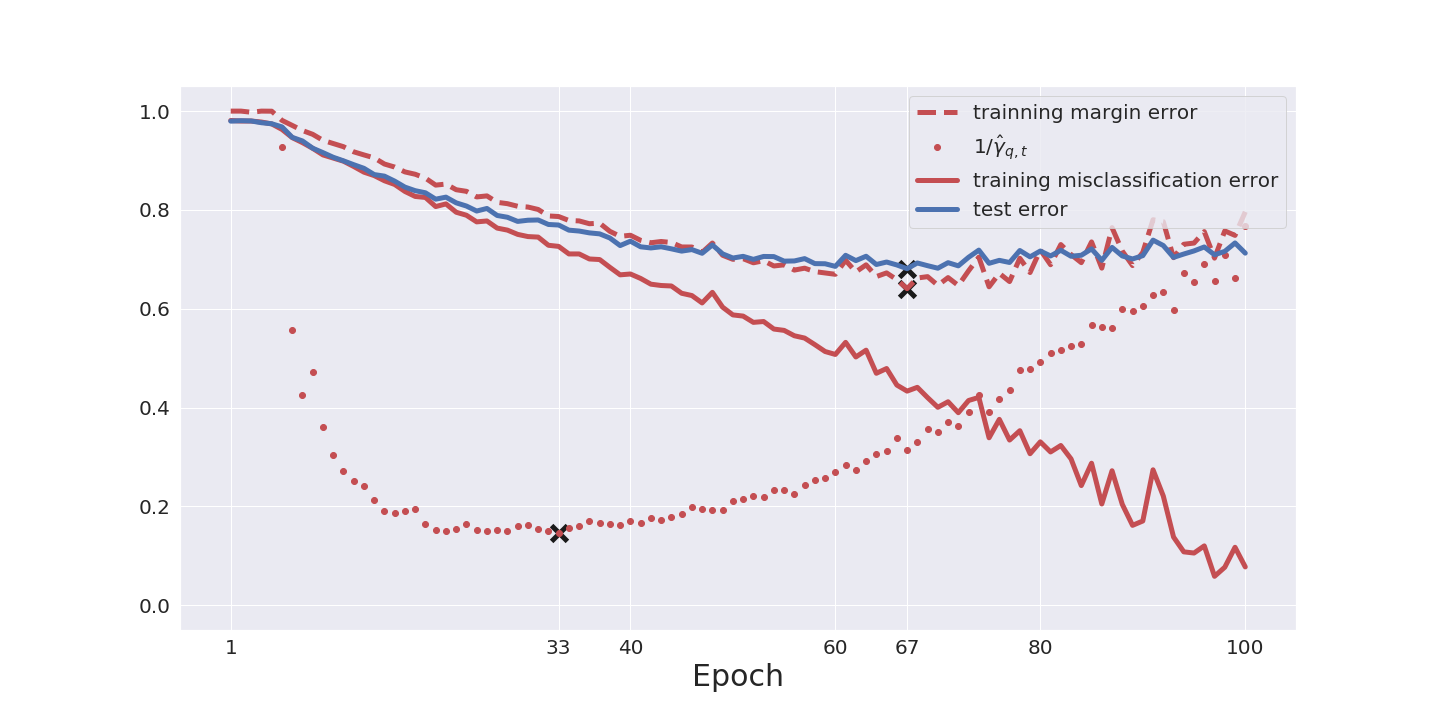}
\includegraphics[width=.32\textwidth]{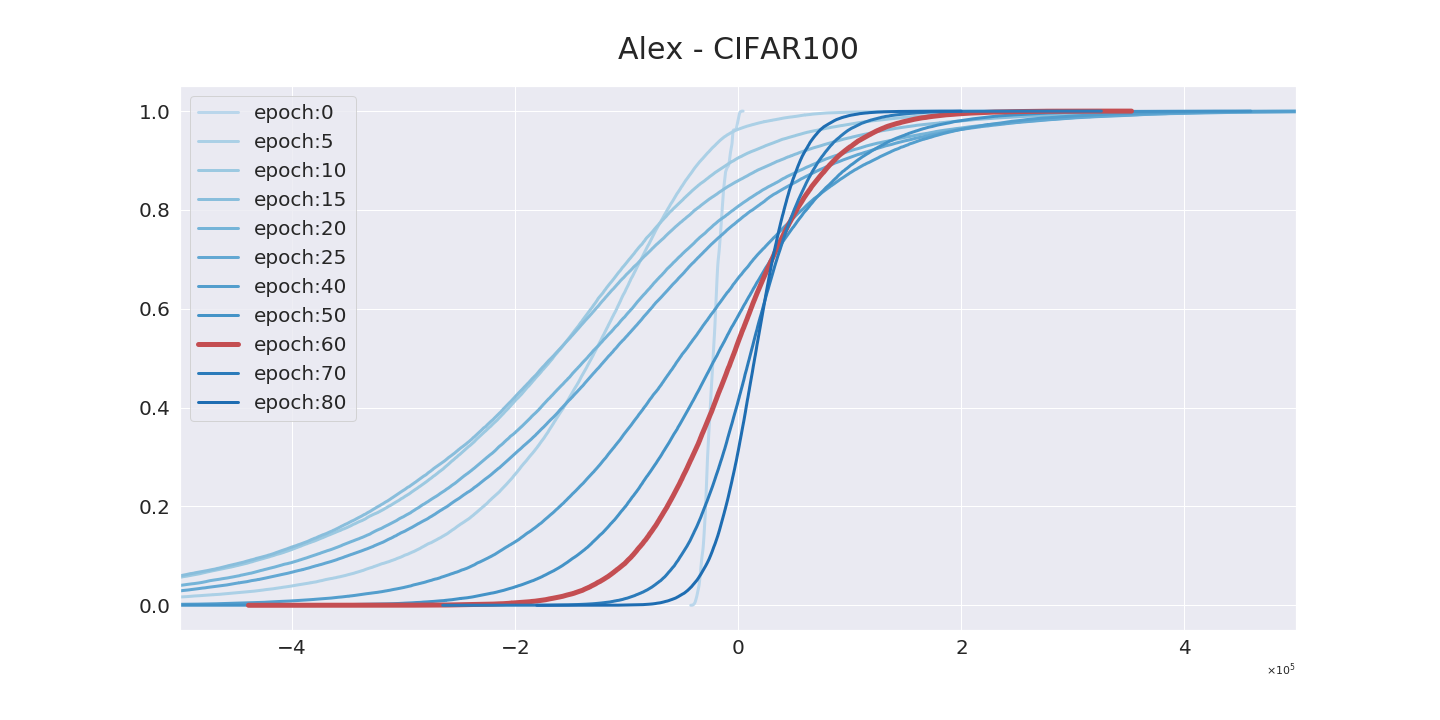}
\includegraphics[width=.32\textwidth]{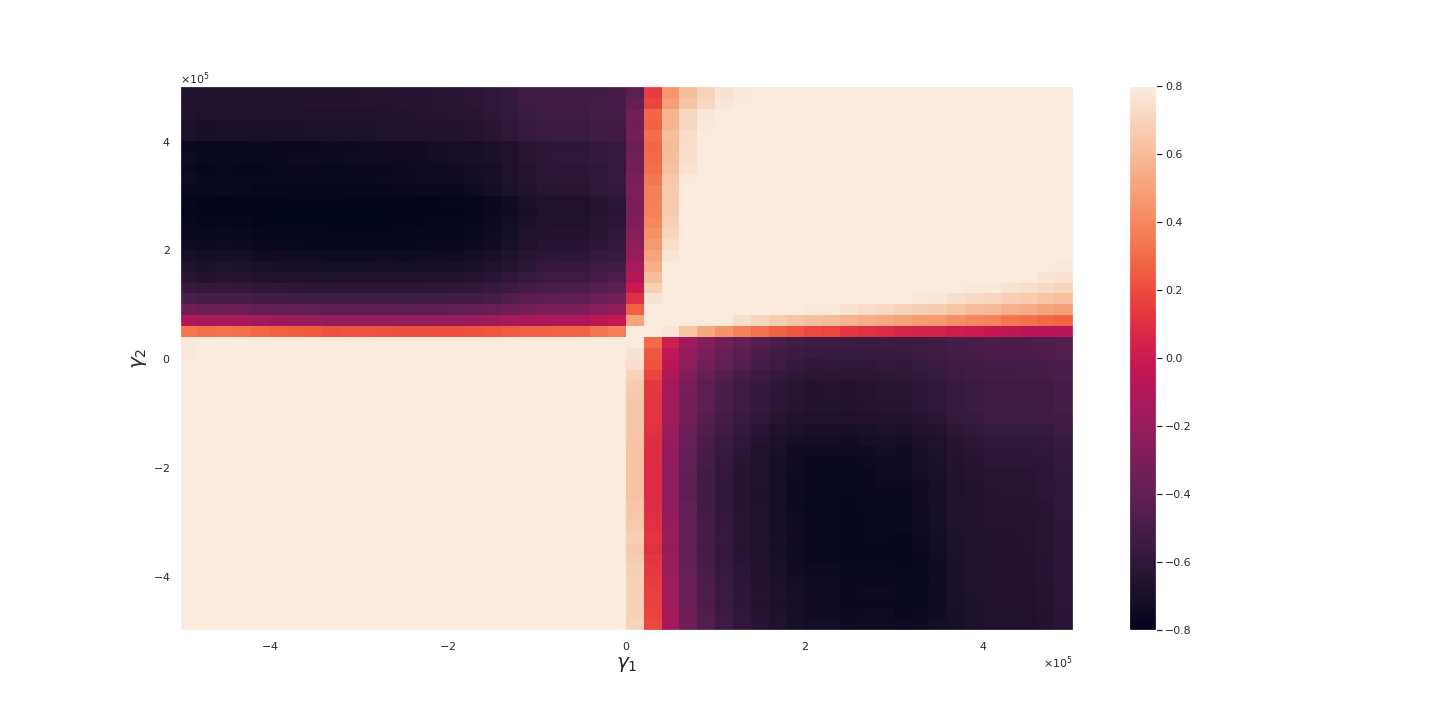}

\includegraphics[width=.32\textwidth]{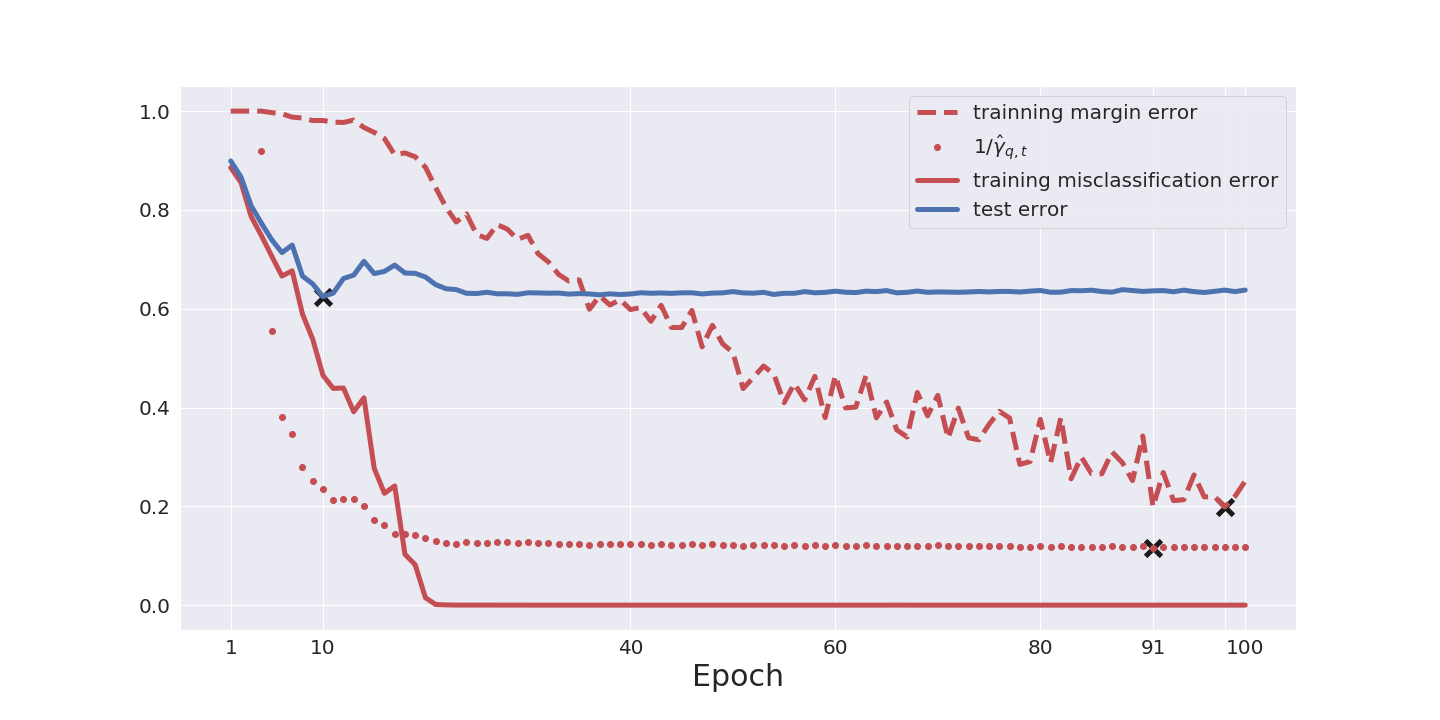}
\includegraphics[width=.32\textwidth]{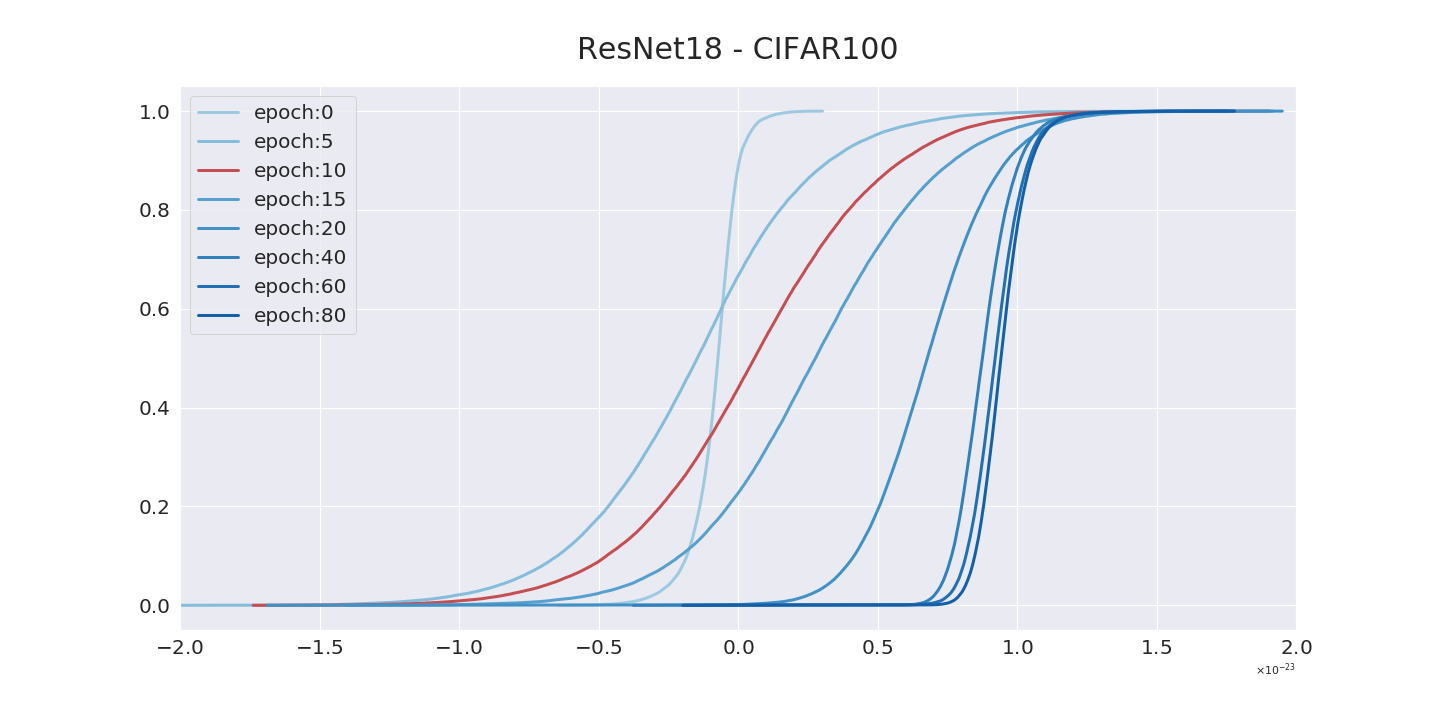}
\includegraphics[width=.32\textwidth]{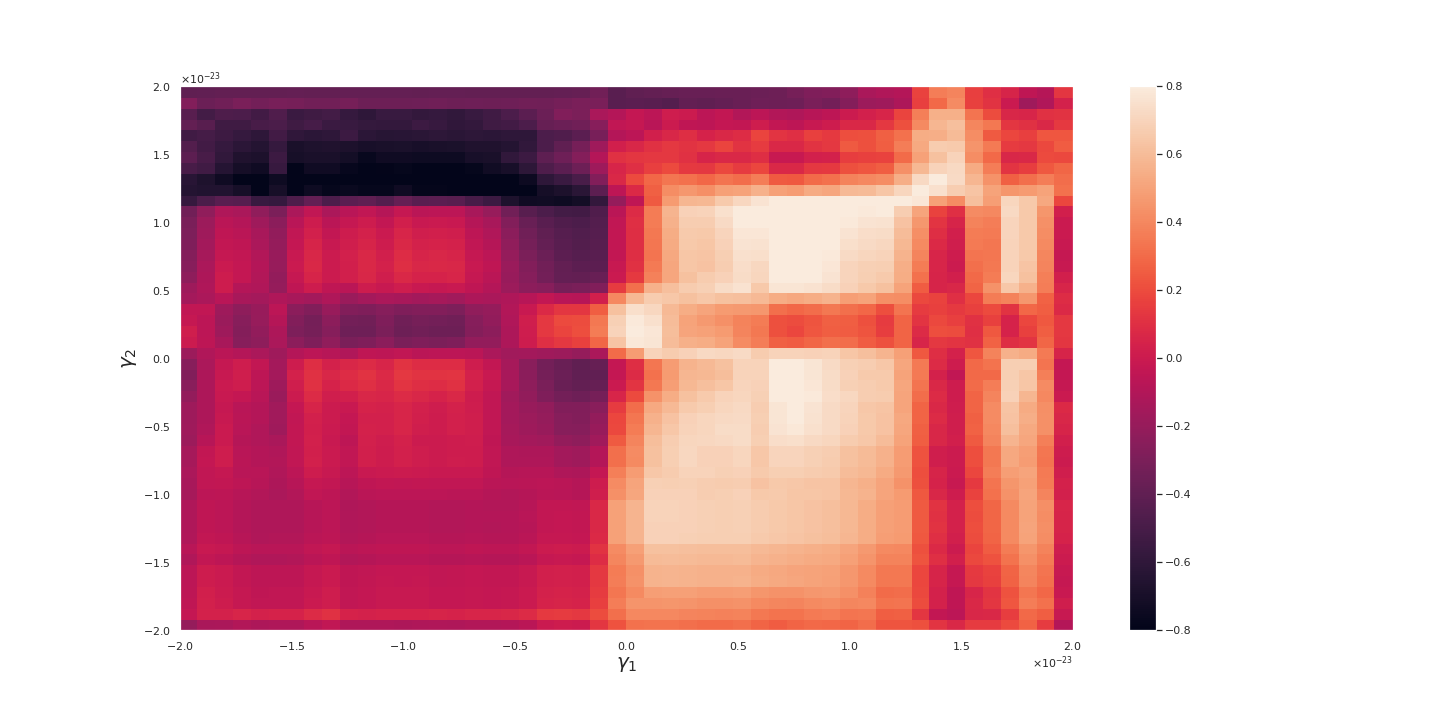}

\includegraphics[width=.32\textwidth]{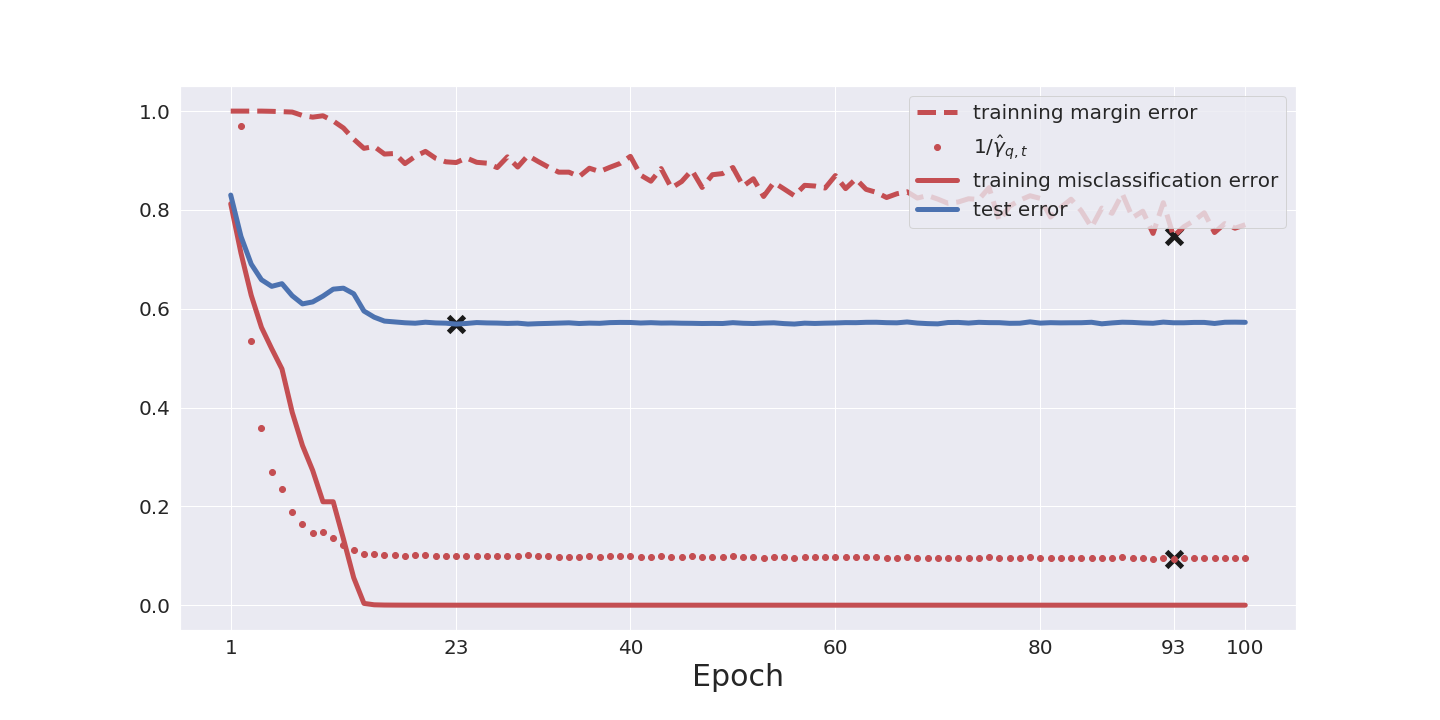}
\includegraphics[width=.32\textwidth]{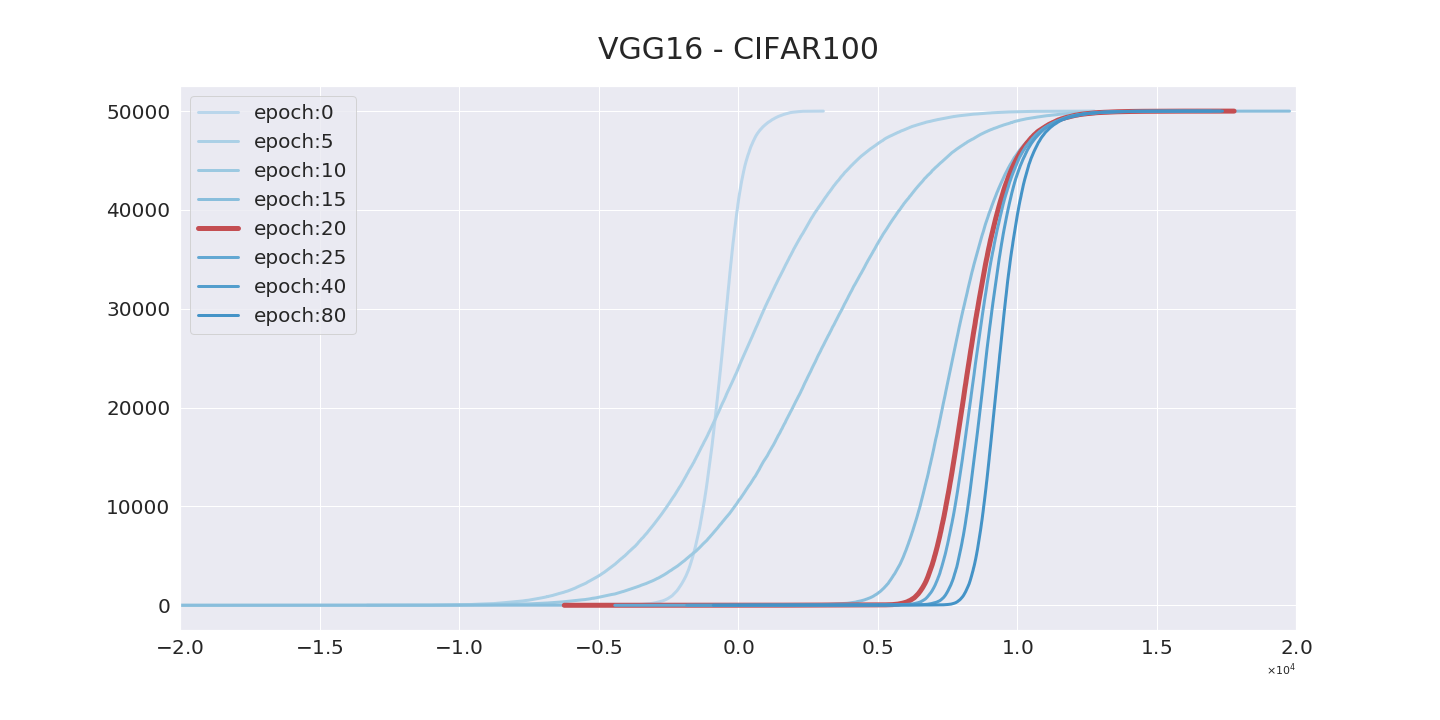}
\includegraphics[width=.32\textwidth]{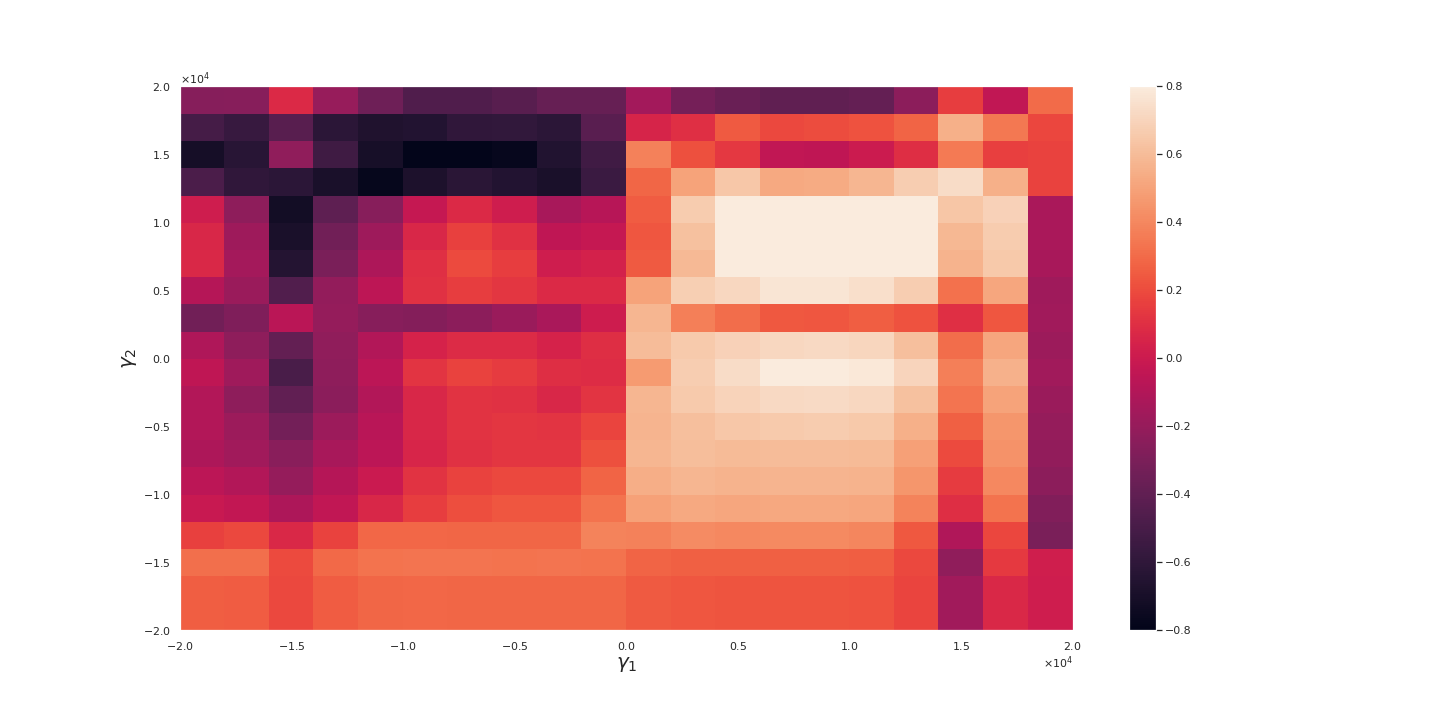}

\includegraphics[width=.32\textwidth]{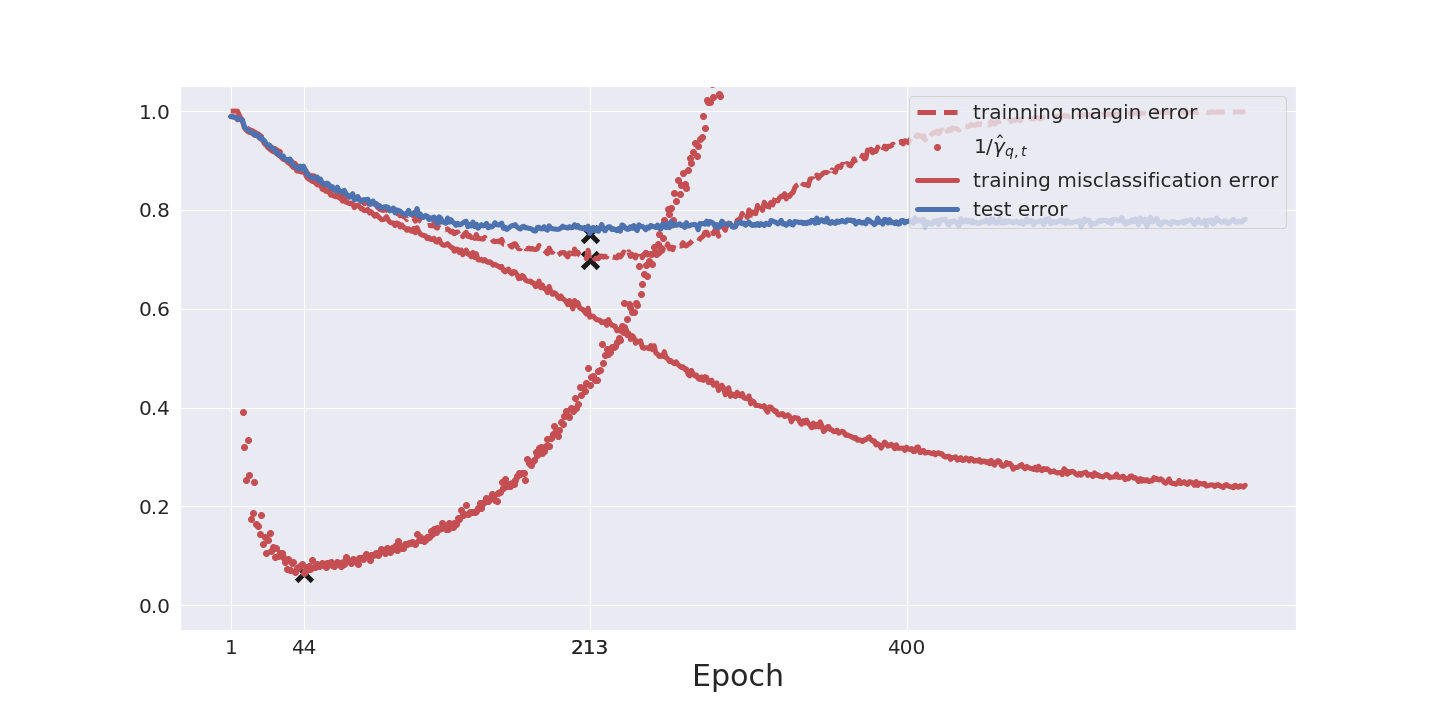}
\includegraphics[width=.32\textwidth]{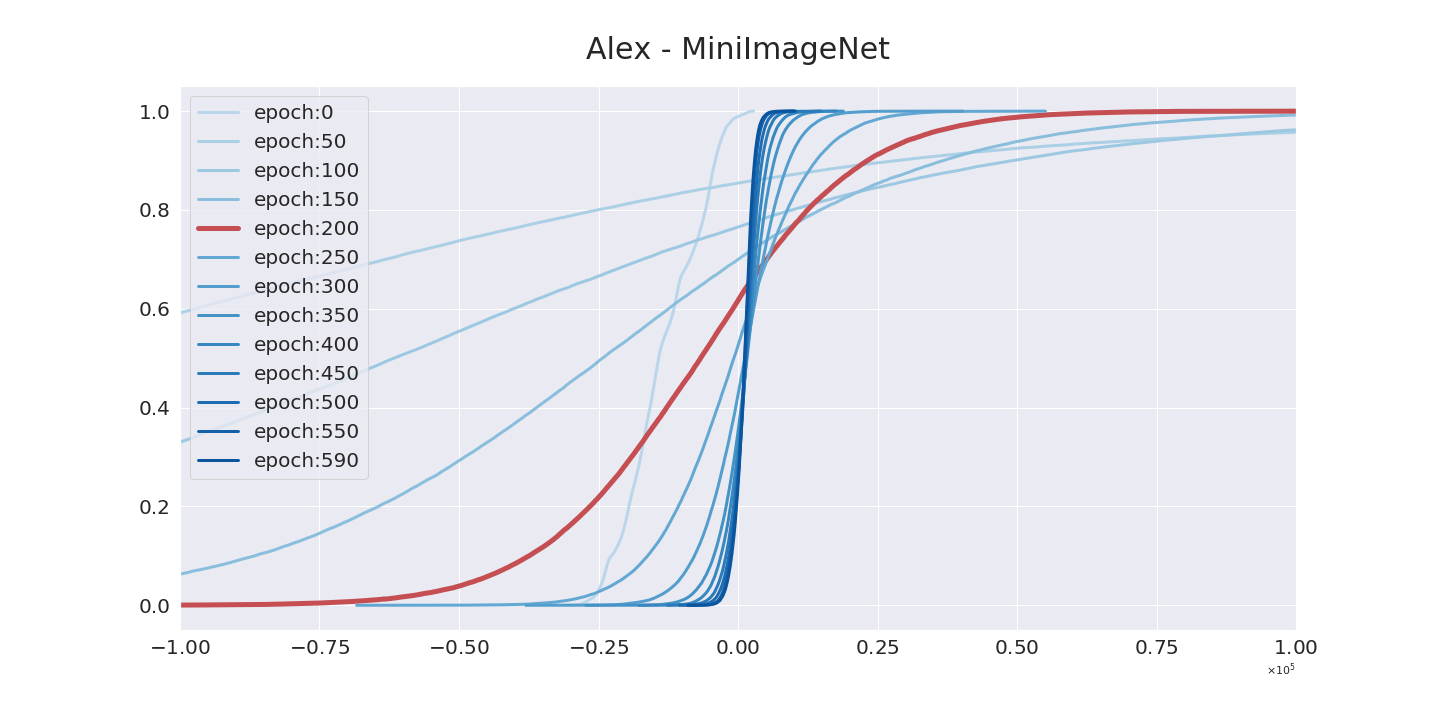}
\includegraphics[width=.32\textwidth]{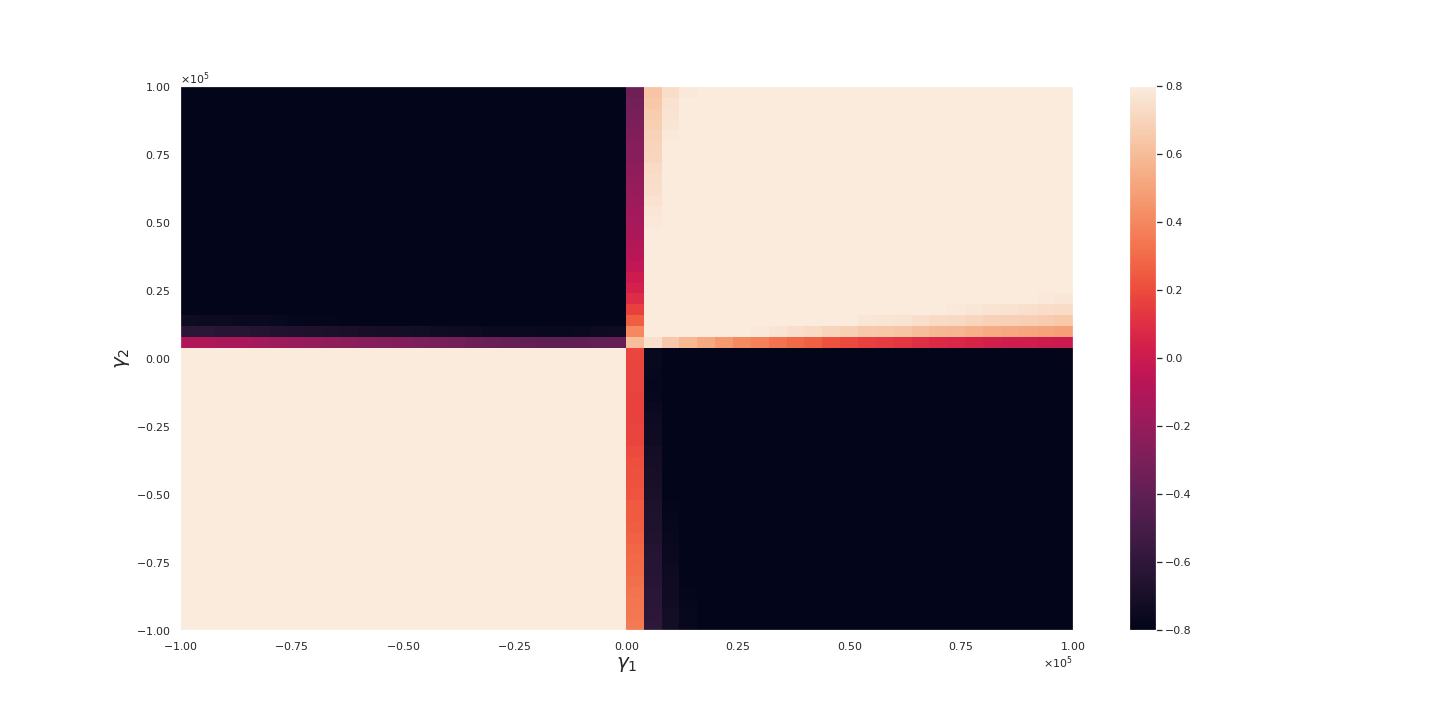}

\includegraphics[width=.32\textwidth]{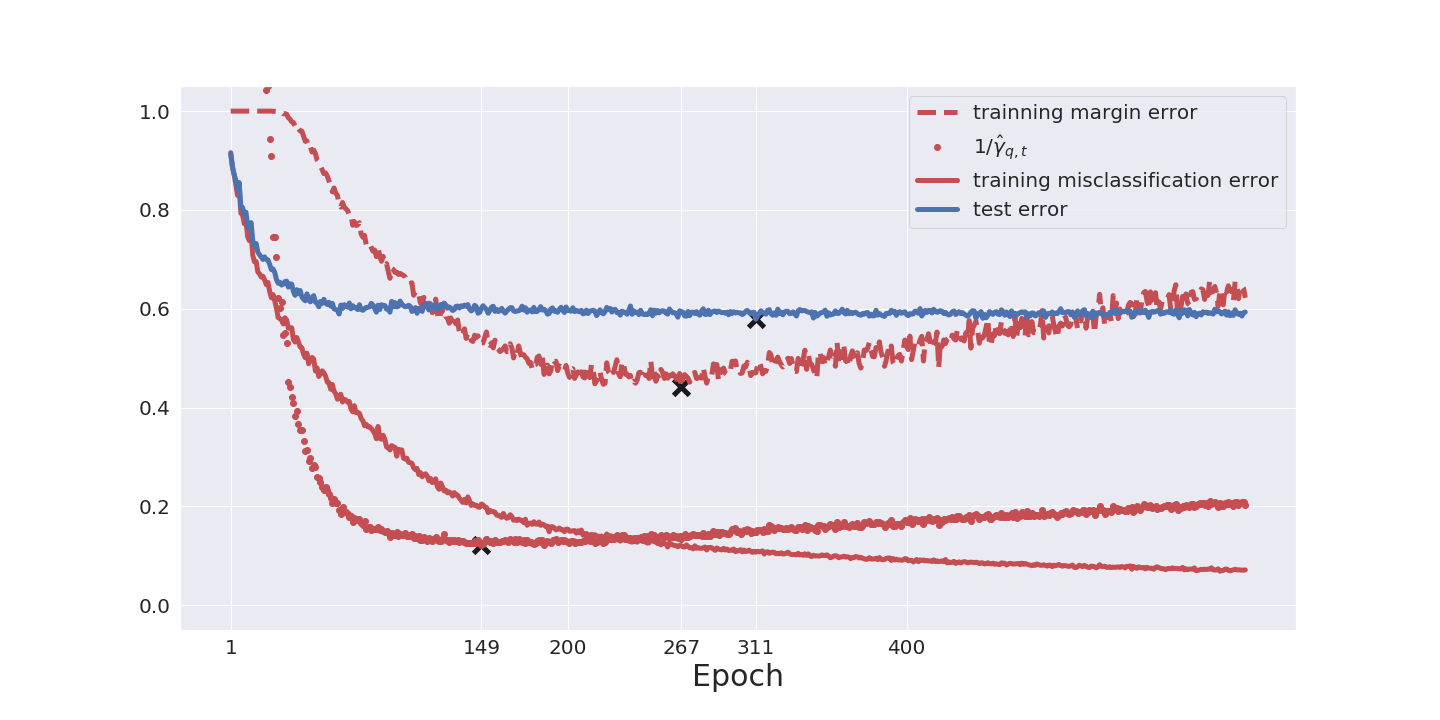}
\includegraphics[width=.32\textwidth]{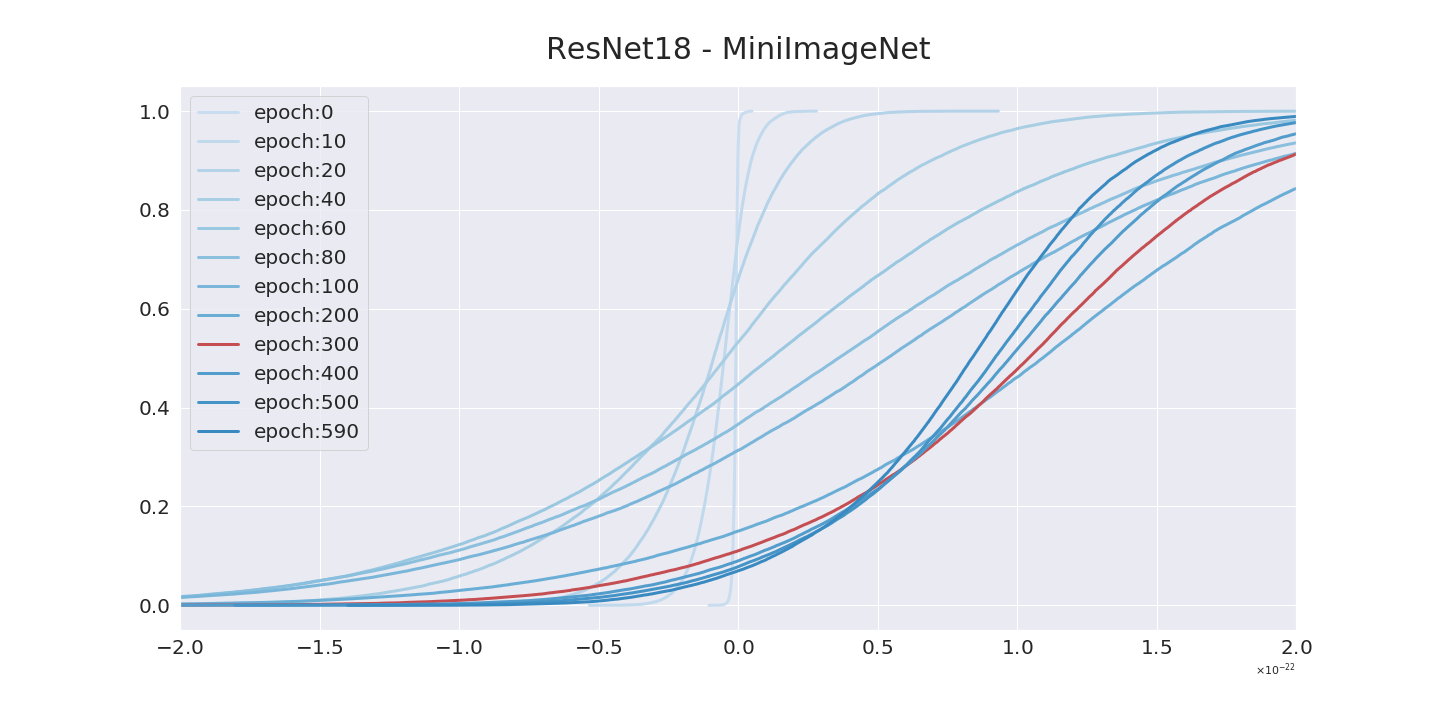}
\includegraphics[width=.32\textwidth]{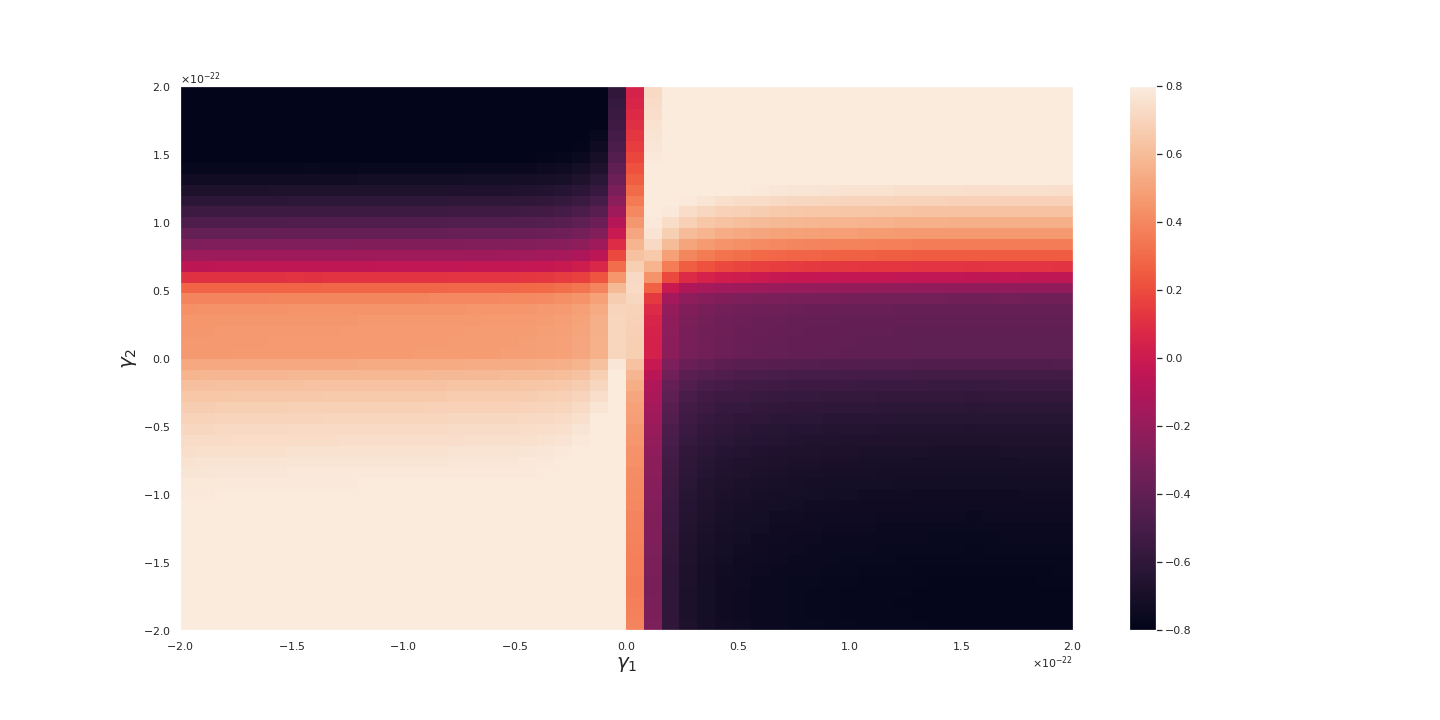}

\includegraphics[width=.32\textwidth]{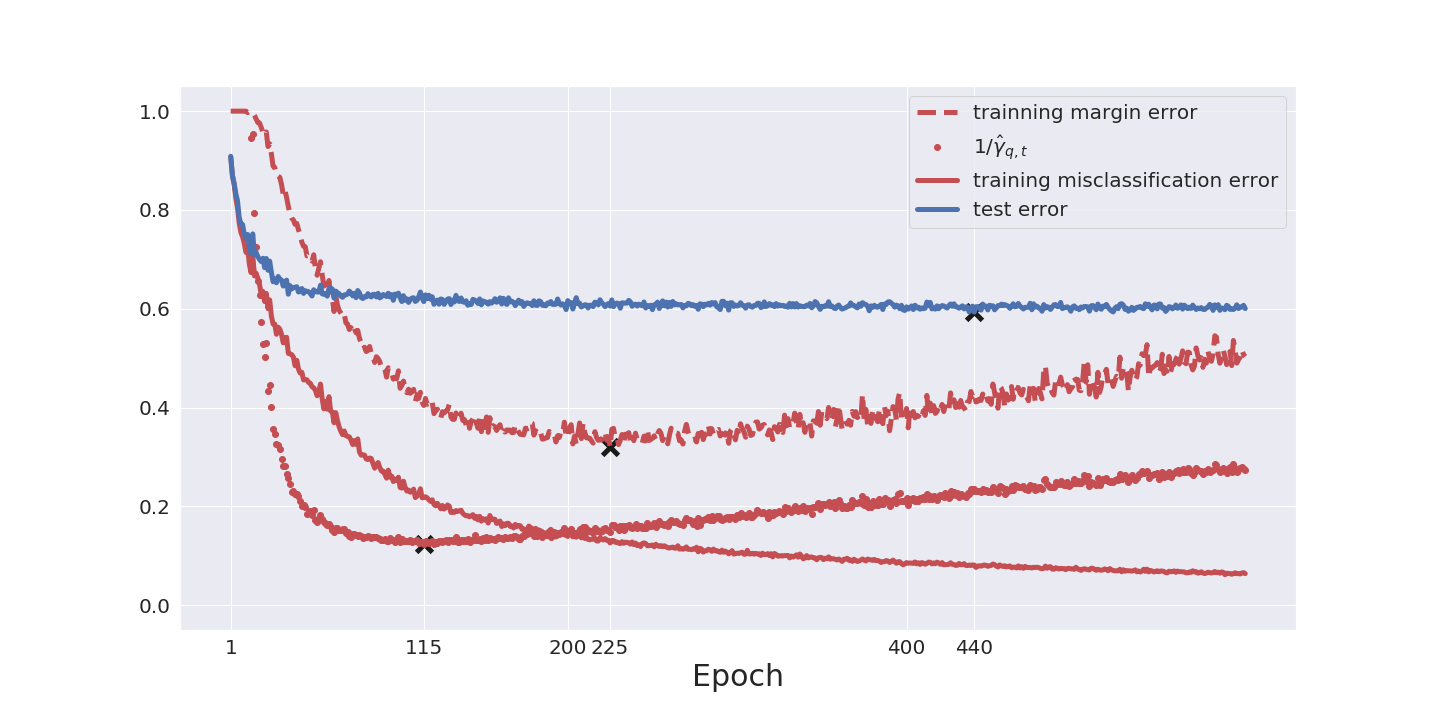}
\includegraphics[width=.32\textwidth]{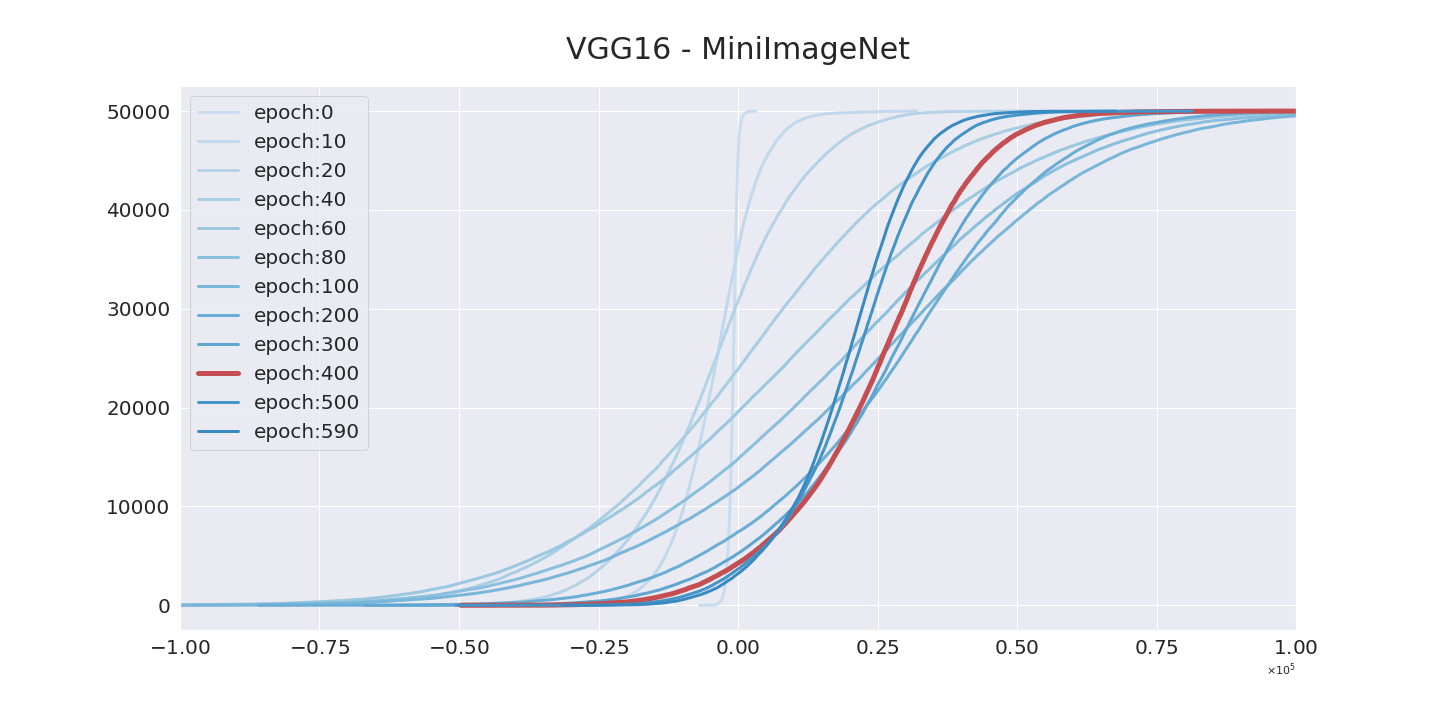}
\includegraphics[width=.32\textwidth]{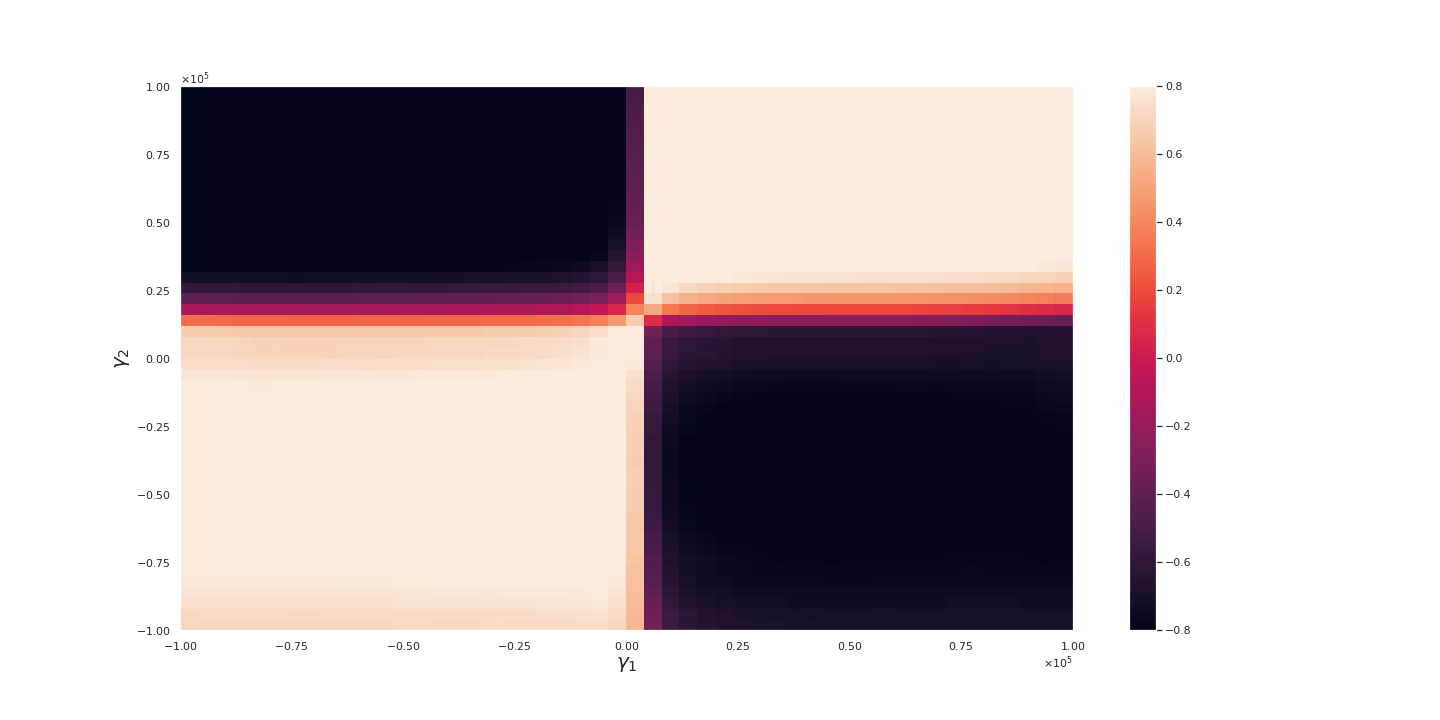}

\caption{Comparisons of Basic CNNs, AlexNet, VGG16, and ResNet-18 in CIFAR10/100, and Mini-ImageNet. The dataset and network in use are marked in titles of middle pictures in each row. Left: curves of training error, generalization error, training margin error and inverse quantile margin. Middle: dynamics of training margin distributions. Right: heatmaps are Spearman-$\rho$ rank correlation coefficients between dynamics of training margin error ($\eP_n[e_{\gamma_2}(\nf(x_i),y_i)]$) and dynamics of test margin error ($\P[e_{\gamma_1}(\nf_t(x),y)]$) drawn on the $(\gamma_1,\gamma_2)$ plane.}\label{fig:pract-examples}
\end{figure}

\subsection{Discussion: Effluence of Normalization Factor Estimates}  

\begin{figure}[H]
\centering
\includegraphics[width=.32\textwidth]{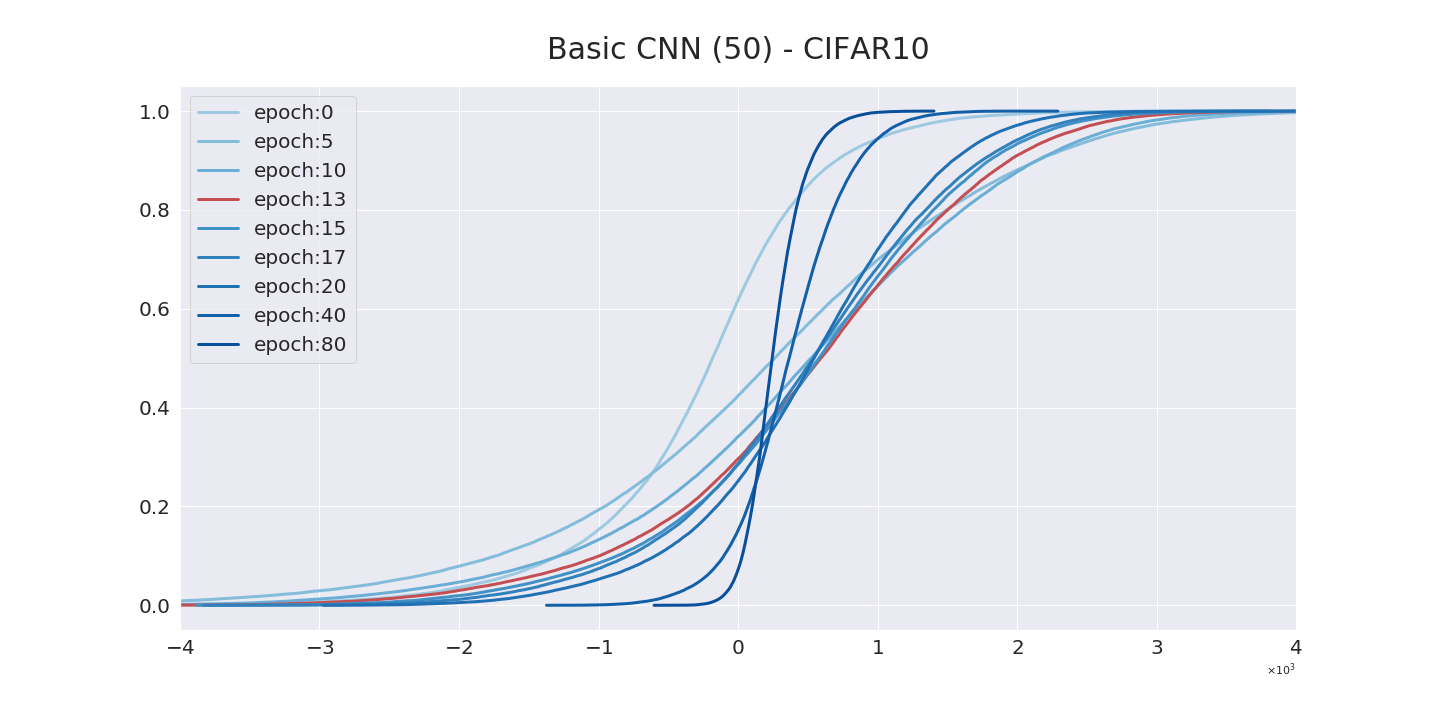}
\includegraphics[width=.32\textwidth]{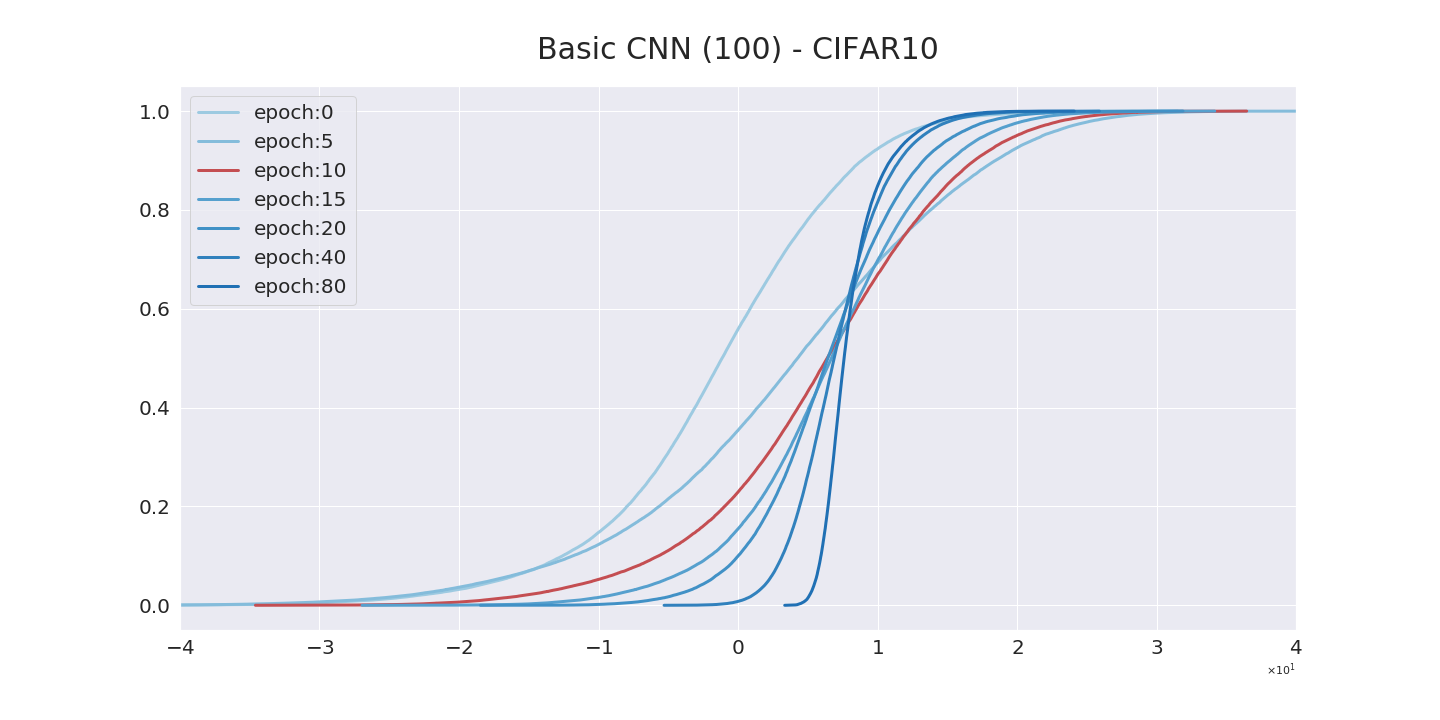}
\includegraphics[width=.32\textwidth]{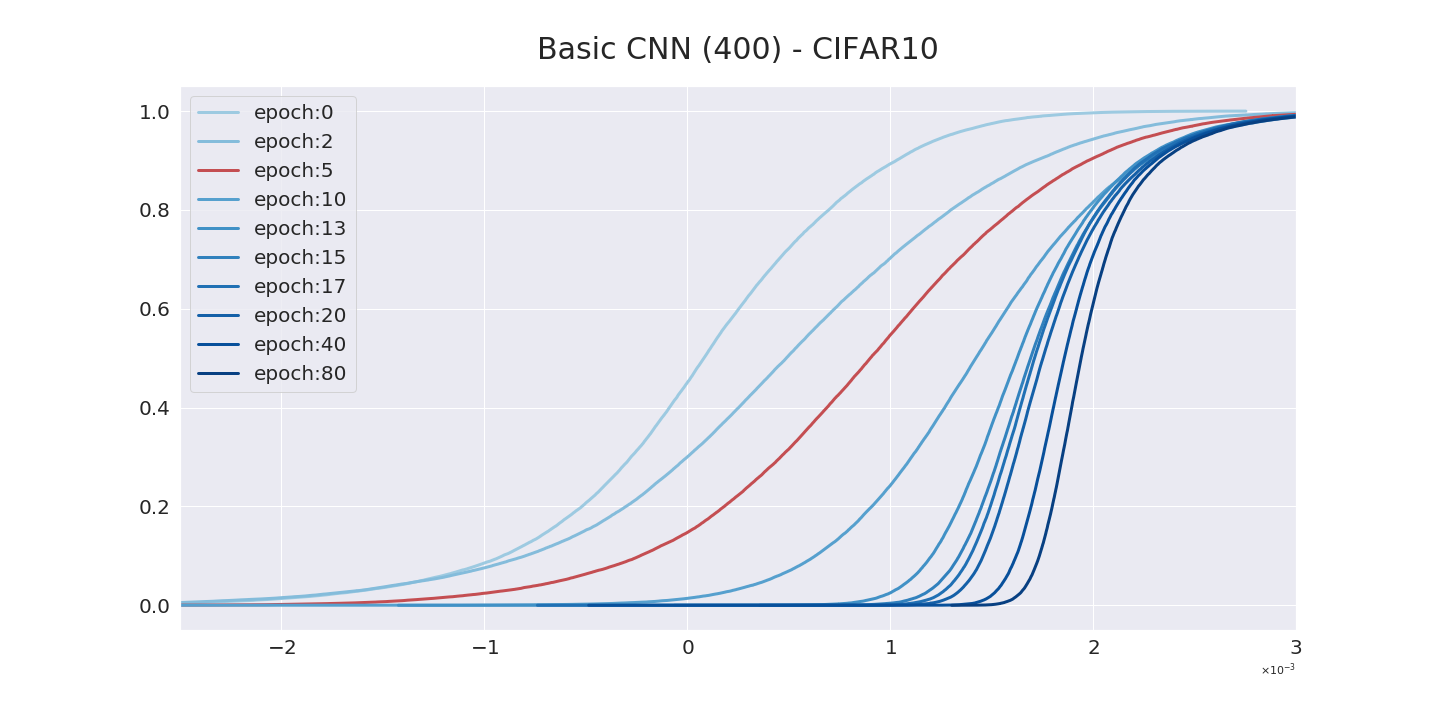}
\includegraphics[width=.32\textwidth]{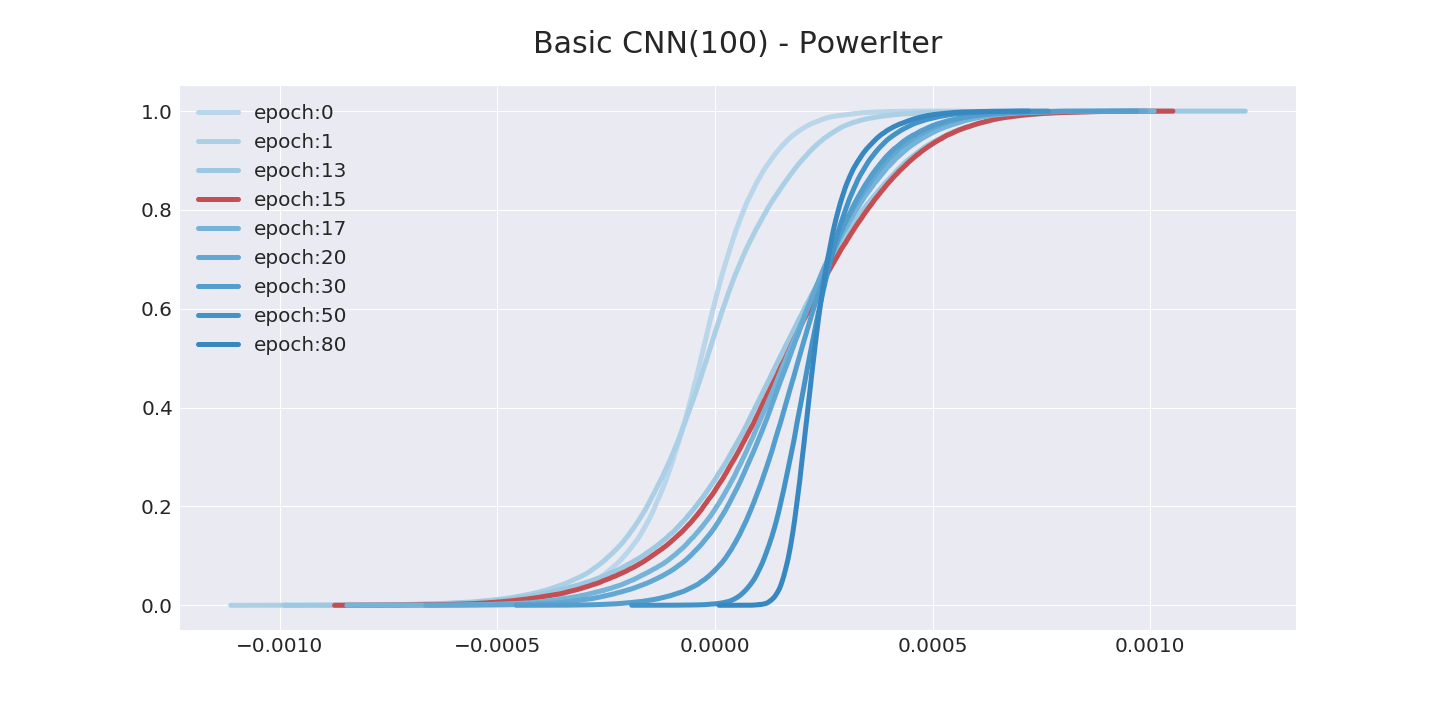}
\includegraphics[width=.32\textwidth]{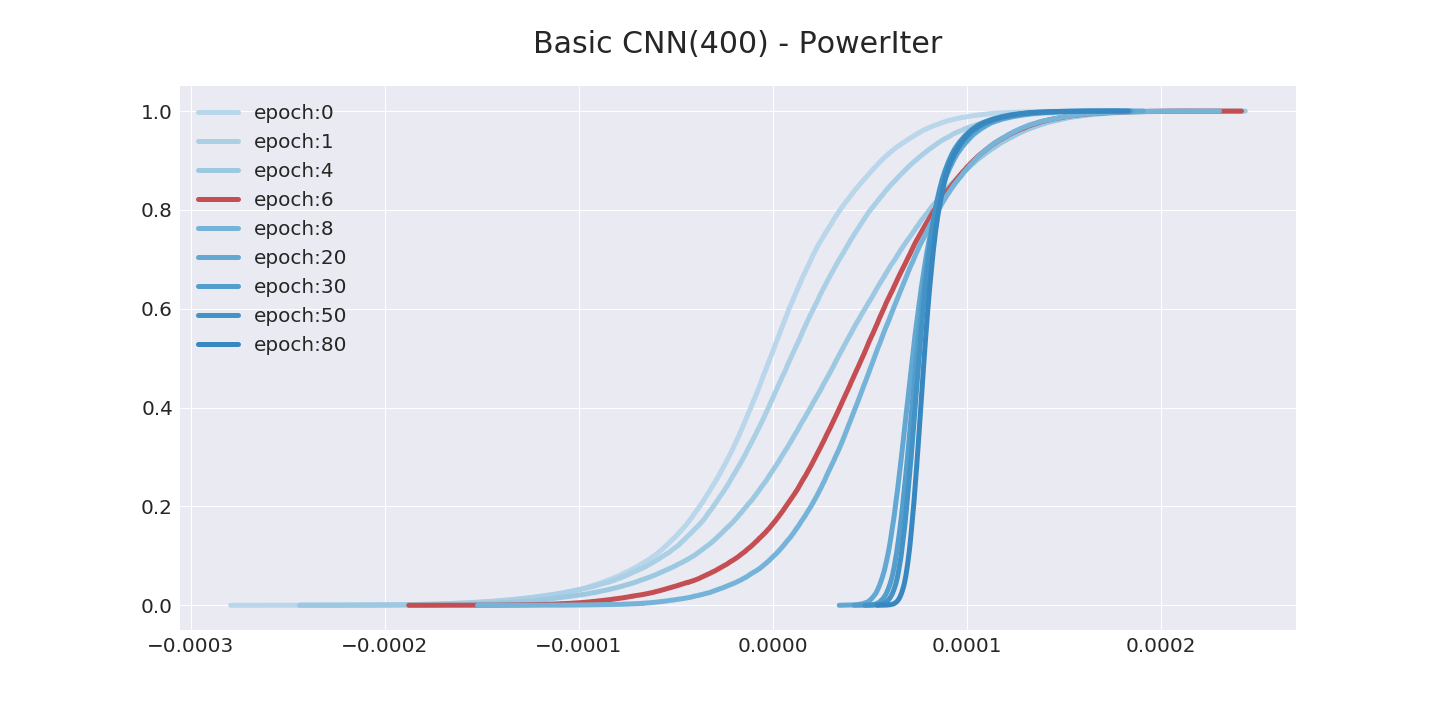}
\includegraphics[width=.32\textwidth]{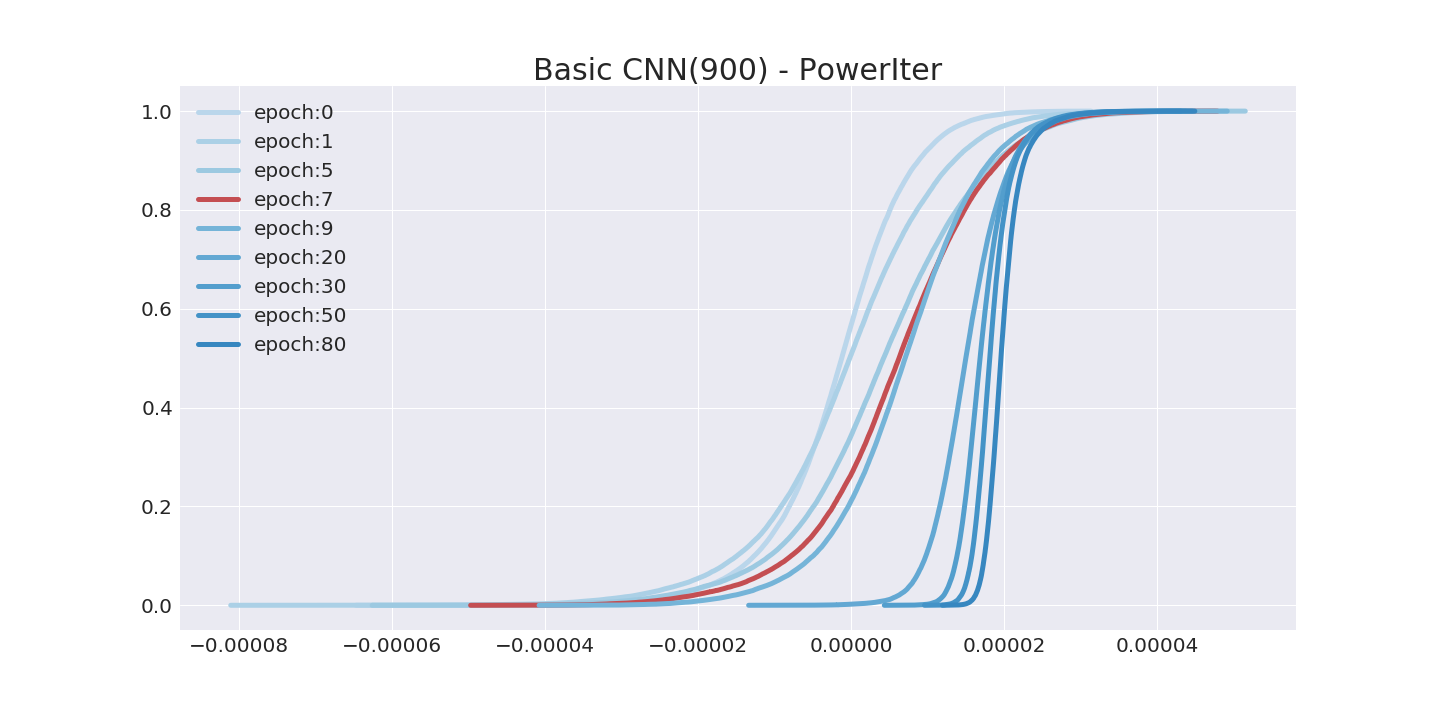}
\includegraphics[width=.32\textwidth]{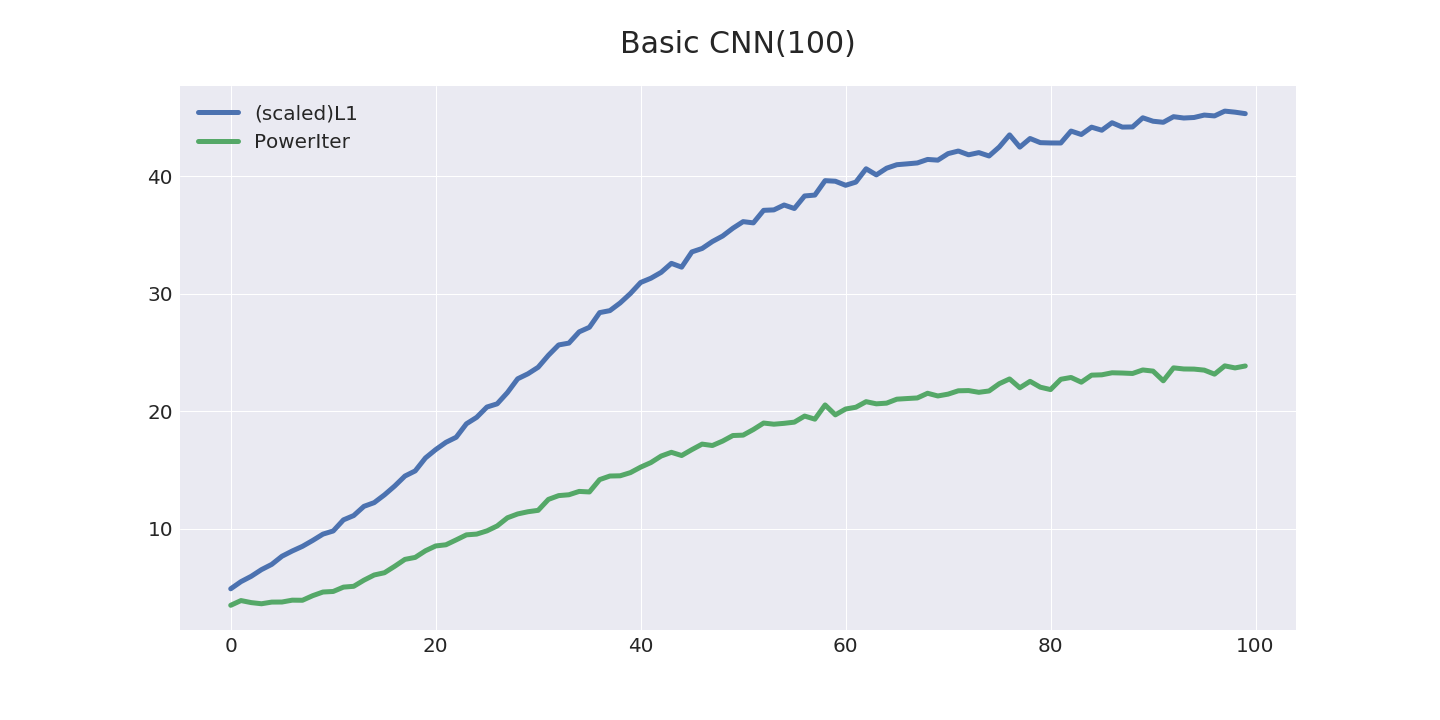}
\includegraphics[width=.32\textwidth]{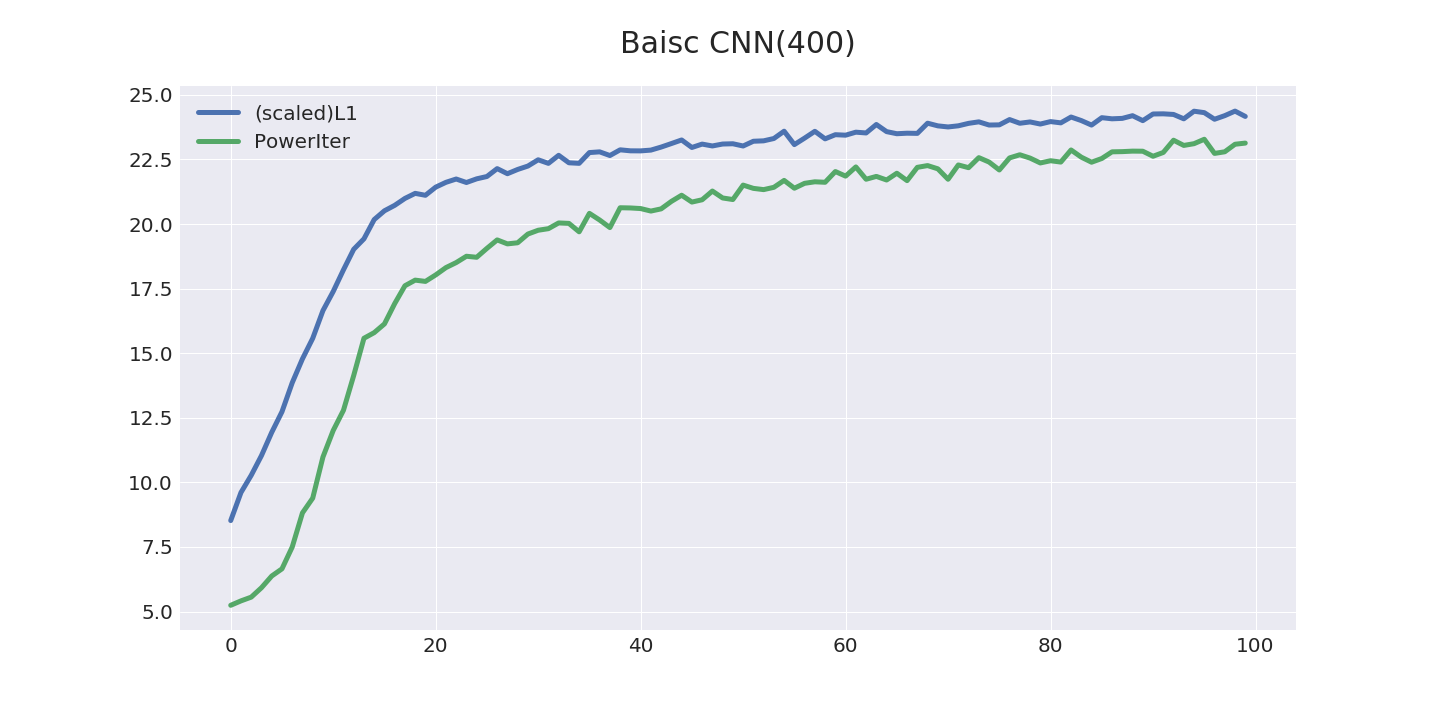}
\includegraphics[width=.32\textwidth]{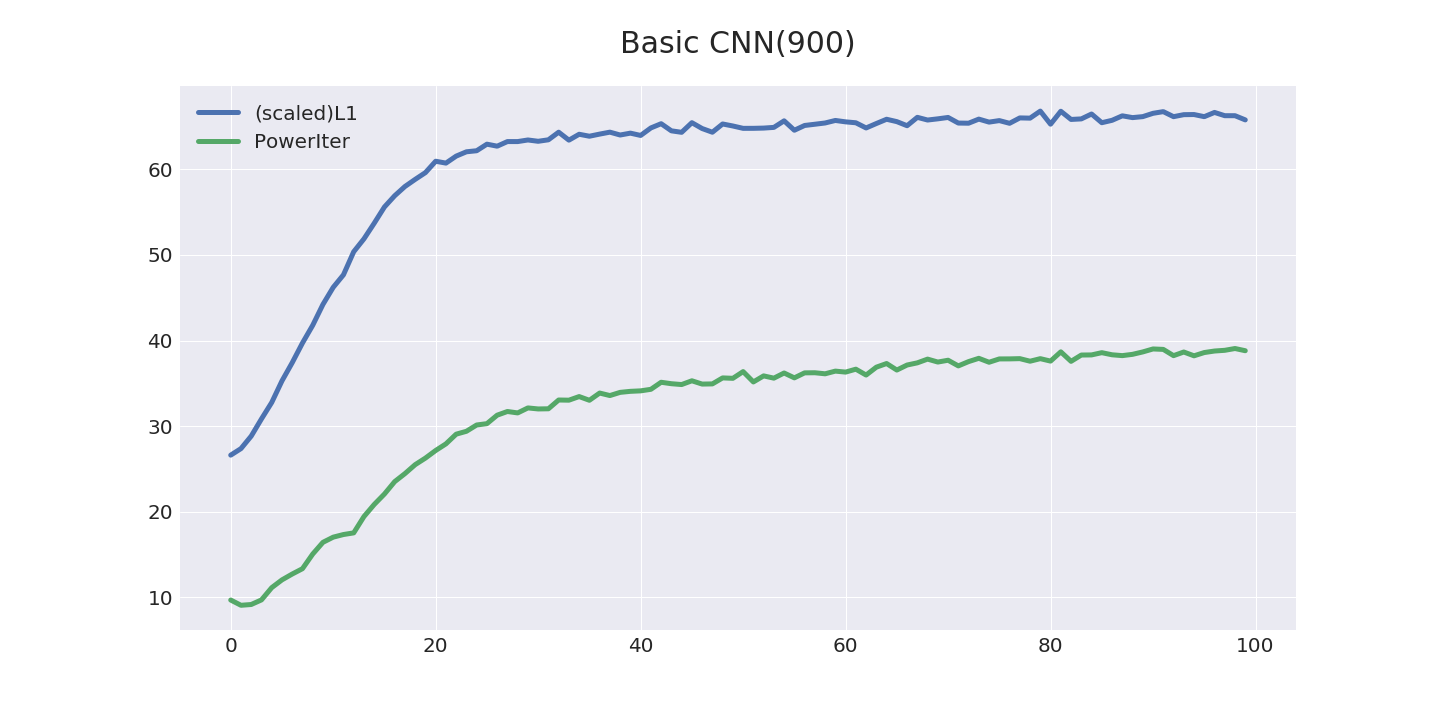}

\caption{Comparisons on normalization factor estimates by power iteration and the $\ell_1$-based estimate. Dataset: CIFAR10 with 10 percents corrupted. Net structure: Basic CNN with channels 50 (Top, Left), 100 (Top, Middle), 400 (Top Right), 200 (Middle, Left), 600 (Middle, Middle), 900 (Middle, Right). In the top row, the spectral norm in $L_f$ is estimated via the $\ell_1$-based estimate method and in the middle row, the spectral norm is estimated by power iteration. Bottom pictures show the estimates of $L_f$ by power iterations (in green color) and by the $\ell_1$-based estimate method (in blue color), respectively. The curves of $L_f$ estimates are rescaled for visualization since a fixed scaling factor along training doesn't influence the occurrence of cross-overs or phase transitions. Note that the original $\ell_1$-based estimates are of order $1e+17, 1e+19, 1e+21$ (100 channels, 400 channels, 900 channels, respectively) and the power iteration estimates are of $1e+3, 1e+3, 1e+3$ (100 channels, 400 channels, 900 channels, respectively). As shown above, a more accurate estimation of spectral norm may extend the range of predictability, but eventually faces the Breiman's dilemma if the model representation power grows too much against the dataset complexity. 
} \label{fig:powiter}
\end{figure}

In the end, it's worth to mention that different choices of the normalization factor estimation may affect the range of predictability, but still exhibit Breiman's dilemma. In all experiments above, normalization factor is estimated via the $\ell_1$-based estimate in Proposition \ref{prop:convnorm} in Section \ref{app:est}. One could also use power iteration \citep{miyato2018spectral} to present a more precise estimation on spectral norm. Usually the $\ell_1$-based estimates lead to a coarser upper bound than the power iterations, see Figure \ref{fig:powiter}. It is a fact that in training margin dynamics, large margins are typically improved at a slower speed than small margins. Therefore it turns out a more accurate estimation of spectral norm with faster increases in training may bring cross-overs (or phase transitions) in large training margins and extend the range of predictability. However Breiman's dilemma still persists when the balance between model representation power and dataset complexity is broken as model complexity arbitrarily grows. 


\section{Conclusion and Future Directions} \label{sec:conclusion}
In this paper, we show that Breiman's dilemma is ubiquitous in deep learning, in addition to previous studies on Boosting algorithms. We exhibit that Breiman's dilemma is closely related to the trade-off between the expressive power of models and the complexity of data. Large margins on training data do not guarantee a good control on model complexity. Instead, we have shown that phase transitions in dynamics of normalized margin distributions are able to reflect the trade-off between model expressiveness and data complexity. In particular, if high or large training margin distributions undergo decrease-increase phase transitions during training, similar to that of test margins, model expressiveness is comparable to data complexity and normalized training margin-based generalization bounds has the prediction power in capturing possible overfitting. We have shown two theorems derived from normalized Rademacher complexity bounds can be used to quantitatively capture a data-driven early stopping rule to prevent overfitting. However, if the training margin distributions, both high and low parts, undergo a uniform increase during training, model has over expressiveness with respect to the data and margin theory above fails. Such phase transitions of margin evolutions may reflect the \emph{degree-of-freedom} of models with respect to data, which measures the sensitivity of model prediction over data response. Roughly speaking, an increase-decrease phase transition in high margin distributions accompanying the increase of low margins, indicates the degree-of-freedom of models is relatively smaller than the data complexity where one has to sacrifice the high margin predictions to improve the low margin predictions. In contrast, a uniform increase of margins over all training sample suggests that the degree-of-freedom of models are larger than the data complexity. A detailed study is left for future on designing data-driven early stopping rule and degree-of-freedom for models by monitoring the margin dynamics.

\section*{Acknowledgement}
We thank Tommy Poggio, Peter Bartlett, and Xiuyuan Cheng for helpful discussions. The research was supported in part by the Hong Kong Research Grant Council (HKRGC) grant 16303817, Innovation and Technology Fund (ITF) UIM/390, the National Natural Science Foundation of China (No. 61370004, 11421110001), as well as awards from the Tencent AI Lab and Si Family Foundation.

\newpage
\appendix

\section{Appendix: More Experimental Figures}
\label{app:suppexp}

\titlespacing*{\subsection}{0pt}{2\baselineskip}{3\baselineskip}

\subsection{Architecture Details about Basic CNNs} 
\begin{figure}[htbp]
\centering
\includegraphics[width=.48\textwidth]{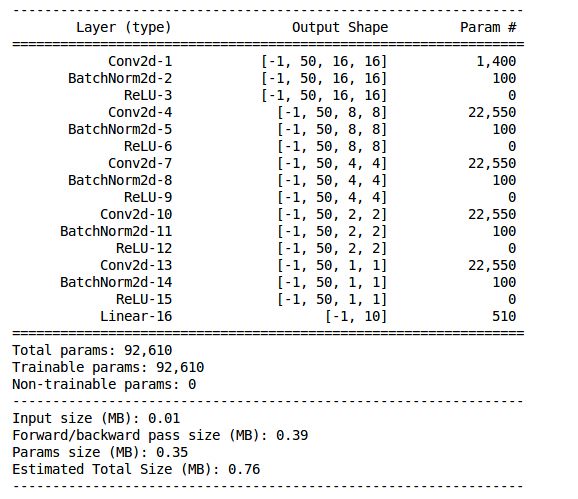}
\includegraphics[width=.48\textwidth]{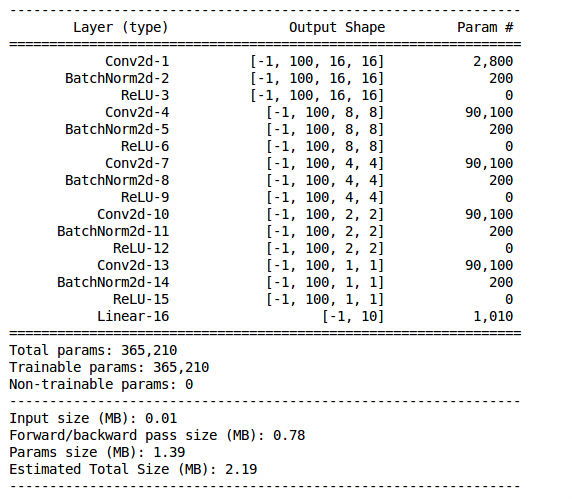}
\includegraphics[width=.48\textwidth]{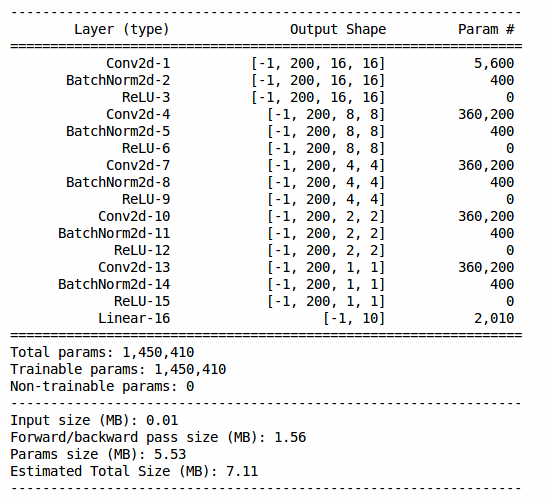}
\includegraphics[width=.48\textwidth]{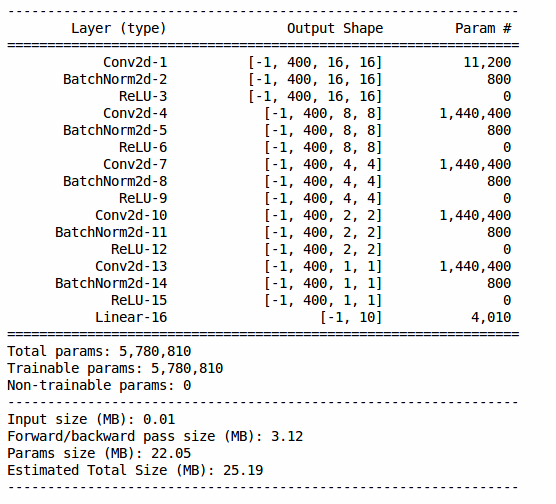}
\caption{Detailed information about CNN(50), CNN(100), CNN(200), and CNN(400).}\label{fig:cnns}
\end{figure}


\subsection{Two local minimums in ResNet-18}
\begin{figure}[H]
\centering
\includegraphics[width=.48\textwidth]{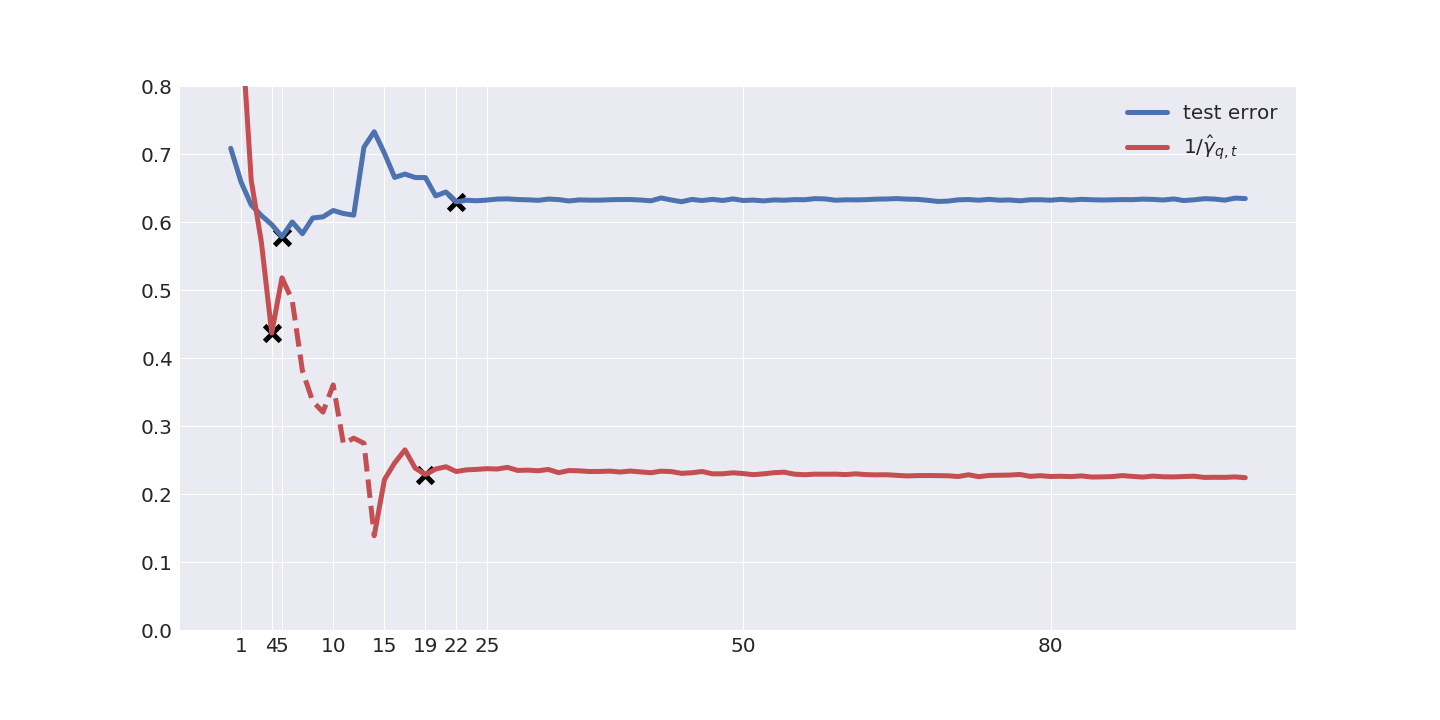}
\includegraphics[width=.48\textwidth]{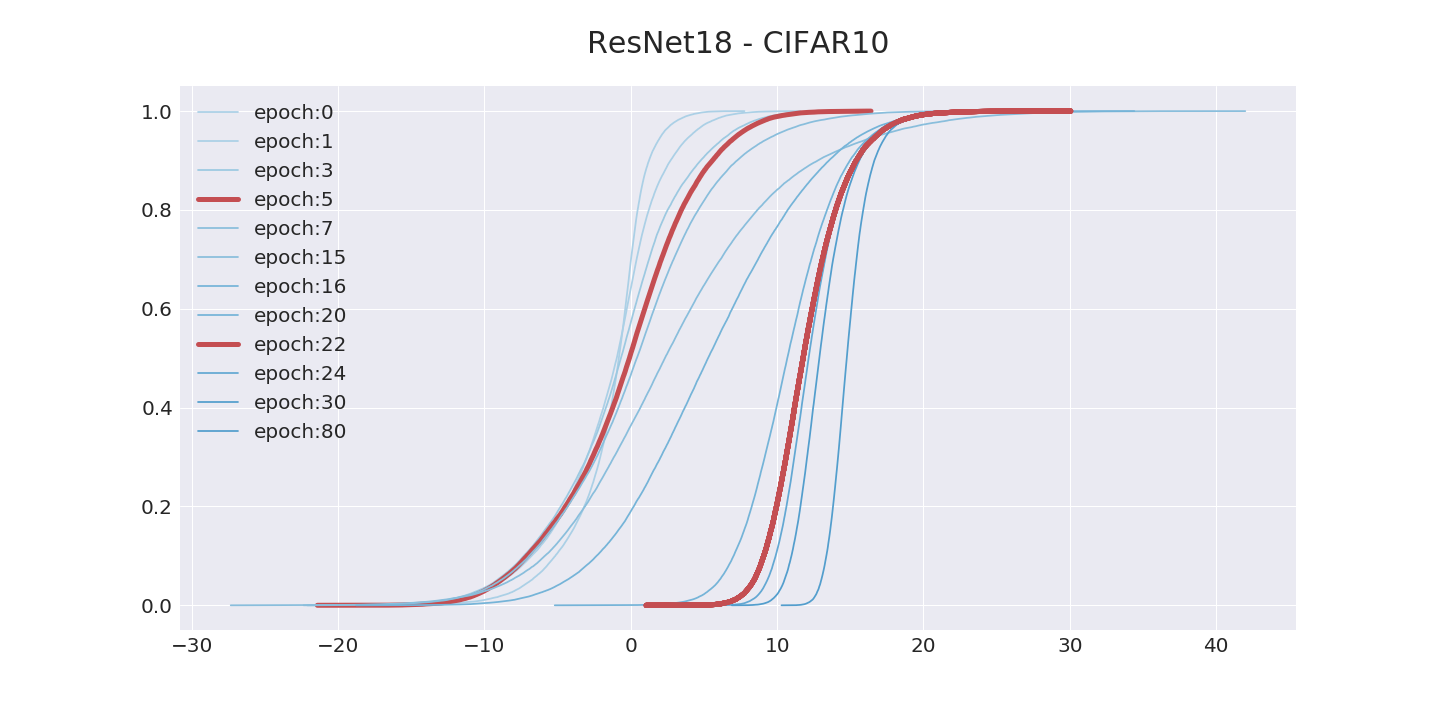}
\caption{Inverse quantile margin captures local optima, though may fail in predicting their relative order when model complexity is over-representative. Network: ResNet-18. Data: CIFAR10 with 10 percents label corrupted. Normalization factor, spectral complexity estimated by power iteration. Left: the dynamics of test error and inverse quantile margin with $q=0.95$. Overfitting occurs and two local minimums are marked with ``x'' in each dynamic. The dash line highlights the epochs when the training margins are monotonically improved. Right: dynamics of training margin distribution. Two distributions corresponding to local minima of test error are highlighted in red color. Since after the first (better) local minimum, the training margin distribution is uniformly improved in the second (worse) local minimum, that leads to the inverse quantile margin showing the second local minimum of smaller value. Yet the true order of the two local minima of test error is opposite. However, the inverse quantile margin still captures the optima locally, where the training margin distributions have cross-overs (phase-transitions) near local minima of test error.}\label{fig:res-qmar-twoloc}
\end{figure}


\section{Appendix: Proofs}
\label{app:pf}

\subsection{Auxiliary Lemmas}\label{app:lem}

\begin{lemma}\label{lem:ulll}
For any $\delta \in(0,1)$ and bounded-value functions $\gF_B:=\{ f:\gX\to \R: \|f\|_{\infty}\leq B\}$, the following holds with probability at least $1-\delta$,
\begin{equation}
    \sup_{f\in \gF_B} \E_n f(x) - \E f(x) \leq 
 2 \RC_n(\gF_B) + B\sqrt{\frac{\log(1/\delta)}{2n}}
\end{equation}
where 
\begin{equation}
    \RC_n(\gF) = \E \sup_{f\in \gF} \frac{1}{n} \sum_{i=1}^n \varepsilon_i f(x_{i})  
\end{equation}
is the Rademacher Complexity of function class $\gF$. 
\end{lemma}

For completeness, we include its proof that also needs the following well-known McDiarmid's inequality (see, e.g. \cite{wainwright19}).

\begin{lemma}[McDiarmid's Bounded Difference Inequality]\label{lem:mcdiarmid}
For $B_i$-bounded difference functions $h:\gX\to\R$ s.t. $|h(x_i,x_{-i}) - h(x_i^\prime,x_{-i})|\leq B_i $,
\[ \P\left\{ \E_n h - \E_x h(x) \geq \varepsilon \right\} \leq \exp\left(-\frac{2\epsilon^2}{\sum_{i=1}^n B_i^2} \right),\]
\end{lemma}

\begin{proof}[Proof of Lemma \ref{lem:ulll}]
It suffices to show that for $\bar{f} = f(x) - \E f(x)$,  
\begin{equation}
    \sup_{f\in \gF_B} \E_n \bar{f}= \sup_{f\in \gF_B} \E_n \bar{f} - \E \sup_{f\in \gF_B} \E_n \bar{f}  + \E \sup_{f\in \gF_B} \E_n \bar{f}   
\end{equation}
where with probability at least $1-\delta$,
\begin{equation} \label{eq:bd}
    \sup_{f\in \gF_B} \E_n \bar{f} - \E \sup_{f\in \gF_B} \E_n \bar{f} \leq B\sqrt{\frac{\log 1/\delta}{2n}} 
\end{equation}
by McDiarmid's bounded difference inequality, and
\begin{equation} \label{eq:rc}
    \E \sup_{f\in \gF_B} \E_n \bar{f} \leq 2\RC_n(\gF) 
\end{equation}
using Rademacher complexity. 

To see (\ref{eq:bd}), we are going to show that $\sup_{f\in \gF_B} \E_n \bar{f}$ is a bounded difference function. Consider $g(x_1^n)=\E_n\bar{f}=\frac{1}{n}\sum_{i=1}^n f(x_i) - \E_x f(x)$. Assume that the $i$-th argument $x_i$ changes to $x^\prime_i$, then for every $g$, 
\begin{eqnarray*}
 g(x_i, x_{-i})-\sup_{g} g(x_i^\prime, x_{-i}) & \leq & g(x_i, x_{-i}) - g(x_i^\prime,x_{-i}) \\
 & \leq & \frac{1}{n} [f(x_i) - f(x_i^\prime)] \\
 & \leq & \frac{B}{n}.
\end{eqnarray*}
Hence $\sup_g g(x_i, x_{-i})-\sup_{g} g(x_i^\prime, x_{-i}) \leq B/n$, which implies that $\sup_{f\in \gF_B} \E_n \bar{f}$ is a $B/n$- bounded difference function. Then (\ref{eq:bd}) follows from the McDiarmid's inequality (Lemma \ref{lem:mcdiarmid}) using $B_i = B/n$ and $\delta=\exp(-2n\varepsilon^2/B^2)$.

As to (\ref{eq:rc}), 
\begin{eqnarray*}
\E \sup_{f\in \gF_B} \E_n \bar{f} & = & \E_{x_1^n} \sup_{f\in \gF_B} \E_{y_1^n} \left[\E_n f(x_1^n) - \E_n f(y_1^n)\right]  \\
& \leq & \E_{x_1^n,y_1^n} \sup_{f\in \gF_B} \left[\E_n f(x_1^n) - \E_n f(y_1^n)\right]  \\ 
& = & \E_{x_1^n,y_1^n} \sup_{f\in \gF_B}\E_{\varepsilon_1^n} \frac{1}{n}\sum_{i=1}^n\varepsilon_i  \left(f(x_i) - f(y_i)\right), \ \ \varepsilon_i \in \{\pm 1\} \sim \Bn(n,1/2) \\
& \leq & \E_{x_1^n,y_1^n, \varepsilon_1^n}\sup_{f\in \gF_B} \frac{1}{n} \sum_{i=1}^n( \varepsilon_i f(x_i) - \varepsilon_i f(y_i) )\\
& \leq & 2 \E_{x_1^n, \varepsilon_1^n}\sup_{f\in \gF_B} \frac{1}{n} \sum_{i=1}^n \varepsilon_i f(x_i) = 2 \RC(\gF_B)
\end{eqnarray*}
that ends the proof.
\end{proof}

We also need the following contraction inequality of Rademacher Complexity  \citep{LedTal91,MeiZha03}. 
\begin{lemma}[Rademacher Contraction Inequality]\label{lem:contraction}
For any Lipschitz function: $\phi:\R \to \R$ such that $|\phi(x) - \phi(y)|\leq L | x - y|$, 
\[ \RC(\phi\circ \gF ) \leq L \RC(\gF).\] 
\end{lemma}
\cite{LedTal91} has an additional factor 2 in the contraction inequality which is dropped in \cite{MeiZha03}. Its current form is stated in \cite{mohri2012foundations} as Talagrand's Lemma (Lemma 4.2). 


The last lemma gives the Rademacher complexity of the hypothesis space of maximum over functions in different hypothesis spaces \citep{LedTal91}.
\begin{lemma}\label{lem:maxrad}
Let $\gF_1, \ldots, \gF_m$ be $m$ hypothesis space and define
$$ \gM=\{\max\{f_1(x), \ldots, f_m(x)\}: \gX\to\R, f_i\in\gF_i, i=1,\ldots, m\}.$$ 
Then,
\begin{equation*}
\RC_n(\gM) \leq \sum_{i=1}^m\RC_n(\gF_i).
\end{equation*}
\end{lemma}

\subsection{Proof of Proposition \ref{prop:lowbound}}
\begin{proof}[Proof of Proposition \ref{prop:lowbound}]
The key idea is to approximate the linear function restricted in the Lipschitz ball by the neural network, where the local linearity of activation functions plays an important role. Therefore, we can show a subset of $\gH_L$ whose Rademacher complexity is larger than that of the (restricted) linear function.

We consider the Taylor expansion of $\sigma(x)$ around $x_0$, $\sigma(x)=\sigma(x_0) + \sigma'(x_0)(x-x_0) + o(x-x_0)$, and thus,
\begin{equation}\label{eq:prop1-1}
\sup_{x\in [x_0-\delta,x_0+\delta]} \frac{|\sigma(x)-(\sigma(x_0)+\sigma'(x_0)(x-x_0))|}{\delta} \to 0\ \textnormal{as}\ \delta\to 0^+,
\end{equation}
and there exists a $\delta_0>0$, $\forall\ 0<\delta\leq\delta_0$,
\begin{equation}\label{eq:prop1-2}
\frac{1}{\sigma'(x_0)}(\sigma(x)-\sigma(x_0)) + x_0 \in [x_0- \delta, x_0 + \delta]\ \textnormal{if}\ x \in [x_0- \delta/2, x_0 + \delta/2].\\
\end{equation}

Without loss of generality, we assume $x_0 = 0, \sigma(0)=0$ and $\sigma'(0)=1$ since we can always do a linear transformation before and after each activation function and the additional Lipschitz can be bounded by a constant. We further assume the Lipschitz constant $L_\sigma=1$ for simplicity. 

Let $\gT (r):=\{\left<w_0, x\right>: \|w_0\|_2\leq r\}$ be the class of linear function with Lipschitz semi-norm less than $r$ and we show that given a $M>0$, for each $t\in\gT (L)$, there exists $f\in\gF$ with $\|f\|_\gF\leq L$ and $y_0\in\{1, \ldots, K\}$ such that $h(x)= [f(x)]_{y_0}$ satisfying $|h(x)-\left<w_0, x\right>|\to 0,\ \forall\ \|x\|_2\leq M$.

To see this, define $t(x)=\left<w_0, x\right>$ with $\|w_0\|_2\leq L$, which satisfies $t\in \gT(L)$. Next we construct a particular $l$-layer network $f_{({w_0}, \delta, M)}: x\to x_l$ as follows
\begin{align*}
x^{0} & = x, \\
x^{0.5} & = W_1x^0 + b_1 = ((\left<w_0, x^0\right>\frac{\delta}{ML}), 0, \ldots, 0), \\
x^{1}  & = \sigma (x^{0.5}) = (\sigma(\left<w_0, x^0\right>\frac{\delta}{ML}), 0, \ldots, 0), \\
x^{i-0.5} &= W_{i}x^{i-1}+b_i = ([x^{i-1}]_1, 0, \ldots, 0),\\
x^{i} & = \sigma (x^{i-0.5}) = (\sigma[x^{i-1}]_1, 0, \ldots, 0) \ \ \ i=2,\ldots,l-1, \\
f_{w_0, \delta, M}(x)  & = W_l x^{l-1}+b_l = (\frac{MLx^{l-1}}{\delta}, 0, \ldots, 0). 
\end{align*}

With such a construction $f_{w_0, \delta, M}(x)$, define $h(x) = [f_{w_0, \delta, M}(x)]_1$. Then $h\in\gH_L$ since $\|f\|_\gF\leq \Pi_{i=1}^l \|W_i\|_\sigma\leq \|w_0\|\frac{\delta}{ML}\frac{ML}{\delta}\leq L$, and 
\begin{align}
|\left<w_0, x\right>- [f_{w_0, \delta, M}(x)]_1| &\leq \frac{ML}{\delta} |\sigma^{l-1}(\tilde{x}) - \tilde{x}|,\nonumber \\
&\leq ML\sum_{i=1}^{l-1} \frac{|\sigma^{i}(\tilde{x}) - \sigma^{i-1}(\tilde{x})|}{\delta}\xrightarrow{\delta\to 0^+} 0, \label{eq:prop1-3}
\end{align}
where $\tilde{x}= \left<w_0, x^0\right>\frac{\delta}{ML}$ and $\sigma^k$ stands for the composite of $k$ $\sigma$ functions. The second inequality is implied from (\ref{eq:prop1-1}) and (\ref{eq:prop1-2}) since $\tilde{x}\in[-\delta, \delta]$. Moreover, given $M>0$ and $\delta>0$, we define a subclass $\gH_L^{\delta, M}\subset \gH_L$ by,
\begin{equation*}
\gH_L^{\delta, M} = \{h(x): h(x)=[f_{w, \delta, M}(x)]_1\ \textnormal{with}\ \|w\|_2\leq L\}
\end{equation*}
We firstly consider the empirical Rademacher complexity for a given sample set $S$ of size $n$. Let $M_S=\sup_{x\in S}\|x\|_2$ and for any given $\delta>0$,
\begin{align}
\RC_S (\gH_L) &\geq \RC_S(\gH_L^{\delta, M_S}), \nonumber \\
&= \E_{\epsilon}\sup_{h\in\gH_L^{\delta, M_S}}\frac{1}{n}\sum_{i=1}^n \epsilon_i h(x_i), \nonumber \\
& = \E_{\epsilon}\sup_{\|w\|_2\leq L}\frac{1}{n}\sum_{i=1}^n \epsilon_i [f_{w, \delta, M_S}(x_i)]_1, \nonumber \\
& = \E_{\epsilon}\sup_{\|w\|_2\leq L}\frac{1}{n}\sum_{i=1}^n \epsilon_i \left(\left<w,x_i\right> - (\left<w,x_i\right> - [f_{w, \delta, M_S}(x_i)]_1)\right), \nonumber \\
&\geq \E_{\epsilon}\left[\sup_{\|w\|_2\leq L}\frac{1}{n}\sum_{i=1}^n \epsilon_i \left<w, x_i\right>\right] + \ldots \nonumber \\
& \ \ \ \ \ \ \ \ \ \ - \E_{\epsilon}\left[\sup_{\|w\|_2\leq L}\frac{1}{n}\sum_{i=1}^n\epsilon_i (\left<w, x_i\right>- [f_{w, \delta, M_S}(x_i)]_1)\right], \nonumber \\
& \geq \E_{\epsilon}\left[\sup_{\|w\|_2\leq L}\frac{1}{n}\sum_{i=1}^n \epsilon_i \left<w, x_i\right>\right] -\sup_i\sup_{\|w\|_2\leq L}|\left<w, x_i\right>- [f_{w, \delta, M_S}(x_i)]_1|, \nonumber \\
& = L\E_\epsilon\left\|\frac{1}{n}\sum_{i=1}^n\epsilon_i x_i \right\|_2 -\sup_i\sup_{\|w\|_2\leq L}|\left<w, x_i\right>- [f_{w, \delta, M_S}(x_i)]_1|, \label{eq:prop1-4} \\
& \geq CL\sqrt{\frac{1}{n}\sum_{i=1}^n\|x_i\|_2} - \sup_i\sup_{\|w\|_2\leq L}|\left<w, x_i\right>- [f_{w, \delta, M_S}(x_i)]_1|, \label{eq:prop1-5}
\end{align}
where (\ref{eq:prop1-4}) is implied from the Cauchy-Schwarz inequality and (\ref{eq:prop1-5}) is due to the Khintchine inequality.\\ 
From (\ref{eq:prop1-3}), we can choose proper $\delta_{M_S}>0$ such that, 
\begin{equation*}
\sup_i\sup_{\|w\|_2\leq L}|\left<w, x_i\right>- [f_{w, \delta_{M_S}, M_S}(x_i)]_1| \leq \frac{CL}{2}\E_S{\sqrt{ \frac{1}{n}\sum_{i=1}^n\|x_i\|_2}},
\end{equation*}
and the right hand side is independent with $S$. Then by taking expectation over $S$ in upper bound (\ref{eq:prop1-5}), 
\begin{equation*}
\RC_n(\gH_L) \geq CL \E_S{\sqrt{\frac{1}{n}\sum_{i=1}^n\|x_i\|_2}},
\end{equation*}
where we absorb a factor $1/2$ into constant $C$ without changing the notation.
\end{proof}

\subsection{Proof of Theorem \ref{thm:marg-err}}
\begin{proof}[Proof of Theorem \ref{thm:marg-err}]
Given $\theta>0$, we firstly introduce a useful lower bound of $\zeta(f(x), y)$, 
\begin{align*}
\zeta^\theta(f(x), y) :=& [f(x)]_y - \max_{y'}([f(x)]_{y'} - \theta 1[y=y']),\\
 	= & \min \left\{  [f(x)]_y - \max_{y'\neq y}([f(x)]_{y'} - \theta 1[y=y']), \theta 1[y=y'] \right\}, \\
	\leq & [f(x)]_y - \max_{y'\neq y}([f(x)]_{y'} - \theta 1[y=y']), \\
	= & [f(x)]_y - \max_{y'\neq y}[f(x)]_{y'} = \zeta(f(x), y).
\end{align*}
Therefore $\zeta^\theta(f(x), y)=\min(\zeta(f(x),y),\theta)$, that implies following equality for all $\theta\geq \gamma_2$,
\[
\ell_{(\gamma_1, \gamma_2)}(\zeta^{\theta}(f(x), y)) = \ell_{(\gamma_1, \gamma_2)}\zeta(f(x), y).
\]
Now define $\gG_L$ and $\gG^\theta_L$ as follows, 
\begin{align*}
\gG_L =& \{g(x, y)=\zeta(f(x), y): \gX\times\gY\to\R, f\in\gF\ \textnormal{with}\ \|f\|_{\gF}\leq L\},\\
\gG^\theta_L :=& \{g(x, y)=\zeta^{\theta}(f(x), y): \gX\times\gY\to\R, f\in\gF\ \textnormal{with}\ \|f\|_{\gF}\leq L\},
\end{align*}
and we can shift our attention from $\gG_L$ to $\gG^\theta_L$ which is the key to achieve a $O(K)$ factor in Theorem \ref{thm:marg-err} rather than $O(K^2)$. 

To see this, let $\nf:=f/L_f$ be the normalized network and thus $\zeta^{2\gamma_2}(\nf(x),y)\in\gG^{2\gamma_2}_1$. Then for any $\gamma_2>\gamma_1\geq 0$,
\begin{align}
P[\zeta(\nf(x),y)<\gamma_1] &\leq \E[\ell_{(\gamma_1, \gamma_2)}(\zeta(\nf(x), y))], \nonumber \\
& = \E[\ell_{(\gamma_1, \gamma_2)}(\zeta^{2\gamma_2}(\nf(x), y))], \nonumber \\
& \leq \eP_n \ell_{(\gamma_1, \gamma_2)}(\tilde{f}(x), y) + 2\RC_n({l_{(\gamma_1, \gamma_2)}\circ\gG^{2\gamma_2}_1}) + \sqrt{\frac{\log(1/\delta)}{2n}}, \nonumber \\
&\leq \eP_n \ell_{(\gamma_1, \gamma_2)}(\tilde{f}(x), y) + \frac{2}{\Delta}\RC_n(\gG^{2\gamma_2}_1) + \sqrt{\frac{\log(1/\delta)}{2n}}, \label{eq:thm1-1}
\end{align}
where the first inequality is implied from $1[\zeta <\gamma_1]\leq\ell_{(\gamma_1, \gamma_2)}(\zeta)$, the second inequality is a direct consequence of Lemma \ref{lem:ulll}, the third inequality results from Rademacher Contraction Inequality (Lemma \ref{lem:contraction}). 

Now we will do a detailed analysis on $\RC_n(\gG^{2\gamma_2}_1)$,
\begin{align}
\RC_n(\gG_1^{2\gamma_2}) &= \frac{1}{n}\E_{S,\epsilon}\left[\sup_{\|f\|_{\gF}\leq 1}\sum_{i=1}^n \epsilon_i \left([f(x_i)]_{y_i}  -  (\max_{y'} [f(x)]_{y'} - 2\gamma_21[y_i=y'])\right)\right],\nonumber \\
&\leq \underbrace{\frac{1}{n}\E_{S,\epsilon}\left[\sup_{\|f\|_{\gF}\leq 1}\sum_{i=1}^n\epsilon_i [f(x_i)]_{y_i}\right]}_{A_1} + \ldots \nonumber \\
& \ \ \ \ \ \ \ \ \ \ 
 	+\underbrace{\frac{1}{n} \E_{S,\epsilon}\left[\sup_{\|f\|_{\gF}\leq 1}\sum_{i=1}^n\epsilon_i \max_{y'}\left([f(x_i)]_{y'} - 2\gamma_21[y_i=y']\right)\right]}_{A_2},\label{eq:thm1-2}
\end{align}
\begin{align*}
A_1 &= \frac{1}{n}\E_{S,\epsilon}\left[\sup_{\|f\|_{\gF}\leq 1}\sum_{i=1}^n\epsilon_i\sum_{y\in\gY} [f(x_i)]_{y}1[y=y_i]\right], \\
&\leq \frac{1}{n}\sum_{y\in\gY}\E_{S,\epsilon}\left[\sup_{\|f\|_{\gF}\leq 1}\sum_{i=1}^n\epsilon_i[f(x_i)]_{y}1[y=y_i]\right], \\
&= \frac{1}{n}\sum_{y\in\gY}\E_{S,\epsilon}\left[\sup_{\|f\|_{\gF}\leq 1}\sum_{i=1}^n\epsilon_i[f(x_i)]_{y}\left(\frac{2\cdot1[y=y_i]-1}{2} + \frac{1}{2}\right)\right],\\
&\leq \frac{1}{2n}\sum_{y\in\gY}\E_{S,\epsilon}\left[\sup_{\|f\|_{\gF}\leq 1}\sum_{i=1}^n\epsilon_i[f(x_i)]_{y}\left(2\cdot1[y=y_i]-1\right)\right] + \ldots \\
&\ \ \ \ \ \ \ \ \ \ +	\frac{1}{2n}\sum_{y\in\gY}\E_{S,\epsilon}\left[\sup_{\|f\|_{\gF}\leq 1}\sum_{i=1}^n\epsilon_i[f(x_i)]_{y}\right],\\
&\leq \frac{1}{2n}\sum_{y\in\gY}\E_{S,\epsilon}\left[\sup_{\|f\|_{\gF}\leq 1}\sum_{i=1}^n\epsilon_i'[f(x_i)]_{y}\right] +
	\frac{1}{2n}\sum_{y\in\gY}\E_{S,\epsilon}\left[\sup_{\|f\|_{\gF}\leq 1}\sum_{i=1}^n\epsilon_i[f(x_i)]_{y}\right],\\
& \ \ \ \ \ \ \ \ \ \ \mbox{where $\epsilon'_i:=\epsilon_i (2\cdot1[y=y_i]-1) \overset{d}{=} \epsilon_i \sim \frac{1}{2}\delta_{-1}+\frac{1}{2}\delta_1 $},\\
&= \frac{1}{n}\sum_{y\in\gY}\E_{S,\epsilon}\left[\sup_{\|f\|_{\gF}\leq 1}\sum_{i=1}^n\epsilon_i[f(x_i)]_{y}\right],\\
&\leq \frac{1}{n}\sum_{y\in\gY}\E_{S,\epsilon}\left[\sup_{h\in\gH_1}\sum_{i=1}^n\epsilon_ih(x_i)\right],\\
&= K\RC_n(\gH_1).
\end{align*}
For the second term $A_2$ in (\ref{eq:thm1-2}),
\begin{align*}
A_2 &\leq \frac{1}{n} \sum_{y\in\gY} \E_{S,\epsilon}\left[\sup_{\|f\|_{\gF}\leq 1}\sum_{i=1}^n\epsilon_i \left([f(x_i)]_{y} - 2\gamma_21[y_i=y]\right)\right],\\
&= \frac{1}{n}\sum_{y\in\gY}\E_{S,\epsilon}\left[\sup_{\|f\|_{\gF}\leq 1}\sum_{i=1}^n\epsilon_i [f(x_i)]_{y}\right] - 
	\frac{1}{n}\sum_{y\in\gY}\E_{S,\epsilon}\left[\sum_{i=1}^n\epsilon_i 2\gamma_21[y_i=y]\right],\\
&= \frac{1}{n}\sum_{y\in\gY}\E_{S,\epsilon}\left[\sup_{\|f\|_{\gF}\leq 1}\sum_{i=1}^n\epsilon_i [f(x_i)]_{y}\right],\\
&\leq \frac{1}{n}\sum_{y\in\gY}\E_{S,\epsilon}\left[\sup_{h\in\gH_1}\sum_{i=1}^n\epsilon_ih(x_i)\right],\\
&= K\RC_n(\gH_1),
\end{align*}
where the first inequality is followed by Lemma \ref{lem:maxrad}. Note that $\zeta^{2\gamma_2}$ allows us to take maximum over $y\in\gY$ rather than $y\in\gY/\{y_i\}$, where in the second case, we have to take summation over two indices, that is $y$ and $y_i$, to get a margin function on $x$, and this will result in a factor $O(K^2)$. We finish the proof by combining the upper bound on $A_1$ and $A_2$ into (\ref{eq:thm1-1}),
\begin{align*}
P[\zeta(\nf(x),y)<\gamma_1] & \leq \eP_n \ell_{\gamma_1, \gamma_2}(\nf (x), y) + \frac{4K}{\Delta}\RC_n(\gH_1) + \sqrt{\frac{\log(1/\delta)}{2n}},\\
&\leq \eP_n \ell_{\gamma_2}(\nf (x), y) + \frac{4K}{\Delta}\RC_n(\gH_1) + \sqrt{\frac{\log(1/\delta)}{2n}},
\end{align*}
where the second inequality is implied from $\ell_{(\gamma_1, \gamma_2)}(\zeta)\leq 1[\zeta <\gamma_2]$.
\end{proof}
\begin{remark}
The key idea, that constructing $\zeta^\theta$ to use summation over one index results in $y$ in a factor $O(K)$, follows the proof of Theorem 2 in \cite{kuznetsov2015rademacher}. However, typical result toward multi-class margin bound has the factor $O(K^2)$ instead \citep{cortes2013multi,mohri2012foundations}.
\end{remark}

\subsection{Proof of Theorem \ref{thm:qmargin}}
\begin{proof}[Proof of Theorem \ref{thm:qmargin}]
Firstly, we show after normalization, the normalize margin has an upper bound, 
\begin{align*}
\|f(x)\|_{2} &= \|\sigma_{l}(W_{l}x_{l-1}+b_{l})\|_{2},\\
	        & \leq L_{\sigma_{l}}\|W_{l}x_{l-1}+b_{l}\|_{2}, \\
	        & \leq (L_{\sigma_{l}}\|\bar{W}_{l}\|_{\sigma})(\|x_{l-1}\|_{2} + 1) \\
	        & \ldots \\
	        & \leq \Pi_{i=1}^{l} (L_{\sigma_{i}}\|\bar{W}_{i}\|_{\sigma})\|x\|_{2} + \Sigma_{i=1}^{l}(\Pi_{j=i}^{l}(L_{\sigma_{i}}\|\bar{W}_{i}\|_{\sigma})), 
\end{align*}
where $x_{i}=\sigma_{i}(W_{i}x_{i-1}+b_{i})$ with $x_{0}=x$, $\bar{W}_{i} = (W_{i}, b_{i})$ and $L_{\sigma_{i}}$ is the Lipschitz constant of activation function $\sigma_{i}$ with $\sigma_{i}(0)=0, i=1,\ldots,l$. In the sequel as we consider the explosion of network Lipschitz over depths typically met in applications, we assume without loss of generality that $\|\bar{W}_{i}\|_{\sigma}\geq 1$ (otherwise we take the unit ball bound). Then, for normalized network $\nf = f/L_f $ with $L_{f} = \Pi_{i=1}^{l} (L_{\sigma_{i}}\|\bar{W}_{i}\|_{\sigma})$ and $\|x\|_{2}\leq M$,
\begin{equation*}
\|\nf(x)\|_{2}\leq M + l.
\end{equation*}
Therefore $\zeta(\nf(x), y)\leq 2\|\nf(x)\|_{2} = 2(M+l)\eqqcolon M_{1}$, and the quantile margin is also bounded $\hat{\gamma}_{q,t}\leq M_{1}$ for all $q\in(0,1),t=1,\ldots,T$. 

The remaining proof follows the idea from \citep{koltchinskii2002empirical, mohri2012foundations}. For any $\epsilon>0$, we take a sequence of $\epsilon_{k}$ and $\gamma_{k}, k=1,2,\ldots$ by $\epsilon_{k}=\epsilon+\sqrt{\frac{\log k}{n}}$ and $\gamma_{k}=M_{1}2^{-k}$. Let $A_{k}$ be the event $\eP[\zeta(\nf_{t}(x),y) < 0] > \eP_{n}[\zeta(\nf(x),y)<\gamma_{k}] + \frac{4K}{\gamma_{k}}\RC(\gH_{1}) + \epsilon_{k}.$ Then by Theorem \ref{thm:marg-err},
\begin{equation*}
\eP(A_{k}) \leq \exp(-2n\epsilon_{k}^{2}),
\end{equation*}
where the probability is taken over samples $\{x_{1},...x_{n}\}$. We further consider the probability for none of $A_{k}$ occurs,
\begin{align*}
\eP(\exists A_{k}) &\leq \Sigma_{k=1}^{\infty} P(A_{k}),\\
		        &\leq \Sigma_{k=1}^{\infty} \frac{1}{k^{2}}\exp(-2n\epsilon^{2}),\\
		        &\leq 2\exp(-2n\epsilon^{2}).
\end{align*}
Hence, fix a $q\in [0,1]$, for any $t=1,\ldots,T$, if $\hat{\gamma}_{q,t}>0$, there exists a $\hat{k}_t\geq 1$ (denoted as $\hat{k}$ for simplicity) such that,
\begin{equation}
\gamma_{\hat{k}+1} \leq \hat{\gamma}_{q,t} < \gamma_{\hat{k}}.
\end{equation}
Therefore, 
\begin{align*}
A_{\hat{k}+1} &\supseteq \eP[\zeta(\nf_{t}(x),y)<0] > \eP_{n}[\zeta(\nf_{t}(x),y)<\hat{\gamma}_{q,t}] + \frac{4K}{\gamma_{\hat{k}+1}}\RC(\gH_{1}) + \epsilon_{\hat{k}+1},\\
	    &\supseteq \eP[\zeta(\nf_{t}(x),y)<0] > \eP_{n}[\zeta(\nf_{t}(x),y)<\hat{\gamma}_{q,t}] + \frac{8K}{\hat{\gamma}_{q,t}}\RC(\gH_{1}) + \epsilon_{\hat{k}+1},\\
	    &= \eP[\zeta(\nf_{t}(x),y)<0] > \eP_{n}[\zeta(\nf_{t}(x),y)>\hat{\gamma}_{q,t}] + \frac{8K}{\hat{\gamma}_{q,t}}\RC(\gH_{1}) + \ldots \\
	    &\ \ \ \ \ \ \ \ \ \ \ \ \ \ \ \ \ \ \ \ + \epsilon + \sqrt{\frac{\log(\hat{k}+1)}{n},}\\
	    &\supseteq \eP[\zeta(\nf_{t}(x),y)<0] > \eP_{n}[\zeta(\nf_{t}(x),y)>\hat{\gamma}_{q,t}] + \frac{8K}{\hat{\gamma}_{q,t}}\RC(\gH_{1}) + \ldots \\
	    &\ \ \ \ \ \ \ \ \ \ \ \ \ \ \ \ \ \ \ \ +\epsilon + \sqrt{\frac{\log\log_{2}(2M_{1}/\hat{\gamma}_{q,t})}{n}}.
\end{align*} 
The first inequality is implied from $\eP_{n}[\zeta(\nf_{t}(x),y)<\hat{\gamma}_{q,t}] > \eP_{n}[\zeta(\nf_{t}(x),y)<\gamma_{\hat{k}+1}]$, since $\gamma_{\hat{k}+1}\leq \hat{\gamma}_{q,t}$. The second inequality is implied from $\hat{\gamma}_{q,t}<2\gamma_{\hat{k}+1}$ and thus, $1/\gamma_{\hat{k}+1} < 2/\hat{\gamma}_{q,t}$. The third equality is the direct definition of $\epsilon_{\hat{k}}$. The last inequality is implied from $\hat{k}+1=\log_{2}(M_{1}/\gamma_{\hat{k}+1})$ and again, $1/\gamma_{\hat{k}+1} < 2/\hat{\gamma}_{q,t}$. The conclusion is proved immediately by letting $\epsilon=\sqrt{\frac{1}{2n}\log \frac{2}{\delta}}$.
\end{proof}

\subsection{Proof of Proposition \ref{prop:convnorm}}\label{app:convnorm}

\begin{proof}[Proof of Lemma \ref{prop:convnorm}] (A)
\begin{eqnarray*}
\|w\ast x\|_2^2 & = & \sum_u \left(\sum_v x(v) w(u-v)\right)^2 \\
& = & \sum_u \left(\sum_v (x(v) \sqrt{|w(u-v)|}\cdot \sqrt{|w(u-v)|} \right)^2 \\
& \leq & \sum_u \left\{\left(\sum_v x(v)^2 |w(u-v)|\right) \left(\sum_v |w(u-v)|\right)\right\},  \\
& = & \|w\|^2_1 \|x\|_2^2,
\end{eqnarray*}
where the second last step is due to the Cauchy-Schwartz inequality. 

(B) Similarly,
\begin{eqnarray*}
\|w\ast x\|_2^2 & = & \sum_{u,j\leq \mathrm{C_{out}}} \left(\sum_{v,i\leq {\mathrm{C_{in}}}} x(v,i) w(j,i,u-v)\right)^2 \\
& = & \sum_{u,j} \left(\sum_{v,i} (x(v,i) \sqrt{|w(j,i,u-v)|}\cdot \sqrt{|w(j,i,u-v)|} \right)^2 \\
& {(a) \atop \leq} & \sum_{u,j} \left\{\left(\sum_{v,i} x(v,i)^2 |w(j,i,u-v)|\right) \left(\sum_{v,i} |w(j,i,u-v)|\right)\right\},  \\
& {(b) \atop =} & \sum_{j}  \| w(j,\cdot,\cdot)\|_1   \left(\sum_{u,v,i} x(v,i)^2 |w(j,i,u-v)|\right),  \\
& {(c) \atop \leq} & \sum_{j}  \| w(j,\cdot,\cdot)\|_1 \left(\sum_{v,i} x(v,i)^2 \|w(j,i,\cdot)\|_1\right),  \\
& {(d) \atop \leq} & \sum_{j}  \| w(j,\cdot,\cdot)\|_1 (\max_i \|w(j,i,\cdot)\|_1) \left(\sum_{v,i} x(v,i)^2 \right), \\
& {(e) \atop \leq} & (\max_{i,j} \|w(j,i,\cdot)\|_1)  \| w\|_1  \left(\sum_{v,i} x(v,i)^2 \right), 
\end{eqnarray*}
where the inequality (a) is due to the Cauchy-Schwartz inequality, step (b) and (c) are due to $\sum_{u} |w(j,i,u-v)|\leq \sum_{v} |w(j,i,u-v)|=\|w(j,i,\cdot)\|_1$ where equality holds if the stride is 1. 

In particular, for a convolution kernel $w$ of large stride $S\geq 1$,  $\sum_{u} |w(j,i,u-v)|\leq D\|w(j,i,\cdot)\|_\infty \leq \max_{i,j} D\|w(j,i,\cdot)\|_\infty$. Hence step (e) becomes 
\[ D  \|w\|_\infty  \| w\|_1\|x\|_2^2, \] 
which gives the stride-sensitive bound.
\end{proof}

\newpage

\bibliography{iclr2019_conference}
\bibliographystyle{natbib}


\end{document}